\documentclass[11pt]{article}
\usepackage{amsmath}
\usepackage{amssymb}
\usepackage{amsthm}
\usepackage[font=footnotesize]{subfig}
\usepackage{color}
\usepackage[left=2.0cm,top=2.5cm,right=2.0cm,bottom=2.5cm,bindingoffset=0.0cm]{geometry}
\usepackage{graphicx}
\usepackage[table]{xcolor}
\usepackage{pgfplots}
\usepackage{bm}

\pgfplotsset{compat=newest} 
\pgfplotsset{plot coordinates/math parser=false} 
\pgfplotsset{ 
  legend style =
  {font=\small\sffamily},
  label style = {font=\footnotesize\sffamily},
	tick label style = {font=\footnotesize}
}
\newlength\figureheight 
\newlength\figurewidth 

\usepackage{tikz}
\usetikzlibrary{arrows,backgrounds,positioning,fit}

\newtheorem{theorem}{Theorem}
\newtheorem{proposition}{Proposition}

\newtheorem{corollary}{Corollary}

\newcommand{\C}{\mathbb{C}}
\newcommand{\mbf}{\bm}
\newcommand{\bs}{\boldsymbol}

\DeclareMathOperator*{\argmin}{arg\,min}

\title{Off-the-Grid Recovery of Piecewise Constant Images\\ from Few Fourier Samples\thanks{This work is supported by grants NSF CCF-0844812, NSF CCF-1116067, NIH 1R21HL109710-01A1, ACS RSG-11-267-01-CCE, ONR-N000141310202, NIH 1R01EB019961-01A1.}}

\author{Greg Ongie\thanks{Department of Mathematics, University of Iowa, Iowa City, Iowa.
(\emph{gregory-ongie@uiowa.edu}).}~~and Mathews Jacob\thanks{Department of Electrical and Computer Engineering, University of Iowa, Iowa City, Iowa.
(\emph{mathews-jacob@uiowa.edu})}}

\begin{document}

\maketitle

\begin{abstract}
We introduce a method to recover a continuous domain representation of a piecewise constant two-dimensional image from few low-pass Fourier samples. Assuming the edge set of the image is localized to the zero set of a trigonometric polynomial, we show the Fourier coefficients of the partial derivatives of the image satisfy a linear annihilation relation. We present necessary and sufficient conditions for unique recovery of the image from finite low-pass Fourier samples using the annihilation relation. We also propose a practical two-stage recovery algorithm which is robust to model-mismatch and noise. In the first stage we estimate a continuous domain representation of the edge set of the image. In the second stage we perform an extrapolation in Fourier domain by a least squares two-dimensional linear prediction, which recovers the exact Fourier coefficients of the underlying image.  We demonstrate our algorithm on the super-resolution recovery of MRI phantoms and real MRI data from low-pass Fourier samples, which shows benefits over standard approaches for single-image super-resolution MRI.
\end{abstract}

\section{Introduction}
Many studies in image processing show that natural images can be represented efficiently as piecewise smooth functions \cite{mallat1989theory,starck2002curvelet,do2005contourlet}. For instance, the Mumford-Shah formulation poses the recovery of a piecewise smooth image as a variational optimization problem that jointly estimates a segmentation of the image and the smooth functions on the regions defined by the segmentation \cite{mumford1988boundary,mumford1989optimal}. This problem is typically solved using iterative algorithms that alternate between the evolution a level-set curve and the recovery of the full signal \cite{chan2001level,tsai2001curve,vese2002multiphase}. These constrained level-set methods have been demonstrated to yield considerably improved reconstructions in challenging inverse problems \cite{ye2002self,Jacob:2006p2905,storath2014jump,burger2001level,burger2005survey,dorn2006level,schweiger20103d}. Unfortunately, these algorithms are often slow and vulnerable to local minima problems, which restrict their utility in applications. 

At the same time, the recovery of piecewise smooth one-dimensional signals using harmonic retrieval/linear prediction has been actively investigated in signal processing \cite{liang1989high,vetterli2002sampling,maravic2005sampling,dragotti2007sampling,blu2008sparse}. This includes the so-called finite-rate-of-innovation (FRI) framework \cite{blu2008sparse} which extends Prony's method \cite{stoica1997introduction,cheng2003review} to the recovery of a much wider class of signals, including non-uniform 1-D splines with unknown knot locations. The FRI framework is also related to the earlier work \cite{haacke1989super,haacke1989constrained,liang1989high,haacke1990image}, where Haacke and Liang proposed a super-resolution reconstruction technique for MRI using piecewise 1-D polynomial modeling. However, all of these techniques are fundamentally tied to the recovery of 1-D signals. While some extensions of 1-D FRI schemes to multidimensional polygonal signals are available \cite{milanfar1995reconstructing,shukla2007sampling,chen20122d}, these specialized models do not provide good approximations of natural images. Recent work on atomic norm minimization can also be applied to the recovery of a weighted linear combination of 1-D Diracs at arbitrary locations from few Fourier samples \cite{candes2013super,tang2013near,candes2014towards}; these methods are off-the-grid continuous domain generalizations of compressed sensing theory \cite{donoho2006compressed,baraniuk2007compressive}. While the framework can account for piecewise constant 1-D signals by analyzing the derivative of the signal \cite{candes2013super}, and linear combination of Diracs in 2-D \cite{xu2014precise}, its extension to piecewise constant signals in two or higher dimensions is not straightforward. For instance, the partial derivatives piecewise constant images have singularities supported on curves that bound the constant regions. We observe that this violates two of the main assumptions required for super-resolution recovery in \cite{candes2013super}: (1) the number of singularities is finite, and (2) the singularities are well-separated. Moreover, the 2-D setting in \cite{xu2014precise,jin2015general} fails to capture the extensive structure that is present in imaging applications. Namely, the image discontinuities are not often isolated, but are instead localized to smooth curves in 2-D.

Recently, Pan \emph{et al.} \cite{pan2013sampling} proposed an extension of FRI methods to recover piecewise complex analytic images in continuous domain. Specifically, they assume that the complex derivative of the image is supported on the zero level-set of a band-limited periodic function, called the annihilating polynomial. The authors rely on the annihilation of the Fourier transform of the derivatives by convolution with the Fourier coefficients of the annihilating polynomial to recover the edge set, and eventually the image. The work in this paper is mainly motivated by \cite{pan2013sampling}, and aims to overcome several of its limitations. For instance, the recovery of an arbitrary piecewise complex analytic signal from fintely many Fourier coefficients is ill-posed, since this class of functions has infinite degrees of freedom. Even in recovering the edge set, no sampling conditions that guarantee perfect recovery were available in \cite{pan2013sampling}. Another problem is that the complex analytic signal model in \cite{pan2013sampling} is not realistic for natural images.

The main focus of this paper is to develop a novel theoretical framework and efficient algorithms to recover piecewise constant images, whose discontinuities are localized to zero level-sets of band-limited functions, from few Fourier samples. We introduce a two step recovery scheme, which has similarities to Prony's method and traditional FRI recovery. The first step involves the determination of the discontinuities of the image as the zero set of a band-limited function, which translates to a set of linear annihilation conditions in Fourier domain. Once the edge set is recovered, we recover the full image as the solution to an optimization problem. This work connects the classical Mumford-Shah algorithm, which alternates between level-set estimation and signal estimation, with FRI theory, which provides recovery guarantees from a finite number of Fourier samples.

Unlike the piecewise analytic model in \cite{pan2013sampling}, the number of degrees of freedom of our proposed piecewise constant representation is finite, which enables the successful recovery of the signal from finite number of Fourier coefficients. We introduce necessary and sufficient conditions for the recovery of the edge set and the edge set aware recovery of the signal. The one-to-one correspondence between the degree of the annihilating polynomial and the number of singularities in the signal, which is used to establish guarantees in the traditional FRI settings \cite{vetterli2002sampling}, cannot be exploited in our setting since the number of singularities is no longer finite. Instead, we introduce novel proof techniques based on the algebraic geometry of zero sets of trigonometric polynomials to establish the uniqueness guarantees. 

We also introduce several refinements to make the two-step algorithm robust to noise and model mismatch. Since the model order--the number of coefficients necessary to describe the edge set--is not typically known in advance, we rely on a singular value decomposition of the block Toeplitz matrix corresponding to the linear annihilation relation for model order determination. Once the null-space of the matrix is estimated, we pose the extrapolation of the Fourier coefficients of the signal from the known samples as an optimization problem. The algorithm determines the signal such that the energy of the convolution between its Fourier coefficients and the null-space filters is minimized. We show that this approach is approximately equivalent to estimating the edge set as a sum of squares polynomial built from a basis of null-space filters. This approach is similar to the projection step in the multiple signal classification algorithm (MUSIC) algorithm \cite{schmidt1986multiple}, used in direction-of-arrival estimation and other signal processing applications.  

We demonstrate the proposed two-stage recovery scheme on the the super-resolution of magnetic resonance (MR) images from low-resolution Fourier samples. The experiments show our scheme has benefits over other standard discrete spatial domain approaches to single-image super-resolution in MRI. In particular, our scheme is better able to resolve fine image features, preserve strong edges, and suppress noise artifacts. Our work is also fundamentally different from several super-resolution approaches in image processing and MRI, where edge information from other sources (co-registered high resolution datasets, anatomical information, \emph{etc.}) is used to constrain reconstructions from few measurements; see e.g.\ \cite{hu1988slim,xiang2005accelerating,eslami2010robust,jacob2007improved,khalidov2007bslim,haldar2008anatomically,vaswani2010modified,luo2012mri,gong2015promise}. Specifically, we seek to estimate the edge locations using only the partial Fourier data itself rather than using additional prior information. This approach can be viewed as a fully 2-D version of the 1-D super-resolution approach investigated in \cite{haacke1989super,haacke1989constrained,liang1989high,haacke1990image}.

Our work also has conceptual similarities to recent low-rank structured matrix completion methods in MRI \cite{sake,haldar2014low,jin2015general}. In particular, \cite{haldar2014low} and \cite{jin2015general} exploit the low-rank structure of a multi-fold Hankel matrix obtained built from Fourier samples. The low-rank property of the structured matrix results from assuming signal sparsity and other constraints. For example, in \cite{haldar2014low} MR images are modeled as signals with limited spatial support and smoothly varying phase. We observe that these image models are far less constrained than the piecewise constant model considered in this work, which makes them less suitable for super-resolution recovery. The model in \cite{jin2015general} assumes that the signal sparse in a transform domain, i.e.\ consists of isolated singularities, which allows them to adapt the theory in \cite{chen2014robust} for recovery. In contrast, we model the partial derivatives of the continuous signal to be localized to curves. The direct extension of the theory from \cite{chen2014robust} as well as the traditional FRI setting is not possible in our context. Since our model captures the smoothness of the edge contours often found in images, it is more constrained than models like \cite{jin2015general} that assume isolated discontinuities.

Preliminary versions of this work were presented in conference papers \cite{isbi2015} and \cite{sampta2015}. In particular, Theorems \ref{thm:unique1} and \ref{thm:unique2} were stated in \cite{sampta2015} without proof. In addition to supplying the proofs of these results, we provide sampling guarantees for edge aware signal recovery that were absent in \cite{sampta2015}. While the basic version of the two step algorithm presented here was introduced in \cite{isbi2015}, it is considerably generalized and validated using MRI data in this work. 

\subsection{Notation}
We collect notation used throughout the paper here for easy reference. Bold lower-case letters $\mbf x$ are used to indicate vector quantities, bold upper-case $\mbf X$ to denote matrices, and calligraphic script $\mathcal{X}$ for general linear operators. We typically reserve lower-case greek letters $\mu, \gamma$, \emph{etc.}\ for trigonometric polynomials \eqref{eq:trigpoly} and upper-case greek letters $\Lambda, \Gamma,$ \emph{etc}.\ for their coefficient index sets, i.e.\ finite subsets of the integer lattice $\mathbb{Z}^2$, with cardinality denoted by $|\Lambda|$. A set of complex coefficients indexed by $\Lambda$ is denoted by $(c[\mbf k] : \mbf k \in \Lambda)$, which we also treat as a vector in $\mathbb{C}^{|\Lambda|}$ when convenient. We write $\Lambda + \Gamma$ for the dilation of the index set $\Gamma$ by $\Lambda$, i.e.\ the Minkowski sum $\{\mbf k + \mbf \ell : \mbf k\in\Lambda,~\bs\ell \in \Gamma\}$, and write $2\Lambda$ to mean $\Lambda+\Lambda$, $3\Lambda = 2\Lambda +\Lambda$, \emph{etc}. We also denote the contraction of $\Gamma$ by $\Lambda$ by $\Gamma\,{:}\,\Lambda = \{\mbf \ell \in \Gamma\,{:}\, \mbf \ell - \mbf k \in \Gamma \text{ for all } \mbf k \in \Lambda\}$. For two coefficient vectors $\mbf c = (c[\mbf k] : \mbf k \in \Lambda)$ and $\mbf d = (d[\mbf k] : \mbf k \in \Lambda)$, we denote their Hermitian inner product by $\langle \mbf c,\mbf d\rangle = \sum_{\mbf k \in \Lambda} c[\mbf k] \overline{d[\mbf k]}$. For any square-integrable functions $f$ and $g$ we use $\langle f, g\rangle$, to denote the usual Hermitian inner product $\langle f, g\rangle = \int f(\mbf r)\overline{g(\mbf r)} d \mbf r$. Finally, $\langle D,\varphi\rangle$ is also used to denote the dual pairing for a tempered distribution $D$ acting on the Schwartz space $\mathcal{S}$ of (complex) test functions $\varphi$.

\section{Overview}
We introduce a two step method to recover a piecewise constant image from few low-pass Fourier samples, assuming the edge set of the image is well-approximated by the zero level-set of a bandlimited function. In the first step, we estimate the edge set by solving for the level-set function. In the second step, we use the level-set function to perform an edge-aware recovery of the full signal. To make the presentation self-contained, and to motivate our proposed 2-D recovery scheme, we start with a brief review of FRI recovery of piecewise constant signals in 1-D. 

\subsection{Review of 1-D FRI theory}
\label{sec:1dFRI}
The FRI framework introduced in \cite{vetterli2002sampling} can be viewed as an extension of Prony's method to the recovery of a wide class of signals from sub-Nyquist sampling rates, including piecewise polynomial and non-uniform spline signals (similar extensions of Prony's method for the recovery of 1-D piecewise polynomials was considered in Haacke and Liang's earlier work \cite{haacke1989super,haacke1989constrained,liang1989high,haacke1990image}). In order to draw parallels to our 2-D recovery scheme, we focus here on the FRI recovery of a continuous domain 1-D piecewise constant signal $f:[0,1]\rightarrow\mathbb{C}$, given as
\begin{equation}
f(x) = \sum_{i=1}^{K-1} a_i~ 1_{[x_i,x_{i+1})}(x); 
\label{eq:1dpwc}
\end{equation}
from its low-pass Fourier coefficients
\[
\widehat{f}[k] = \int_0^1 f(x) e^{-j2\pi k x} dx; ~~k\in\mathbb{Z},\,|k| \leq N.
\]
for some fixed $N$. In \eqref{eq:1dpwc}, $1_U$ denotes the characteristic function of the set $U$, $a_i\in\mathbb{C}$ are the signal amplitudes, and $0 < x_1 < x_2 < \cdots < x_K < 1$ are the edge locations. See Figure \ref{fig:1d2dfri} for an illustration of this signal model and the two-step FRI recovery described below.

First, observe that the derivative $\partial f$ (in the sense of distributions) is a stream of Diracs:
\begin{equation}
\partial f(x) = \sum_{i=1}^K b_i\,\delta(x-x_i);
\label{eq:steamdiracs}
\end{equation}
where $b_i = a_{i}-a_{i-1}$, for $i=1,\ldots,K$ with $a_K = a_0 = 0$. 
Therefore, the Fourier coefficients of $\partial f$ are a linear combination of exponentials
\begin{equation}
\widehat{\partial f}[k] = \sum_{i=1}^K b_i\,e^{-j2\pi k x_i},
\label{eq:1dparamodel}
\end{equation}
which can be computed from the known samples $\widehat{f}[k]$ by the Fourier domain relation $\widehat{\partial f}[k] = j2\pi k\widehat{f}[k]$. Recovering $f$ now reduces to estimating the $2K$ unknown parameters $\{b_i,x_i\}_{i=1}^{K}$ from the known $2N+1$ samples. This can be achieved using Prony's method, or one of its robust variants \cite{cheng2003review}. The solution relies on determining the so-called \emph{annihilating polynomial}, i.e.\ the trigonometric polynomial defined as
\begin{equation}
\mu(x) = \prod_{i=1}^K (e^{j2\pi x}-e^{j2\pi x_i}) = \sum_{k=0}^K c[k]~ e^{j2\pi kx}.
\label{eq:1dmu}
\end{equation}
which is zero at each of the edge locations $\{x_i\}_{i=1}^K$, and these are its only zeros. An easy computation shows that the Fourier coefficients $\widehat{\partial f}[k]$ are annihilated by convolution with the annihilating polynomial coefficients $c[k]$:
\begin{equation}
\sum_{k=0}^K c[k]~ \widehat{\partial f}[\ell-k] = 0,~~\text{for all}~~\ell\in\mathbb{Z}.
\label{eq:1dannsys}
\end{equation}
This can also be seen from fact that in spatial domain,
\[
\mu\,\partial f = 0,
\]
in the sense of distributions. Therefore, provided we have at least $2K$ contiguous samples of $\widehat{\partial f}$ (i.e.\ $N \geq K$) we can form a linear Toeplitz system from \eqref{eq:1dannsys}, and solve for the $K+1$ unknown coefficients $c[k]$. This allows us to recover the edge locations $\{x_i\}_{i=1}^K$ as the roots of $\mu(x)$. Then we may solve for the coefficients $b_i$, and hence the amplitudes $a_i$, by substituting the $\{x_i\}_{i=1}^K$ into the model \eqref{eq:1dparamodel} and solving an overdetermined linear system.

The main tool enabling FRI recovery is the existence of a trigonometric polynomial whose zero set coincides with the edge locations of the piecewise constant signal, or equivalently, the singularities of its derivative. Similar to \cite{pan2013sampling}, we propose extending this scheme to 2-D by considering the infinite one-dimensional zero sets of two-variable trigonometric polynomials.

\begin{figure*}
\centering
\subfloat[][FRI recovery of 1-D piecewise constant signals]{
\includegraphics[width=0.85\textwidth]{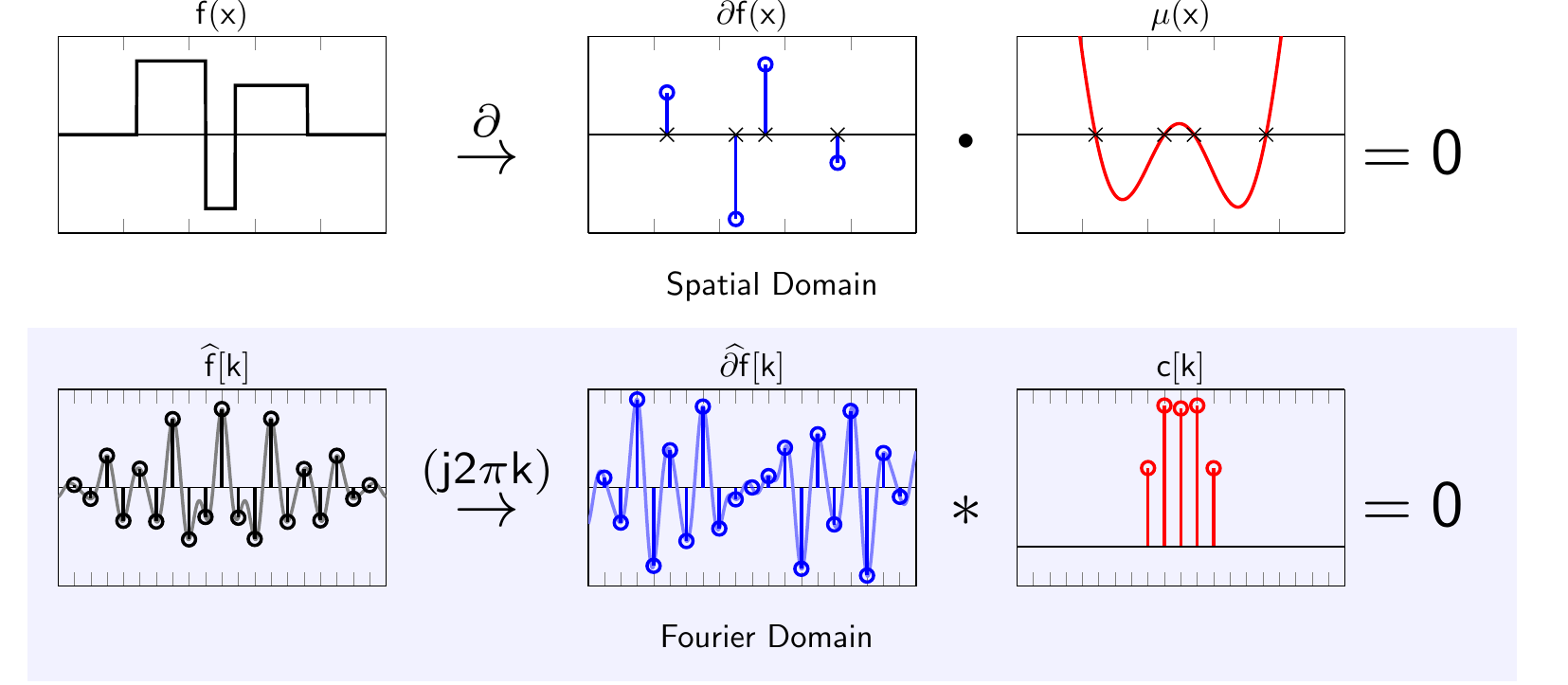}
}\\

\subfloat[][Proposed extension to 2-D piecewise constant images]{
\includegraphics[width=0.85\textwidth]{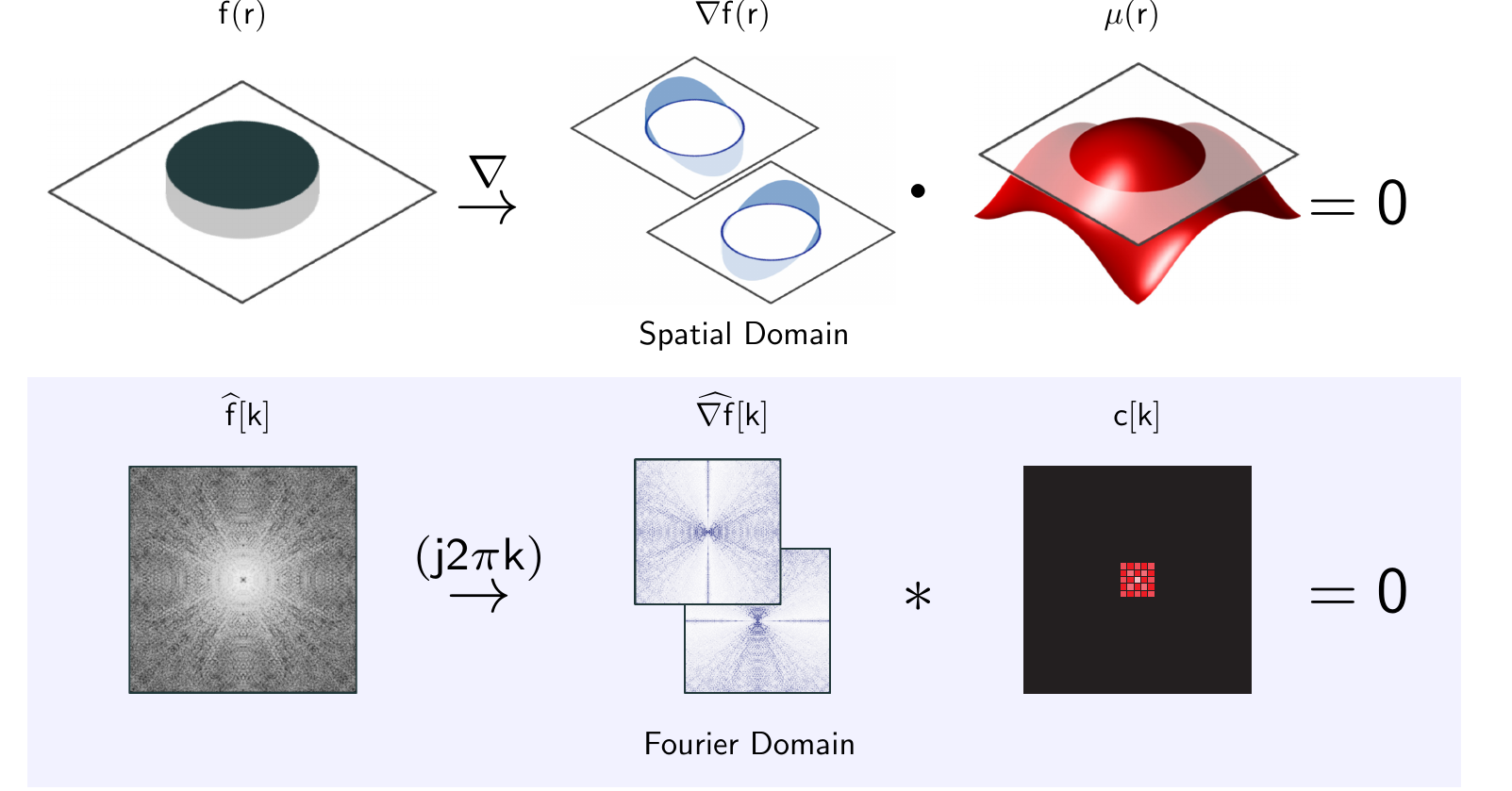}
}
\caption{\small \emph{Overview of proposed extension of FRI to 2-D piecewise constant images.} In (a) we illustrate the FRI recovery of a 1-D continuous domain piecewise constant signal $f(x)$, as described in \S\ref{sec:1dFRI}. The distributional derivative $\partial f(x)$ is a stream of Diracs supported on the edge set of the signal. Hence, multiplication by a trigonometric polynomial $\mu(x)$ which is zero on the edge set will annihilate $\partial f(x)$ in spatial domain. This descends to a linear convolution relation in Fourier domain. Therefore, given sufficiently many uniform low-pass Fourier samples $\widehat{f}[k]$ we may compute $\partial\widehat{f}[k]$, and form a linear system to solve for the finite Fourier coefficients $c[k]$ of $\mu(x)$. In (b) we illustrate the proposed extension of this scheme to 2-D presented in \S\ref{sec:annprop}. Here $f(\mbf r)$, $\mbf r = (x,y)$, is now the characteristic function for a region $U$ in the plane. We show the distributional gradient of $f$ is supported on the one-dimensional edge set $\partial U$. Multiplication by a trigonometric polynomial $\mu(\mbf r)$ which vanishes on $\partial U$ will annihilate $\nabla f$ in spatial domain, and this likewise translates to a linear convolution annihilation relation in Fourier domain, which we may use to solve for coefficients $c[\mbf k]$ of $\mu(\mbf r)$.}
\label{fig:2dfri}
\end{figure*}
\subsection{Image Model}

We focus on a continuous domain representation of images, that is, as functions $f:\mathbb{R}^2\rightarrow\mathbb{C}$ with compact support. For simplicity we will assume the support is always contained within $[0,1]^2$. We are particularly interested in images that can be modeled as 2-D \emph{piecewise constant} functions 
\begin{equation}
\label{eq:pwc}
f(\mbf r) = \sum_{i=1}^n a_i\,1_{U_i}(\mbf r),~~\text{ for all }~~\mbf r = (x,y)\in{[0,1]}^2,
\end{equation}
where $a_i \in \mathbb{C}$, and $1_{U}$ is the characteristic function of the region $U$:
\begin{equation}
1_U(\mathbf r) = 
\begin{cases}
1 & \text{if } \mbf r \in U\\
0 & \text{else}.
\end{cases}
\label{eq:char_func}
\end{equation} 
Here we assume each region $U_i\subseteq [0,1]^2$ has piecewise smooth boundary $\partial U_i$. We also assume the representation \eqref{eq:pwc} is expressed with the fewest number of characteristic functions such that the regions $U_i$ are pairwise disjoint. In this case the set of discontinuous points of $f$ coincides with $E = \cup_{i=1}^n \partial U_i$, which we call the \emph{edge set} of the image.

Our aim is to recover the piecewise constant image $f(\mbf r)$ from few of its Fourier coefficients 
\begin{equation}
\widehat f [\mbf k] = \int_{[0,1]^2} f(\mbf r) e^{-j 2\pi \mbf k\cdot\mbf r} d\mbf r,~~\mbf k\in\Gamma \subseteq \mathbb{Z}^2
\label{eq:Fourier}
\end{equation}
where $\Gamma$ is some finite sampling set. Without any further assumptions on the edge set, unique recovery of images in the form \eqref{eq:pwc} from finitely many Fourier samples is clearly ill-posed since the class of piecewise smooth curves has infinite degrees of freedom. Hence, we restrict the edge set to belong to a parametric model described by finitely many parameters. Motivated by the FRI curves framework \cite{pan2013sampling}, we will constrain the edge set to be the zero set of a 2-D band-limited periodic function:
\begin{equation}
C =\{\mathbf r \in[0,1]^2 : \mu(\mbf r)=0\}, ~\mbox{ where } ~\mu(\mbf r) = \sum_{\mbf k\in{ \Lambda}} c[\mbf k]\, e^{j2\pi\mbf k \cdot \mbf r},\quad \forall \mbf r\in{[0,1]}^2,
\label{eq:trigpoly}
\end{equation}
where the coefficients $c[\mbf k]\in\mathbb{C}$, and ${\Lambda}$ is any finite subset of $\mathbb{Z}^2$. We call any function $\mu$ described by \eqref{eq:trigpoly} a \emph{trigonometric polynomial}, and the zero set $C = \{\mu = 0\}$ a \emph{trigonometric curve}, provided $C$ is infinite and has no isolated points. Note a piecewise constant image cannot have isolated points in its edge set, so this restriction on the zero set is fully compatible with our image model. We elaborate more on the properties of trigonometric curves in \S \ref{sec:trigcurves}; see also Fig.\ \ref{fig:trigcurves}.

\subsection{Annihilation property for piecewise constant images}
\label{sec:annprop}
The annihilation relation \eqref{eq:1dannsys} for 1-D piecewise constant signals is derived from the fact that the derivative of the signal can be represented as a sum of isolated Diracs. Similar annihilation relations extend to multi-dimensional signals that can expressed as sum of isolated Diracs in a transform domain \cite{shukla2007sampling,chen20122d,jin2015general}. However, the derivative of a piecewise constant image cannot be represented as a sum of isolated Diracs, which prohibits the direct extension of these methods to our setting. However, \cite{pan2013sampling} showed it is still possible to derive an annihilation relationship for 2-D complex analytic signals having singularities supported on trigonometric curves \eqref{eq:trigpoly}. Similarly, we now show that the Fourier coefficients of the partial derivatives of piecewise constant images \eqref{eq:pwc} having edge set described by a trigonometric curve satisfy a 2-D annihilation relation in Fourier domain.

First consider the case of a single characteristic function $1_U$ where $U$ is a simply connected region with piecewise smooth boundary $\partial U$. We will show that the product of the gradient of $1_U$ by any smooth periodic function $\psi$ vanishing on $\partial U$ is identically zero:
\begin{equation}
\psi\, \nabla 1_U = 0.
\label{eq:dist_ann}
\end{equation}
However, since $1_U$ is non-smooth, the above equality can only be understood in a weak sense. In particular, we will find it useful to interpret the partial derivatives of $1_U$ as tempered distributions, i.e.\ as linear functionals acting on test functions of Schwartz class $\mathcal{S}$ (see, e.g., \cite{strichartz2003guide} or \cite{edwards2012fourier}). Since the domain $[0,1]^2$ is compact, in this case $\mathcal{S}$ coincides with all $C^\infty$ smooth complex-valued periodic functions on $[0,1]^2$. Letting $\varphi \in \mathcal{S}$ denote any such test function, we have:
\begin{equation}
\langle \partial_x 1_U, \varphi\rangle  = - \langle 1_U, \partial_x \varphi\rangle =  -\int_U \partial_x \varphi \, d\mbf r = -\oint_{\partial U} \varphi\, dy
\label{eq:partx}
\end{equation}
where the last step follows by Green's theorem. Likewise, one can show
\begin{equation}
\langle \partial_y 1_U, \varphi\rangle  = \oint_{\partial U} \varphi\, dx.
\label{eq:party}
\end{equation}
We can generalize the above relation to the gradient $\nabla 1_U = (\partial_x 1_U, \partial_y 1_U)$ by considering it to be a distribution acting on test fields $\boldsymbol\varphi = (\varphi_1,\varphi_2)$ where $\varphi_1,\varphi_2\in\mathcal{S}$. For all such $\mbf \varphi$ we have
\begin{equation}
\langle \nabla 1_U, \boldsymbol\varphi \rangle = -\langle 1_U, \nabla \cdot \boldsymbol\varphi \rangle = -\int_U \nabla \cdot \boldsymbol\varphi\, d\mathbf r = -\oint_{\partial U} \boldsymbol\varphi \cdot \mbf n\, ds
\label{eq:continuum}
\end{equation}
where $\mbf n$ is the unit outward normal to the curve $\partial U$.

Now suppose $\psi$ is any smooth function that vanishes on $\partial U$, i.e. $\psi(\mbf r) = 0$ for all $\mbf r \in \partial U$. Since $\psi$ is smooth the product $\psi \nabla 1_{U}$ is well-defined as a tempered distribution, and for all test fields $\bs \varphi$ we have
\begin{equation}
\langle \psi \nabla 1_U, \bs\varphi \rangle =\langle \nabla 1_U, \psi\,\bs\varphi \rangle =  -\oint_{\partial U} \psi\,(\bs\varphi \cdot \mbf n)\, ds  = 0.
\label{eq:duality}
\end{equation}
Therefore we have shown \eqref{eq:dist_ann} holds in the distributional sense. The preceeding analysis easily generalizes to piecewise constant functions by linearity: if $f$ is a piecewise constant function \eqref{eq:pwc}, and $\phi$ is any smooth function that vanishes on the edge set $E = \cup_{i=1}^n \partial U_i$, then
\begin{equation}
\phi\,\nabla f = \sum^n_{i=1} a_i\,\underbrace{(\phi\,\nabla 1_{U_i})}_{0} = 0.
\label{eq:spacedom2}
\end{equation}

Now assume that $\phi=\mu$ is a trigonometric polynomial, and hence is described by a finite number of Fourier coefficients, which we denote by $(c[\mbf k] : \mbf k \in \Lambda)$, where $\Lambda\subseteq\mathbb{Z}^2$ is its finite Fourier support set. Taking Fourier transforms of \eqref{eq:spacedom2}, and applying the convolution theorem for tempered distributions yields the following annihilation relation:
\begin{proposition}
Let $f = \sum_{i=1}^n a_i\, 1_{U_i}$ be piecewise constant where the edge set $E = \cup_{i}^n \partial U_i$ is a subset of the trigonometric curve $\{\mu = 0\}$. Then the Fourier coefficients of the gradient of $f$ are annihilated by convolution with the Fourier coefficients $(c[\mbf k] : \mbf k \in \Lambda)$ of $\mu$, i.e.\ 
\begin{equation}
  \sum_{\mbf k \in \Lambda} c[\mbf k]~ \widehat{\nabla f}[\bs \ell - \mbf k] = \bs 0,
   ~~\forall ~\bs \ell \in \mathbb{Z}^2,
   \label{eq:annihilation}
\end{equation}
where $\widehat{\nabla f}[\mbf k] = (\widehat{\partial_x f}[\mbf k], \widehat{\partial_y f}[\mbf k])$.
\end{proposition}
Due to the above property, we call $\mu$ an \emph{annihilating polynomial} for $f$, and the coefficients $(c[\mbf k] : \mbf k \in \Lambda)$ an \emph{annihilating filter}. We show in \S\ref{sec:edgerecovery} that the above Fourier domain annihilation relation can be used to recover an annihilating polynomial, and hence the edge set, from finitely many low-pass Fourier coefficients of the image.

We note that a similar annihilation relationship to \eqref{eq:annihilation} was derived for piecewise complex analytic signals in \cite{pan2013sampling}. Specifically, in \cite{pan2013sampling} it is shown that functions $f(z)=f(x+jy)$ complex analytic on a region $U$ with boundary $\partial\Omega$ given by a trigonometric curve $\{\mu=0\}$ satisfy the annihilation relation
\begin{equation}
\mu\,\partial_{\overline{z}}f = 0,
\label{eq:friann}
\end{equation}
where $\partial_{\overline{z}}:=\partial_x + j\partial_y$ is the complex Wirtinger derivative. While the proofs in \cite{pan2013sampling} rely on a representation of $f$ as a complex curve integral, \eqref{eq:friann} can also be understood in terms of distributions. Also note that if $f$ is piecewise constant, then it satisfies \eqref{eq:friann}. The main difference with our approach is that we assume that both partial derivatives of $f$ are jointly annihilated by $\mu$. This rules out functions $f$ that are non-constant inside the regions defined by the curve $\{\mu=0\}$.
\subsection{Properties of the edge set model: Trigonometric curves}
\label{sec:trigcurves}
Our image model, which assumes the edge set coincides with a trigonometric curve \eqref{eq:trigpoly}, may seem overly restrictive at first. However, trigonometric curves can represent a wide variety of curve topologies, even with relatively few coefficients. For instance, trigonometric curves can have multiple connected components, non-smooth cusps, and can self-intersect in complicated ways. Moreover, any compactly supported plane curve can be approximated to an arbitrary degree of accuracy by a trigonometric curve, provided the number of coefficients is large enough; see Fig.\ \ref{fig:trigcurves}.

Trigonometric curves can be characterized in terms similar to algebraic plane curves, i.e.\ the zero sets of polynomials in two complex variables. In particular, one can associate with every algebraic plane curve a unique minimal degree polynomial whose zero set generates the curve. Similarly, we prove there is a unique minimal degree trigonometric polynomial associated with any trigonometric curve $C$, which we call the \emph{minimal polynomial} for $C$. Here we define the \emph{degree} of a trigonometric polynomial $\mu$ to be the dimensions of the smallest rectangle that contains the frequency support set ${\Lambda}$, which we denote as $deg(\mu) = (K,L)$. In Appendix \ref{sec:appendix_trig} we prove the following:
\begin{proposition}
\label{prop:minpoly}
For every trigonometric curve $C$ there is a unique (up to scaling) real-valued trigonometric polynomial $\mu_0$ with $C=\{\mu_0 = 0\}$ such that for any other trigonometric polynomial $\mu$ with $C=\{\mu = 0\}$ we have $deg(\mu_0) \leq deg (\mu)$ componentwise.
\end{proposition}

This proposition shows trigonometric curves can always be described as the zero set of a \emph{real-valued} trigonometric polynomial, a property we will use implicitly from now on. Moreover, it allows us to uniquely define the \emph{degree} of a trigonometric curve as the degree of its associated minimal polynomial. 
Intuitively, the degree also gives some measure of inherent complexity of a trigonometric curve. For example, the degree puts bounds on how many non-smooth points the curve may have and the number of connected curve components. One particularly important bound for this work is the following:
\begin{proposition}
If $C$ is a trigonometric curve of degree $(K,L)$, then the number of connected components of the complement set $[0,1]^2/C$ is at most $2KL$.
\label{prop:connectedcomp}
\end{proposition}
Consequently, a piecewise constant image with edge set given by a trigonometric curve of degree $(K,L)$ cannot consist of more than $2KL$ distinct regions, i.e.\ the number of terms in the sum \eqref{eq:pwc} is bounded by $2KL$. This means the total degrees of freedom of such an image is at most roughly $3KL$, since we need at most $(K+1)(L+1)$ parameters to describe the edge set, and at most $2KL$ parameters for the region amplitudes.

Some further topological and algebraic properties of trigonometric curves and minimal polynomials can be found in Appendix \ref{sec:appendix_proofs}, which are used in proofs of the sampling theorems in \S\ref{sec:edgerecovery} and \S\ref{sec:imagerecovery}.

\begin{figure*}[ht!]
\centering
\subfloat[][$13$$\times$$13$ coefficients]{
\begin{tikzpicture}[every node/.style={anchor=south west,inner sep=0pt}]
\node[draw=black, thick] (image) at (0,0)
{\includegraphics[height=0.29\textwidth]{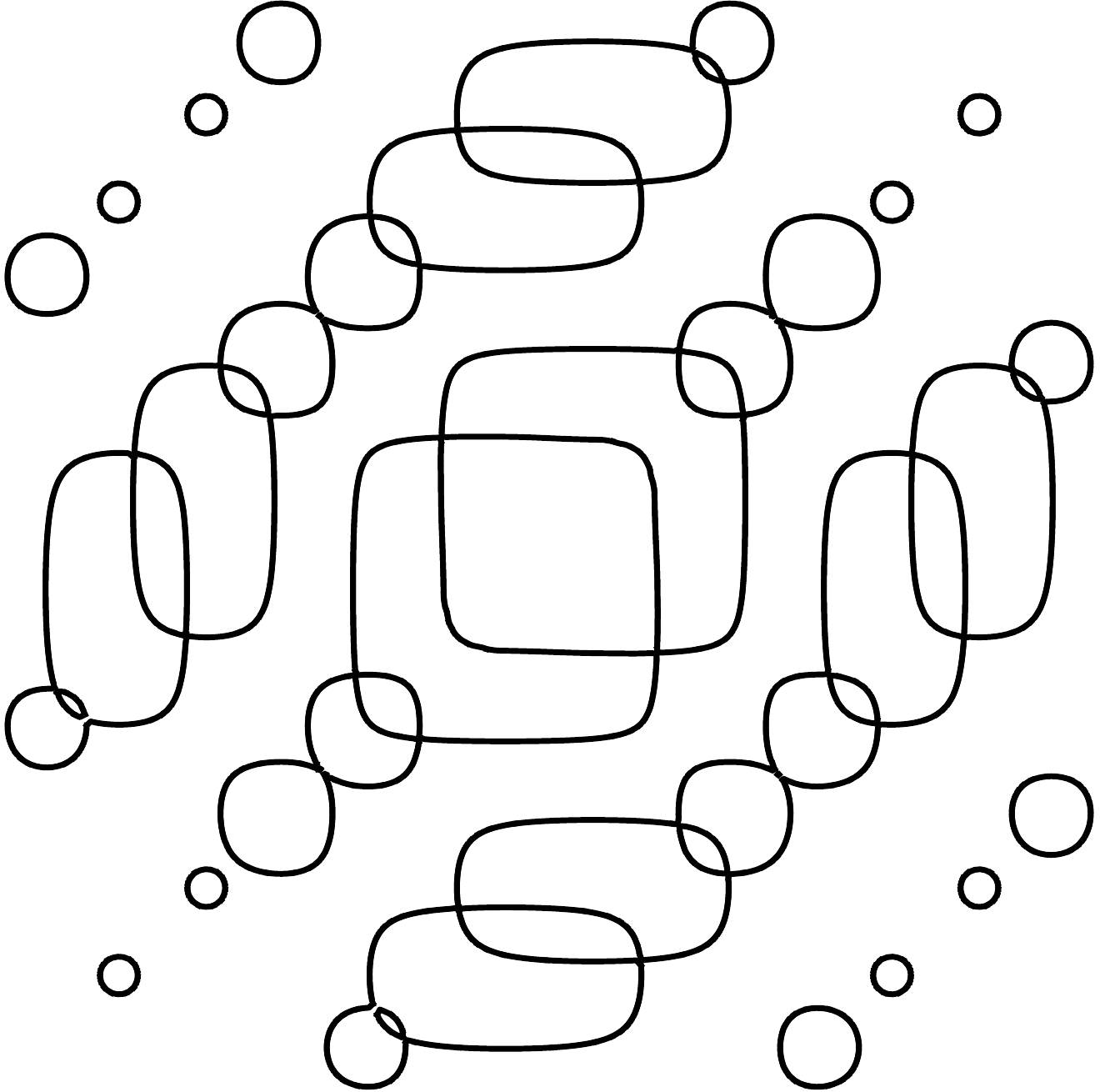}};
\node[draw=black] (fig2) at (2.5,-0.5)
    {\includegraphics[width=0.15\textwidth,trim = 0cm -1cm 0cm 0cm]{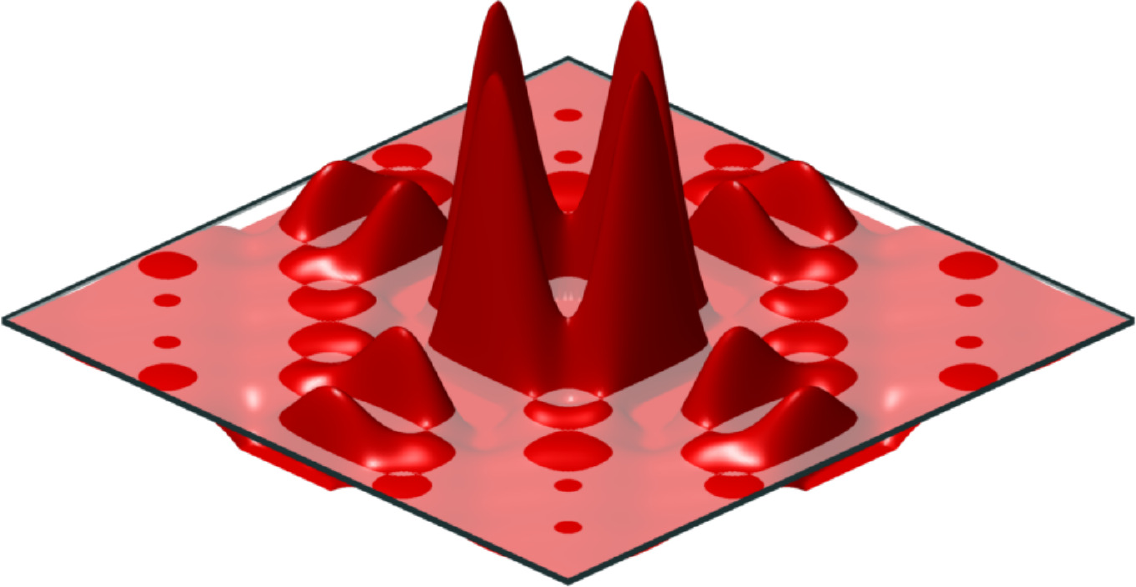}}; 
\end{tikzpicture}
}
\subfloat[][$7$$\times$$9$ coefficients]{
\begin{tikzpicture}[every node/.style={anchor=south west,inner sep=0pt}]
\node[draw=black, thick] (image) at (0,0)
	{\includegraphics[height=0.29\textwidth]{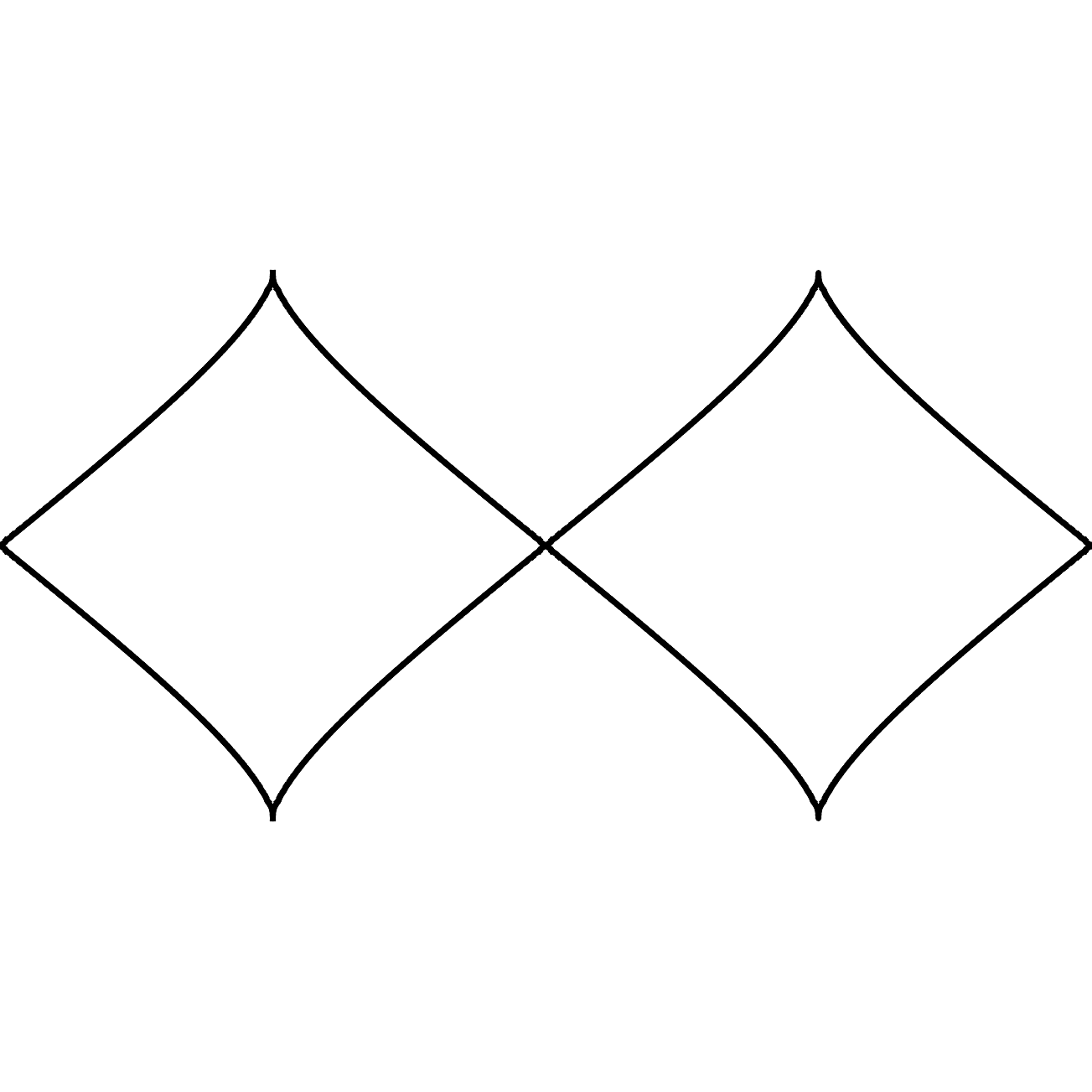}};
\begin{scope}[x={(image.south east)},y={(image.north west)}]
    \draw[red,ultra thick,-latex] (0.40,0.95) -- (0.25,0.78);
    \draw[red,ultra thick,-latex] (0.63,0.75) -- (0.52,0.56);    
\end{scope}
\node[draw=black]  (fig2) at (2.5,-0.5)
    {\includegraphics[width=0.15\textwidth]{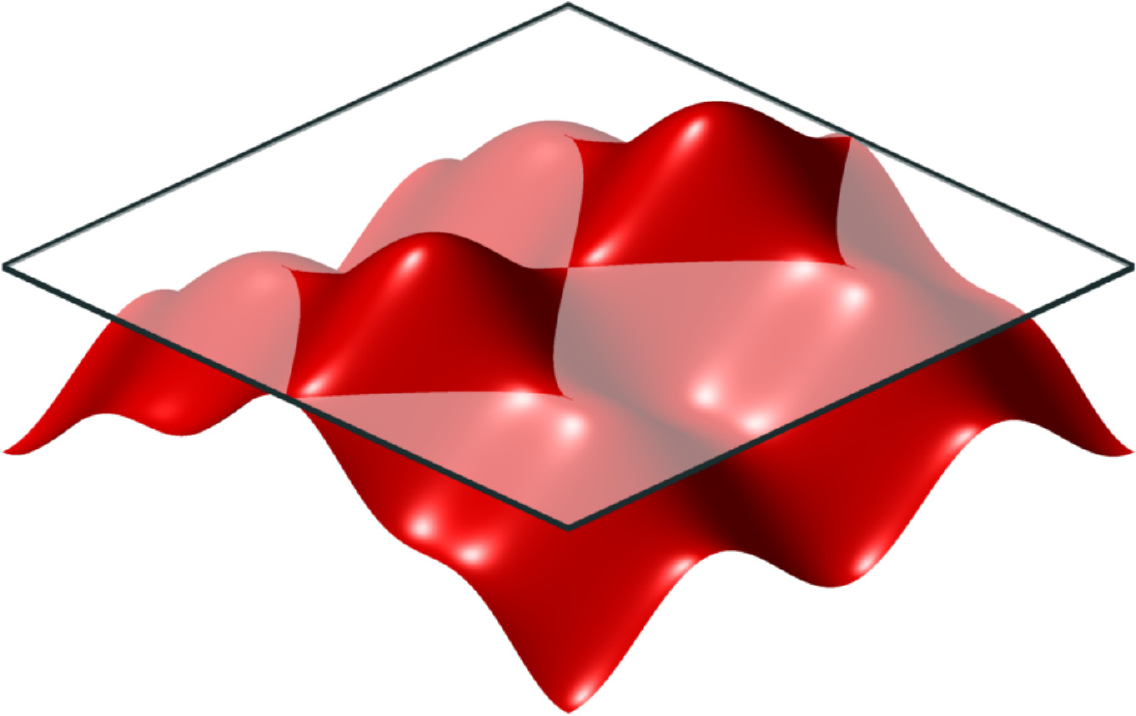}}; 
\end{tikzpicture}
}
\subfloat[][$25$$\times$$25$ coefficients]{
\begin{tikzpicture}[every node/.style={anchor=south west,inner sep=0pt}]
\node[draw=black, thick] (image) at (0,0)
{\includegraphics[height=0.29\textwidth]{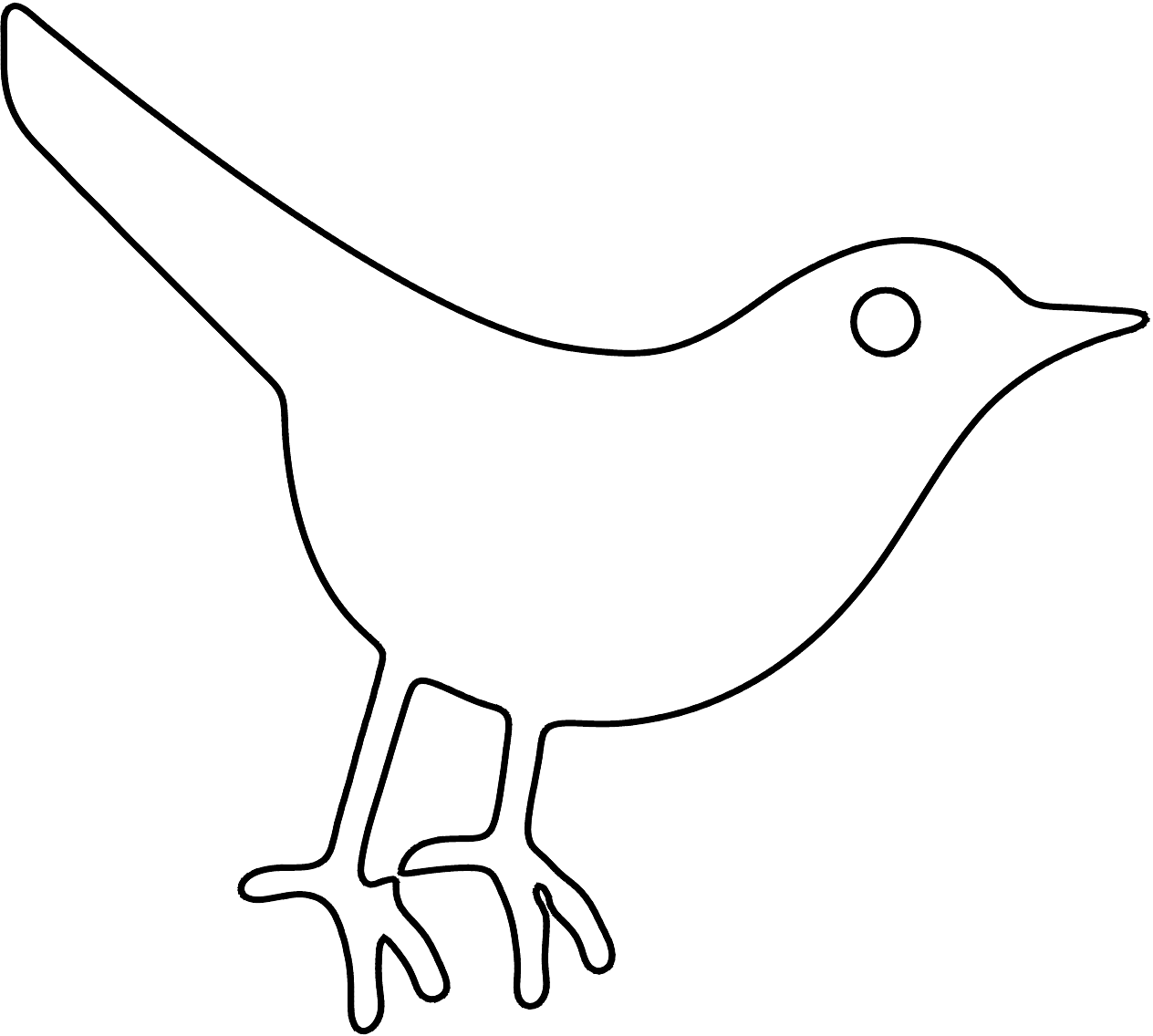}};
\node[draw=black] (fig2) at (3.0,-0.5)
    {\includegraphics[width=0.15\textwidth]{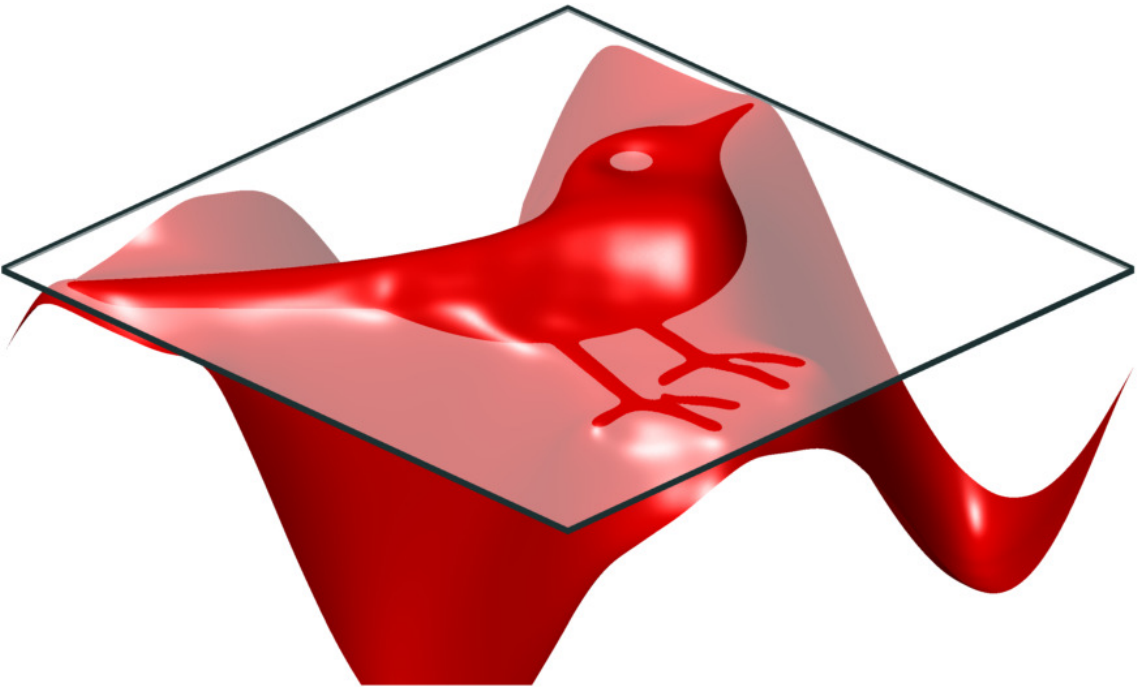}}; 
\end{tikzpicture}
}
\caption{\small \emph{Examples of trigonometric curves.} The one-dimensional zero sets of two-variable trigonometric polynomials are a topologically diverse class of plane curves, which can (a) have many distinct curve components, (b) non-smooth cusps and self-intersections (indicated by arrows), and (c) approximate arbitrary plane curves using few coefficients. Shown inset is a surface plot of the trigonometric polynomial whose zero set gives the corresponding trigonometric curve.}
\label{fig:trigcurves}
\end{figure*}

\section{Recovery of edge set from finite Fourier samples}
\label{sec:edgerecovery}

In this section, we focus on sampling guarantees for the first step of our recovery scheme. We give necessary and sufficient conditions on the number of samples required for unique recovery of the edge set of the signal from the annihilation relation. Similar sampling guarantees for edge-aware recovery of the full signal are provided in the next section.

Let $f$ be any piecewise constant image \eqref{eq:pwc} whose edge set is described by a trigonometric curve $\{\mu = 0\}$ where $\mu$ is a trigonometric polynomial with coefficients $(c[\mbf k] : \mbf k \in \Lambda)$. If we have access to all Fourier samples $\widehat{f}[\mbf k]$ for all $\mbf k \in {\Gamma}$, then from the annihilation relation \eqref{eq:annihilation} we can form the \emph{finite} linear system of equations:
\begin{equation}
\sum_{\mbf k \in {\Lambda}} c[\mbf k]\, \widehat{\nabla f}[\bs\ell - \mbf k] = \bs 0, 
   ~~\forall ~\bs \ell\in \Gamma\,{:}\,\Lambda
   \label{eq:annsys}
\end{equation}
where the index set $\Gamma\,{:}\,\Lambda\subseteq\mathbb{Z}^2$ is the set $\bs\ell\in\mathbb{Z}^2$ such that $\bs\ell-\mbf k \in \Gamma$ for all $\mbf k \in \Lambda$. Assuming $\Lambda$ is symmetric, i.e.\ for all $\mbf k \in \Lambda$ we have $-\mbf k \in \Lambda$, then $\Gamma\,{:}\,\Lambda$ is the simply the set of all integer shifts $\mbf \ell$ such that the translated support $\Lambda+\mbf\ell$ is contained in $\Gamma$. Here the expression $\widehat{\nabla f}[\mbf k] = (\widehat{\partial_x f}[\mbf k], \widehat{\partial_y f}[\mbf k])$ is computed from the known samples using the Fourier domain relations $\widehat{\partial_x f}[\mbf k] = j2\pi k_x \widehat{f}[\mbf k]$, and $\widehat{\partial_y f}[\mbf k] = j2\pi k_y \widehat{f}[\mbf k]$, for any $\mbf k = (k_x,k_y) \in \mathbb{Z}^2$. Provided we have enough samples of $\widehat{f}$ to construct a determined linear system of equations \eqref{eq:annsys}, we can potentially solve for the unknown coefficients $(c[\mbf k] : \mbf k \in \Lambda)$. We now derive conditions for the necessary and sufficient number of samples needed for unique recovery via the linear system \eqref{eq:annsys}. 

\subsection{Necessary condition for recovery of edge set}
Suppose we know the exact Fourier support $\Lambda\subseteq\mathbb{Z}^2$ of the minimal polynomial $\mu$ describing the edge set of the image, and let $\Gamma\subseteq\mathbb{Z}^2$ be an arbitrary sampling set. To allow for the unique recovery of the coefficients $(c[k]:\mbf k \in \Lambda)$ of $\mu$ by solving \eqref{eq:annsys}, we need enough equations so that the nullspace of the system is at most one-dimensional. This means we need the number of equations to match $|\Lambda|-1$. For every shift of the filter support $\Lambda$ contained in the sampling set $\Gamma$, i.e.\ every $\mbf \ell \in \Gamma\,{:}\,\Lambda$, we obtain two equations corresponding to each partial derivative. Therefore, a necessary condition for the recovery of the $(c[k]:\mbf k \in \Lambda)$ is
\begin{equation}
\label{eq:numshifts}
2|\Gamma\,{:}\,\Lambda| \geq |\Lambda|-1.
\end{equation}
However, in general we will not know in advance the exact Fourier support of the minimal polynomial. To gain some intuition on the number of samples required to recover the signal, we focus on the case of rectangular Fourier supports. In particular, we assume the degree of the minimal polynomial is known, i.e.\ the dimensions of the smallest rectangle containing the true Fourier support. Recall that the degree also gives some measure of the complexity of the edge-set (see Prop. \ref{prop:minpoly} \& Prop. \ref{prop:connectedcomp}). To recover the minimal polynomial coefficients, we can instead solve \eqref{eq:annsys} for all coefficients within the minimal rectangular support $\Lambda_R$. When the sampling set $\Gamma$ is also rectangular, by counting shifts of $\Lambda_R$ in $\Gamma$, the expression in \eqref{eq:numshifts} reduces to the following necessary condition:
\begin{proposition}
\label{prop:necessity}
Let $f$ be piecewise constant with edge set described by a trigonometric curve of degree $(K,L)$. A necessary condition to recover the edge set from the annihilation equations \eqref{eq:annsys} is to collect samples of $\widehat{f}$ on a $(K'+1)\times(L'+1)$ rectangular grid in $\mathbb{Z}^2$ with $K'\geq K$ and $L'\geq L$ such that $$2\,(K'-K+1)(L'-L+1) \geq (K+1)(L+1)-1.$$
\end{proposition}
To illustrate this bound, suppose $\mu$ has degree $(K,L)$, and we take Fourier samples from a rectangular grid of size $(1+\alpha) K \times (1+\alpha) L$, $\alpha>0$. Then it is necessary that $\alpha \geq \frac{1}{\sqrt{2}}$, or that we collect roughly $3KL$ Fourier samples to recover the edge set. 

Our numerical experiments on simulated data (see Fig.\ \ref{fig:edgerecovery}) indicate the above necessary conditions might also be \emph{sufficient} for unique recovery; that is, we hypothesize the coefficients of the minimal polynomial are the only non-trivial solution to the system of equations \eqref{eq:annsys} whenever it is fully determined.

\begin{figure*}[!ht]
\centering
\subfloat[Original image]{\includegraphics[height=0.25\textwidth]{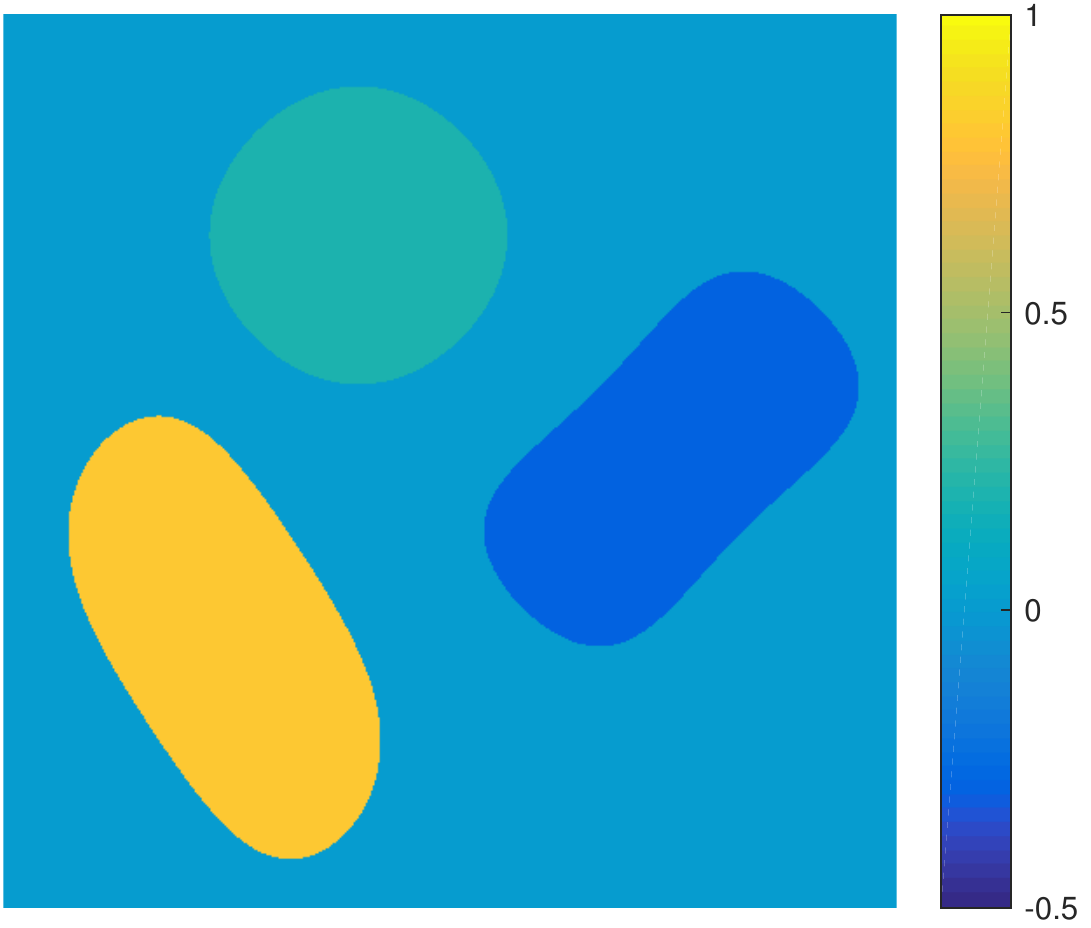}}~
\subfloat[Recovery from necessary number of samples]{\includegraphics[height=0.25\textwidth]{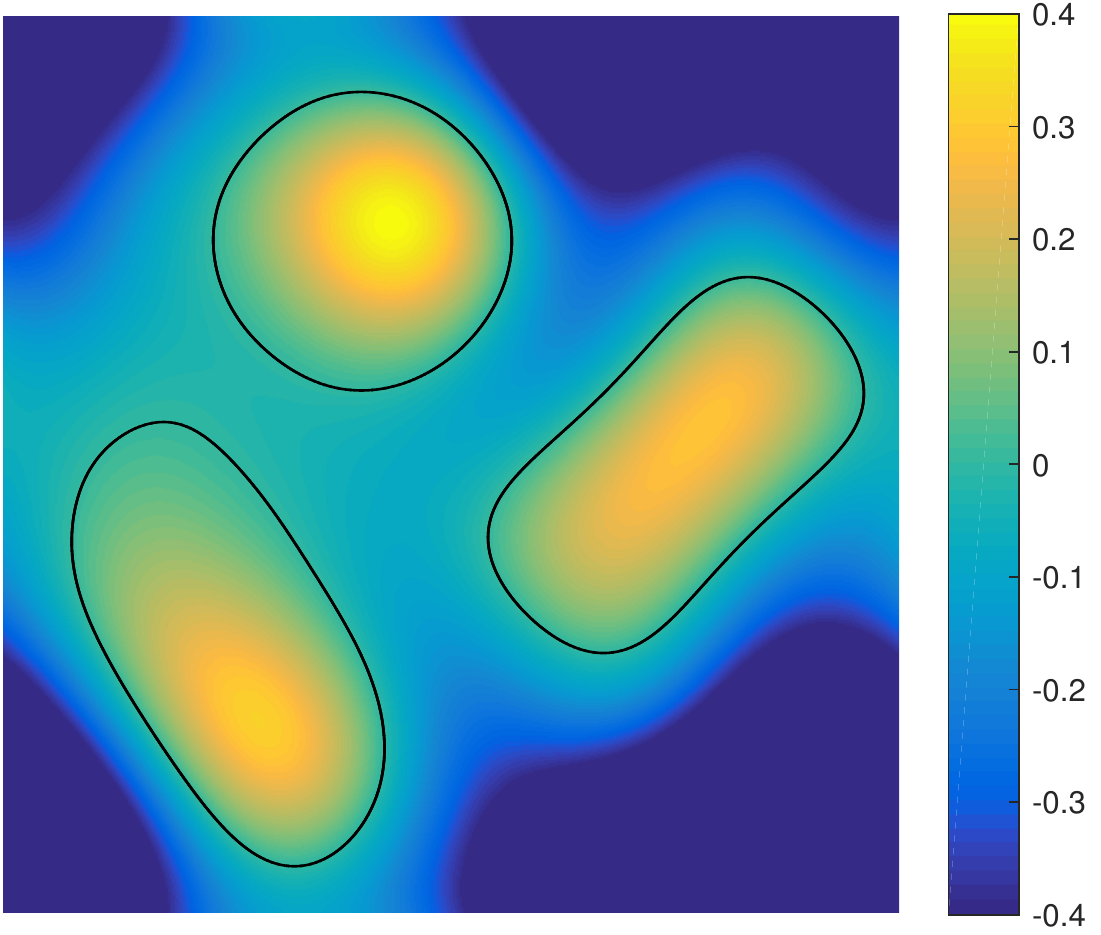}}~
\subfloat[Recovery from fewer than necessary samples]{\includegraphics[height=0.25\textwidth]{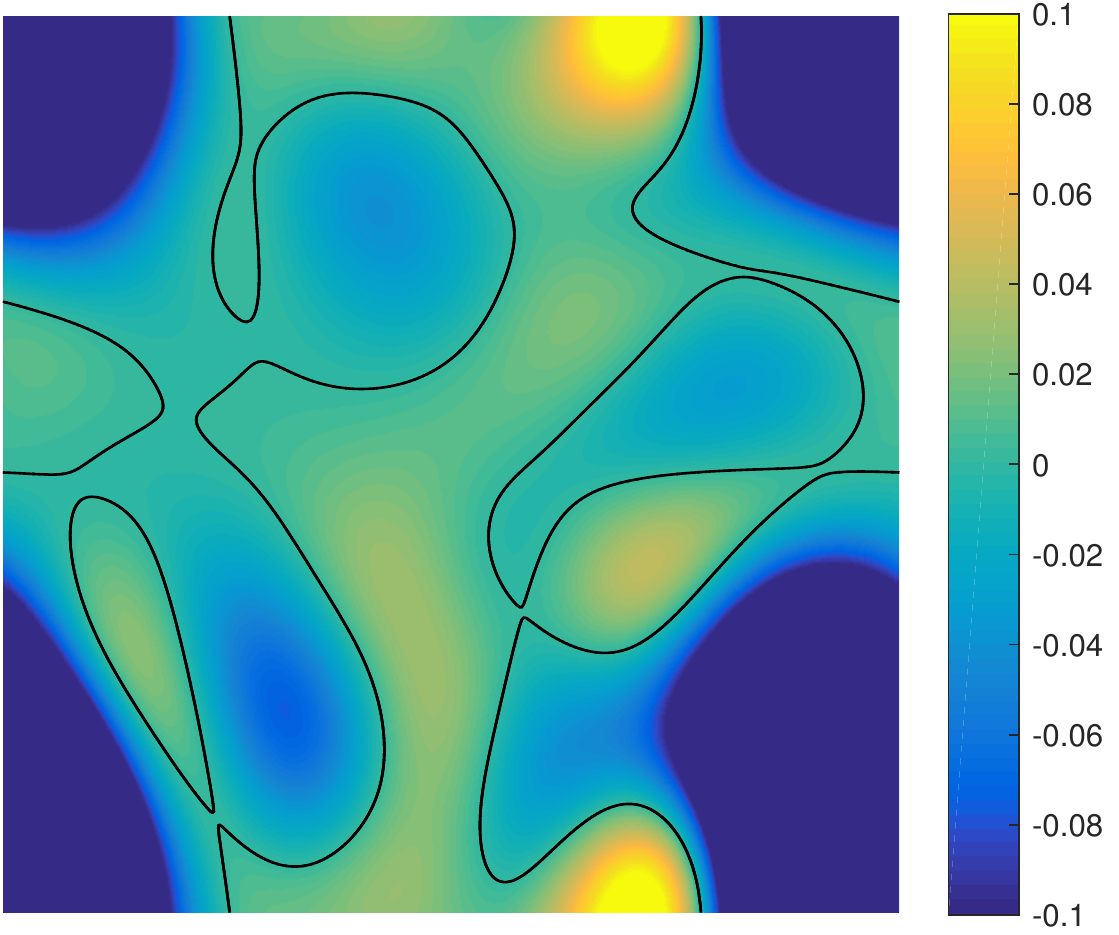}}\\
\subfloat[Original image]{\includegraphics[height=0.25\textwidth]{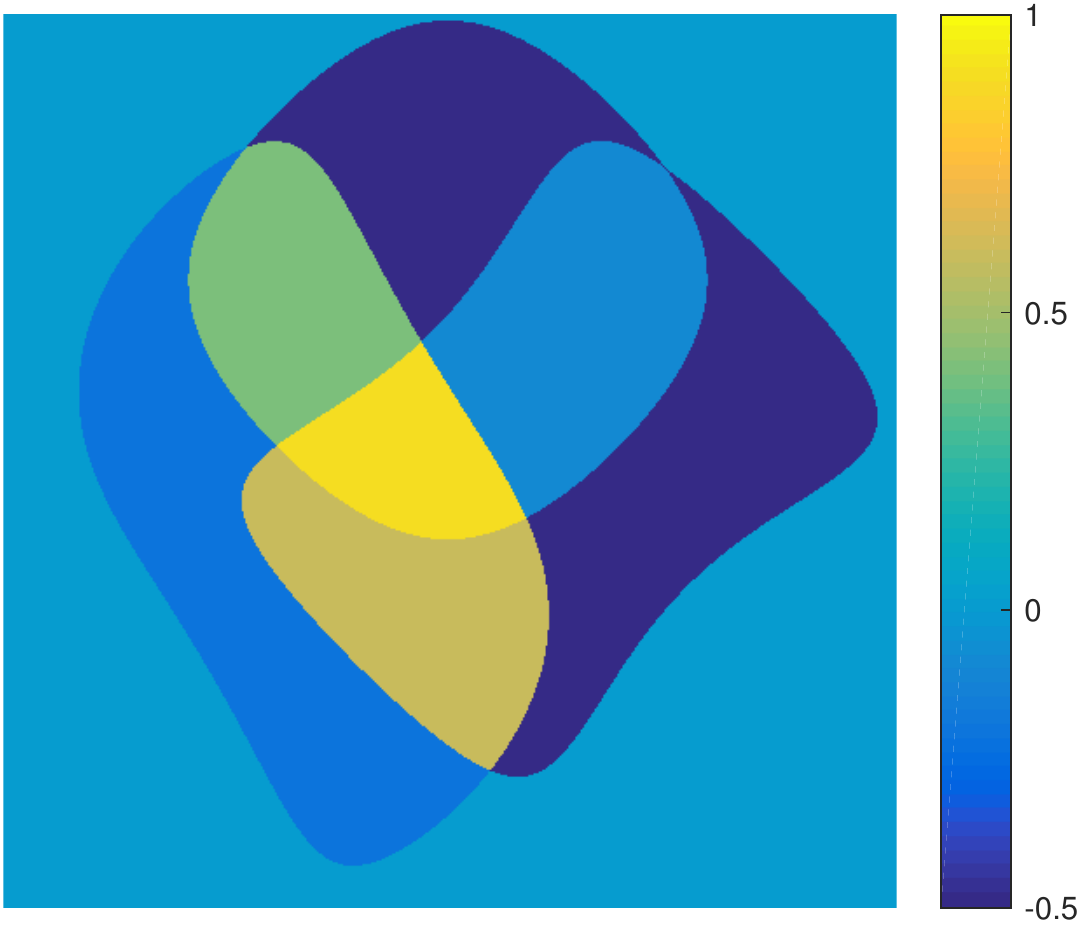}}~
\subfloat[Recovery from necessary number of samples]{\includegraphics[height=0.25\textwidth]{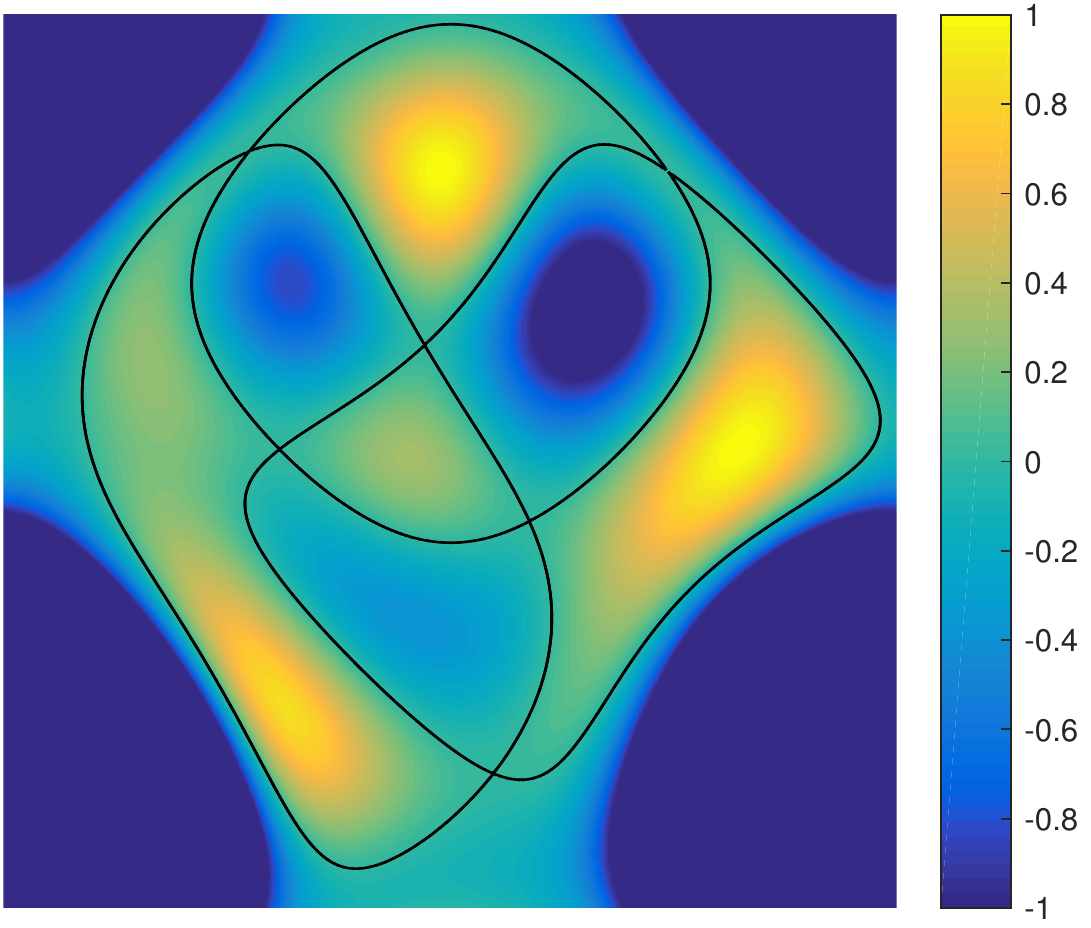}}~
\subfloat[Recovery from fewer than necessary samples]{\includegraphics[height=0.25\textwidth]{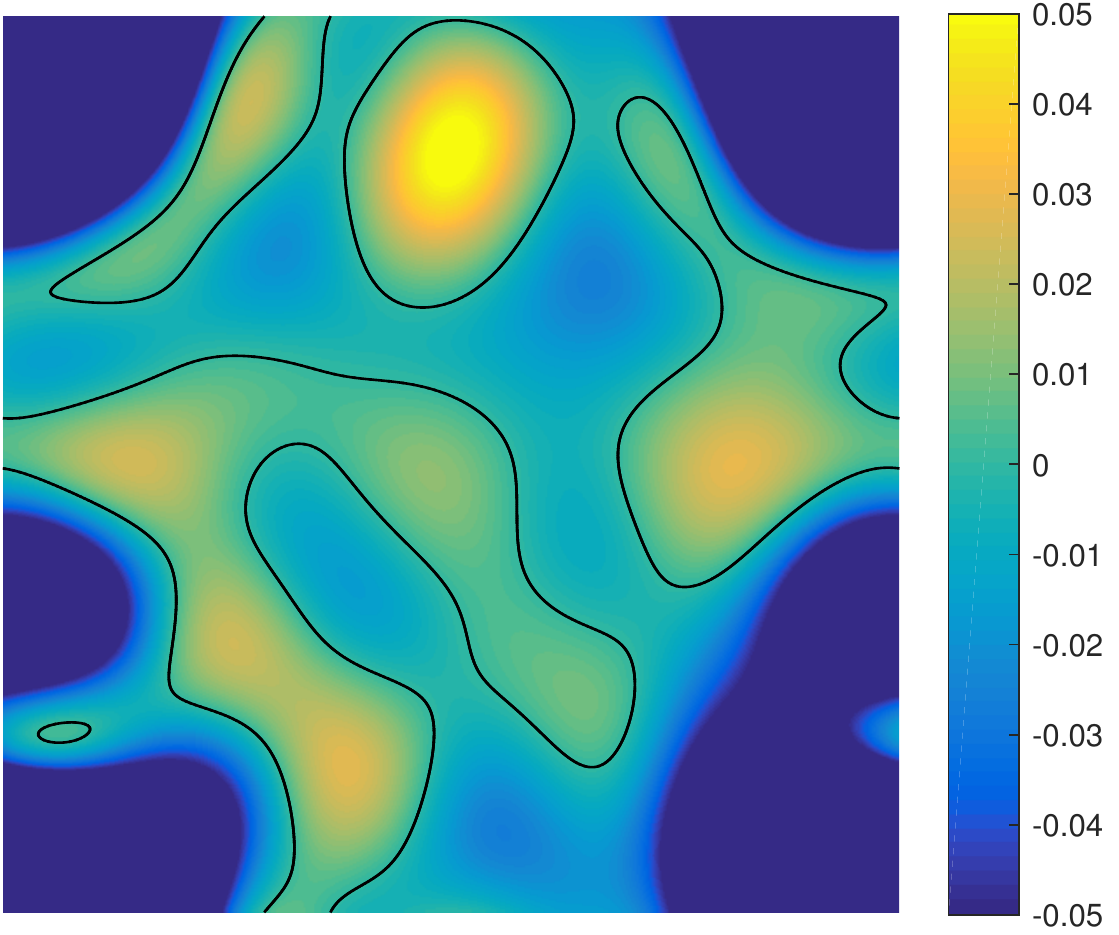}}\\
\caption{\small \emph{Exact recovery of edge set from the minimum necessary number of Fourier samples}. The piecewise constant images shown in (a)\&(d) were generated to have edge set given by the trigonometric curve $\{\mu =0\}$ of degree $(6,6)$, i.e.\ the known minimal polynomial $\mu$ is described by $7$$\times$$7$ coefficients. We recover $\mu$ by solving the linear system \eqref{eq:annsys} built from $11\times11$ Fourier samples of the original image, which were approximated numerically to high precision. This is the necessary minimum number of samples specified by Prop. \ref{prop:necessity}, and is less than the sufficient number of samples specified by Thm.\ \ref{thm:unique2} ($19$$\times$$19$ samples). The recovered minimal polynomial $\mu$ is shown in (b)\&(e) with $\{\mu=0\}$ plotted in black; we observe the recovered coefficients match the known minimal polynomial coefficients to within machine precision. In (c)\&(f) we show a solution obtained from fewer than the necessary minimum number of samples. Note that in this case the true edge set does not coincide with the zero set of the recovered trigonometric polynomial.}
\label{fig:edgerecovery}
\end{figure*}

\subsection{Sufficient conditions for recovery of edge set}
We now focus on sufficient conditions for the recovery of the edge set of a piecewise constant image, or equivalently, the coefficients of its minimal polynomial, from \eqref{eq:annsys}. 
For simplicity, we first state our result in the case of a single characteristic function. We give the proof of this result, and all the remaining proofs in \S\ref{sec:edgerecovery}--\ref{sec:alg}, in Appendix \ref{sec:appendix_proofs}.

\begin{theorem}
\label{thm:unique1}
Let $f = 1_U$ be the characteristic function of a simply connected region $U$ with boundary $\partial U = \{\mu = 0\}$, where $\mu$ is the minimal polynomial. Then the coefficients  $\mbf c = (c[\mbf k]:\mbf k \in {\Lambda})$ of $\mu$ can be uniquely recovered (up to scaling) from samples of $\widehat{f}$ in $3\Lambda$ as the only non-trivial solution to the equations
\begin{equation}
\sum_{\mbf k \in {\Lambda}} c[\mbf k]\, \widehat{\nabla f}[\bs\ell - \mbf k] = \bs 0,
   ~~\text{for all}~~\bs\ell \in 2 {\Lambda}.
  \label{eq:unique1}  
\end{equation}
\end{theorem}
The proof of Theorem \ref{thm:unique1} entails showing any other trigonometric polynomial $\eta$ having coefficients $(d[\mbf k] : \mbf k \in {\Lambda})$ satisfying \eqref{eq:unique1} must vanish on $\partial U$. From this it follows that $\eta$ is a scalar multiple of $\mu$ by properties of minimal polynomials. We now extend the above result to piecewise constant signals, provided the boundaries of the regions do not intersect:
\begin{theorem}
\label{thm:unique2}
Let $f = \sum_{i=1}^n a_i\,1_{U_i}$ be piecewise constant such that the boundary curves $\partial U_i, i=1,\ldots,n$, are connected, pairwise disjoint, and given by trigonometric curves: $\partial U_i = \{\mu_i = 0\}$, where $\mu_i$ is the minimal polynomial. Then, the coefficients $\mbf c = (c[\mbf k]:\mbf k \in {\Lambda})$ of $\mu=\mu_1\cdots\mu_n$, and equivalently the edge set $E=\cup_{i=1}^n \partial U_i$, can be uniquely recovered (up to scaling) from samples of $\widehat{f}$ in $3\Lambda$ as the only non-trivial solution of
\begin{equation}
\sum_{\mbf k \in {\Lambda}} c[\mbf k]\, \widehat{\nabla f}[\bs\ell - \mbf k] = \bs 0,
   ~~\text{for all}~~\bs \ell \in 2\Lambda.
   \label{eq:unique2}
\end{equation}
\end{theorem}
We remark that the proof of Theorem \ref{thm:unique2} can be adapted to many instances where the boundary curves are allowed to intersect, provided we take slightly more samples than $3\Lambda$. The number of additional samples required depends on the precise intersection geometry, and is difficult to characterize for arbitrary curves, and so we do not state a general result here.

Note the sampling requirement of $3\Lambda$ in Theorems \ref{thm:unique1} and \ref{thm:unique2} is greater than the necessary number of samples given in Prop.\ \ref{prop:necessity}. For example, if the edge set has degree $(K,L)$, the necessary number of samples is roughly $3KL$ whereas the sufficient number is roughly $9KL$. In other words, there is a gap in our results between the necessary and sufficient number of samples needed for unique recovery of the edge set. We conjecture that Theorems \ref{thm:unique1} and \ref{thm:unique2} can in fact be sharpened to the necessary number of samples, and extended to edge sets with arbitrary intersection.

Finally, we provide a generalization of Theorem \ref{thm:unique2} that will be useful in practical applications when the degree of the edge set is unknown. Specifically, assume that the system \eqref{eq:annsys} is built with a larger coefficient support set $\Lambda'$ than the minimal support $\Lambda$. In this case the solution to \eqref{eq:annsys} is not necessarily unique: any multiple of the minimal polynomial $\eta = \mu \gamma$ such that $\gamma$ has coefficients in $\Lambda':\Lambda$ is also a solution. This is because $\eta$ satisfies $\{\mu = 0\} \subseteq \{\eta = 0\}$, hence is also an annihilating polynomial. We now show that \emph{every} solution to \eqref{eq:annsys} must be the coefficients of a multiple of the minimal polynomial, under the same restriction as Theorem \ref{thm:unique2}:
\begin{theorem}
\label{thm:unique_3}
Let $f$ be as in Theorem \ref{thm:unique2}. Suppose $\mu = \mu_1\cdots\mu_n$ has coefficients in ${\Lambda}$, and let $\Lambda'$ be another index set with ${\Lambda} \subseteq \Lambda'$. If the coefficients $(d[\mbf k] : \mbf k \in {\Lambda'})$ of a trigonometric polynomial $\eta$ are a non-trivial solution of
\begin{equation}
\sum_{\mbf k \in {\Lambda'}} d[\mbf k]\, \widehat{\nabla f}[\bs\ell - \mbf k] = \bs 0,
   ~~\text{for all}~~~\bs \ell \in \Lambda' + \Lambda,
   \label{eq:unique3}
\end{equation}
then $\mu$ divides $\eta$, that is $\eta = \mu\,\gamma$ where $\gamma$ is another trigonometric polynomial with coefficients supported in $\Lambda':\Lambda$.
\end{theorem}

\section{Edge set aware recovery of the image}
\label{sec:imagerecovery}
In this section we establish conditions for the unique recovery of a piecewise constant image from few of its low-pass Fourier samples with knowledge of the minimal polynomial describing its edge set. We will study the problem of identifiability of the signal in continuous domain: given ideal Fourier samples $\widehat{f}[\mbf k], \mbf k \in {\Gamma}$, and the minimal polynomial $\mu$, when is $f$ is the unique function satisfying the annihilation relation, $\mu \nabla f = 0$?
Here we will restrict solutions $f$ to belong to $L^1([0,1]^2)$, and interpret $\nabla f$ in a a distributional sense. First, we show any function $g\in L^1([0,1]^2)$ satisfying the annihilation relation can be compactly parametrized as a piecewise constant function:
\begin{proposition}
\label{thm:piecewisecst}
Let $\mu$ be any trigonometric polynomial. Suppose $g\in L^1([0,1]^2)$ satisfies the annihilation relation
\begin{equation}
\mu\, \nabla g = 0
\end{equation}
where equality holds in the distributional sense. Then, $g$ can be parameterized as the piecewise constant function
\begin{equation}
\label{eq:parametric}
g = \sum_{i=1}^n b_i\, 1_{U_i}
\end{equation}
almost everywhere, where $b_i \in \mathbb C$ and $U_i$ are the connected components of $\{\mu=0\}^C$.
\end{proposition}

The above proposition implies that the full recovery of the piecewise constant image now simplifies to the estimation of the unknown amplitudes $b_i$, provided the regions $U_i$ are known. One approach of recovering the signal then would be to substitute the parametric model \eqref{eq:parametric} back into the annihilation system \eqref{eq:annsys}, and solve for the unknown amplitudes, analogous to the amplitude recovery step in Prony's method. However, this direct approach is not feasible in practice, since to identify the regions $U_i$ would require us to factor a high degree multivariate polynomial having arbitrary complex coefficients. For this reason, obtaining an exact off-the-grid parametric representation of the full signal in spatial domain is especially challenging in our setting. 

We avoid this direct approach and instead pose recovery as the solution to the following optimization problem in spatial domain:
\begin{equation}
\min_{g\in L^1([0,1]^2)} \|\mu \nabla g \| ~~~ \mbox{such that}~~ \hat g[\mbf k] =  \hat f[\mbf k], \text{ for all } \mbf k\in\Gamma,
\label{eq:recovery_space}
\end{equation}
where $\|\cdot\|$ is any suitable norm. Alternatively, we can solve for the solution entirely in Fourier domain:
\begin{equation}
\min_{\widehat g[\mbf k],\mbf k\in\mathbb{Z}^2}~ \left\|\mbf c\ast\widehat{\nabla g}\right\| ~~~ \mbox{such that}~~ \hat g[\mbf k] =  \hat f[\mbf k], \text{ for all } \mbf k\in\Gamma.
\label{eq:recovery_Fourier}
\end{equation}
where $\mbf c = (c[\mbf k] : \mbf k \in \Lambda)$ are the finitely supported coefficients of the minimal polynomial $\mu$, $\ast$ denotes discrete 2-D convolution, and again $\|\cdot\|$ is some suitable norm. In the formulation \eqref{eq:recovery_Fourier}, the solution is an extrapolation in Fourier domain from the known samples in $\Gamma$ to the unknown coefficients, which can be viewed as a linear prediction in 2-D. Note that \eqref{eq:recovery_space} and \eqref{eq:recovery_Fourier} are still infinite-dimensional optimization problems, and not solvable in practice. However, \eqref{eq:recovery_space} and \eqref{eq:recovery_Fourier} will motivate the practical algorithms we introduce in \S\ref{sec:alg}. We now derive necessary and sufficient conditions for the unique recovery of the signal from \eqref{eq:recovery_space} or equivalently \eqref{eq:recovery_Fourier}.
\subsection{Necessary conditions}
\label{sec:necessaryamp}
A piecewise constant function $g$ in the form \eqref{eq:parametric} with known regions $\{U_i\}_{i=1}^n$ has $n$ remaining degrees of freedoms corresponding to the unknown region amplitudes $\{b_i\}_{i=1}^n$. Hence the necessary minimum number of samples for unique recovery of the full signal---using \emph{any} algorithm---is $n$. However, typically this quantity will be much smaller than the number of available samples required to estimate the edge set. For instance, if the edge set of the image is a trigonometric curve of degree $(K,L)$, then by Proposition\ \ref{prop:connectedcomp}, then the number of distinct regions of the image is at most $2KL$. However, by Proposition\ \ref{prop:necessity} the necessary minimum number of samples required to uniquely reconstruct the edge set from the annihilation equations is roughly $3KL$, which is more than the maximum bound on the number of regions.

Here we note the parallels with Prony's method and 1-D FRI theory: To recover a signal with $K$ singularities it is necessary (and sufficient) to collect $2K+1$ uniform Fourier samples \cite{vetterli2002sampling}. In the first stage, the full $2K+1$ samples are used to identify the locations of the singularities. However, in the second stage only $K$ of the $2K+1$ samples are necessary (and sufficient) to recover the signal amplitudes. Likewise, in our setting, once the necessary number of samples is obtained to reconstruct the edge set, no further samples should be needed to recover the amplitudes.
\subsection{Sufficient conditions}
\label{sec:sufficientamp}
Under the same non-intersecting assumption on the edge set of the piecewise constant function $f$ as in the Theorems of \S\ref{sec:edgerecovery}, we can show the low-pass samples of $\widehat{f}$ are sufficient to uniquely identify $f$ from the annihilation relation with knowledge of minimal polynomial. Specifically, we show that $f$ is identifiable from samples of $\widehat{f}$ in $\Lambda$, the coefficient support set of the minimal polynomial:
\begin{theorem}
\label{thm:uniqueamplitudes}
Let $f$ be as in Theorem \ref{thm:unique2}, with known minimal polynomial $\mu$ having coefficients in $\Lambda$, and let the Fourier sampling set $\Gamma \supseteq \Lambda$. If $g \in L^1([0,1]^2)$ satisfies
\begin{equation}
\mu\, \nabla g = 0\quad and\quad \widehat{g}[\mbf k] = \widehat{f}[\mbf k], \mbf k \in {\Gamma}
\label{eq:infiniterecovery}
\end{equation}
then $g = f$ almost everywhere.
\end{theorem}
Again, we conjecture that this result also holds for piecewise constant functions having an edge set described by an arbitrary trigonometric curve $C$, not just those consisting of non-intersecting curves. 

This result shows one benefit of our piecewise constant image model versus the piecewise complex analytic signal model considered in \cite{pan2013sampling}: in the piecewise constant case,  recovery from the annihilation relation given finitely many Fourier samples is well-posed. However, this is not the case for piecewise analytic signals using the annihilation relation \eqref{eq:friann}, as considered in \cite{pan2013sampling}. This is because, in analogy with Proposition \ref{thm:piecewisecst}, one can show
\begin{equation}
\mu\, \partial_{\overline{z}}g = 0
\label{eq:cmplxann}
\end{equation}
if and only if $g$ is complex analytic on each connected component $U_i$ of $\{\mu=0\}^C$. Or, equivalently:
\begin{equation}
g(z) = \sum_{i=1}^n h_i(z) 1_{U_i}(z),
\label{eq:cmplxpara}
\end{equation}
where $h_i$ is any function complex analytic on $U_i$. Signals in this class have infinite degrees of freedom since each $h_i(z)$ is described by infinitely many complex parameters, namely its power series coefficients. Therefore, recovery of the signal $g$ from the annihilation relation \eqref{eq:cmplxann} given finitely many Fourier samples is ill-posed. In particular, for any finite Fourier sampling set $\Gamma$, there are infinitely many signals $g$ in the form \eqref{eq:cmplxpara} satisfying $\widehat{f}[\mbf k] = \widehat{g}[\mbf k]$, for all $\mbf k \in \Gamma$.

\subsection{Recovery from finitely many annihilation equations}
The conditions derived in \S \ref{sec:necessaryamp} and \S\ref{sec:sufficientamp} settle the problem of unique recovery from using the infinite dimensional annihilation relation \eqref{eq:infiniterecovery}. A more subtle question, relevant to practical applications, is whether the Fourier extrapolation solution is uniquely determined from finitely many annihilation equations, i.e.\ for some large extrapolation set $\Delta \subseteq \mathbb{Z}^2$, and some small sampling set $\Gamma\subseteq \mathbb{Z}^2$ do the conditions
\begin{equation}
\label{eq:finiterecovery}
\sum_{\mbf k \in \Lambda} c[\mbf k] ~\nabla g [\bs\ell - \mbf k] = 0~~\text{for all}~~\bs\ell \in\Delta:\Lambda\quad and \quad \widehat{g}[\mbf k] = \widehat{f}[\mbf k]~~\text{for all}~~\mbf k \in {\Gamma}.
\end{equation}
imply $\widehat{g}[\mbf k] = \widehat{f}[\mbf k]$ for all $\mbf k \in \Delta$? In other words, given few samples of $\widehat{f}$ and knowledge of the minimal polynomial coefficients $(c[\mbf k] : \mbf k \in \Lambda)$, can we uniquely extrapolate $\widehat{f}$ to a larger grid $\Delta \subseteq \mathbb{Z}^2$ using the Fourier domain annihilation relation \eqref{eq:finiterecovery}? Note that \eqref{eq:infiniterecovery} may be thought of as the limiting case of \eqref{eq:finiterecovery}, where the extrapolation set $\Delta$ expands to all of $\mathbb{Z}^2$. Our numerical experiments (see Fig.\ \ref{fig:amp_recovery}) indicate that, given the minimal polynomial, unique recovery is in fact possible from \eqref{eq:finiterecovery} with the necessary minimum number of samples given in \S\ref{sec:necessaryamp}. We hypothesize that a sufficient condition analogous to Theorem \ref{thm:uniqueamplitudes} also holds in this case. However, establishing this condition is less straightforward than in the infinite dimensional setting \eqref{eq:infiniterecovery} since in this context we lack the parameterization result \eqref{eq:parametric}. We plan to pursue this question in a future work.
\begin{figure*}[ht!]
\hspace{-1.5em}
\subfloat[][Recovery of $n=4$ regions (including background)]{
\begin{tabular}{cccc}
 & $N=2$ & $N=4$ & \\
\hspace{2em}
& 
\includegraphics[height=0.18\textwidth]{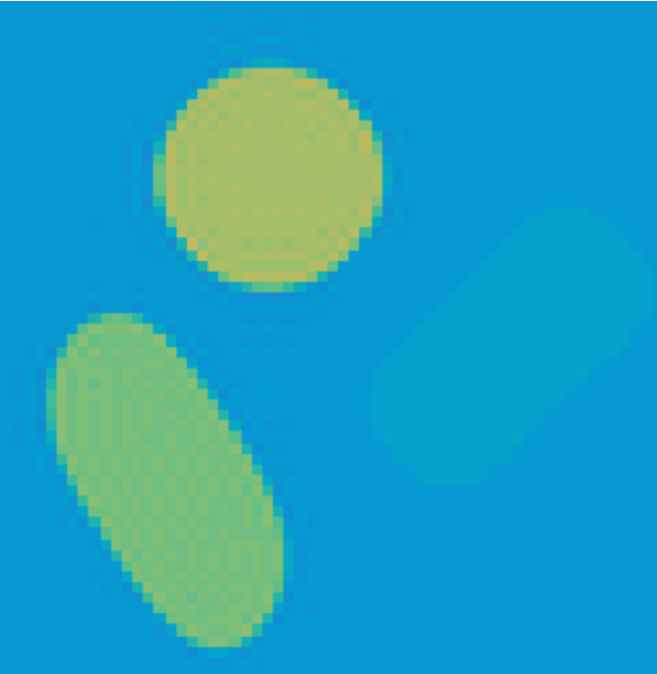}
\hspace{-1em}
&
\includegraphics[height=0.18\textwidth]{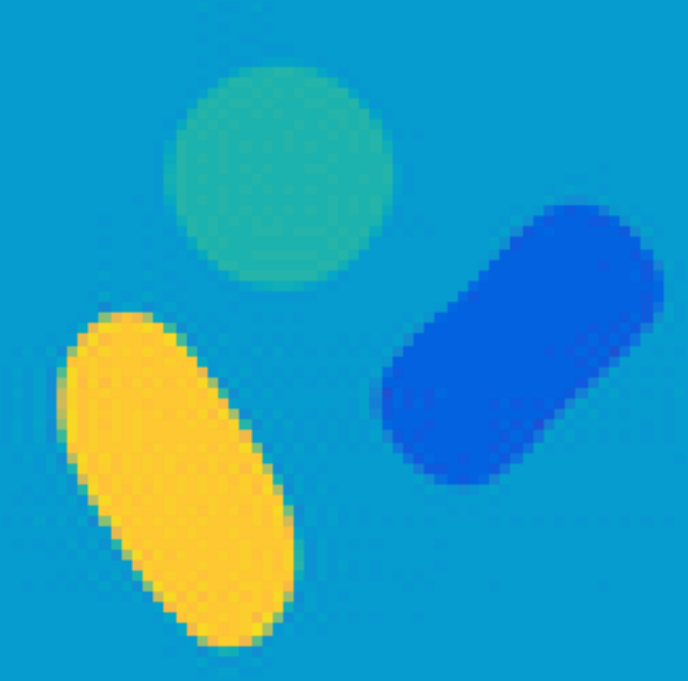}
&
\raisebox{-0.5em}{
\hspace{-2em}
\setlength\figureheight{0.18\textwidth}
%
%
\begin{tikzpicture}

\pgfplotsset{
    colormap={parula}{
        rgb=(0.2081,0.1663,0.5292)
        rgb=(0.2116,0.1898,0.5777)
        rgb=(0.2123,0.2138,0.627)
        rgb=(0.2081,0.2386,0.6771)
        rgb=(0.1959,0.2645,0.7279)
        rgb=(0.1707,0.2919,0.7792)
        rgb=(0.1253,0.3242,0.8303)
        rgb=(0.0591,0.3598,0.8683)
        rgb=(0.0117,0.3875,0.882)
        rgb=(0.006,0.4086,0.8828)
        rgb=(0.0165,0.4266,0.8786)
        rgb=(0.0329,0.443,0.872)
        rgb=(0.0498,0.4586,0.8641)
        rgb=(0.0629,0.4737,0.8554)
        rgb=(0.0723,0.4887,0.8467)
        rgb=(0.0779,0.504,0.8384)
        rgb=(0.0793,0.52,0.8312)
        rgb=(0.0749,0.5375,0.8263)
        rgb=(0.0641,0.557,0.824)
        rgb=(0.0488,0.5772,0.8228)
        rgb=(0.0343,0.5966,0.8199)
        rgb=(0.0265,0.6137,0.8135)
        rgb=(0.0239,0.6287,0.8038)
        rgb=(0.0231,0.6418,0.7913)
        rgb=(0.0228,0.6535,0.7768)
        rgb=(0.0267,0.6642,0.7607)
        rgb=(0.0384,0.6743,0.7436)
        rgb=(0.059,0.6838,0.7254)
        rgb=(0.0843,0.6928,0.7062)
        rgb=(0.1133,0.7015,0.6859)
        rgb=(0.1453,0.7098,0.6646)
        rgb=(0.1801,0.7177,0.6424)
        rgb=(0.2178,0.725,0.6193)
        rgb=(0.2586,0.7317,0.5954)
        rgb=(0.3022,0.7376,0.5712)
        rgb=(0.3482,0.7424,0.5473)
        rgb=(0.3953,0.7459,0.5244)
        rgb=(0.442,0.7481,0.5033)
        rgb=(0.4871,0.7491,0.484)
        rgb=(0.53,0.7491,0.4661)
        rgb=(0.5709,0.7485,0.4494)
        rgb=(0.6099,0.7473,0.4337)
        rgb=(0.6473,0.7456,0.4188)
        rgb=(0.6834,0.7435,0.4044)
        rgb=(0.7184,0.7411,0.3905)
        rgb=(0.7525,0.7384,0.3768)
        rgb=(0.7858,0.7356,0.3633)
        rgb=(0.8185,0.7327,0.3498)
        rgb=(0.8507,0.7299,0.336)
        rgb=(0.8824,0.7274,0.3217)
        rgb=(0.9139,0.7258,0.3063)
        rgb=(0.945,0.7261,0.2886)
        rgb=(0.9739,0.7314,0.2666)
        rgb=(0.9938,0.7455,0.2403)
        rgb=(0.999,0.7653,0.2164)
        rgb=(0.9955,0.7861,0.1967)
        rgb=(0.988,0.8066,0.1794)
        rgb=(0.9789,0.8271,0.1633)
        rgb=(0.9697,0.8481,0.1475)
        rgb=(0.9626,0.8705,0.1309)
        rgb=(0.9589,0.8949,0.1132)
        rgb=(0.9598,0.9218,0.0948)
        rgb=(0.9661,0.9514,0.0755)
        rgb=(0.9763,0.9831,0.0538)
    }
}

\begin{axis}[
    hide axis,
    scale only axis,
    width=1em,
    height=3em,
    point meta min=-0.5,
    point meta max=1,
    colorbar,
    colorbar style={
        height=\figureheight,
        width=0.7em,
        ytick={-0.5,0,0.5,1},
        yticklabel style={
            xshift = -0.5ex,
            font = \tiny
        }
    }]
    \addplot [draw=none] coordinates {(0,0)};
\end{axis}
\end{tikzpicture}%
}\\
\multicolumn{4}{c}{
\hspace{-1.5em}
\setlength\figureheight{0.15\textwidth}
\setlength\figurewidth{0.35\textwidth}
%
%
\begin{tikzpicture}

\begin{axis}[%
width=\figurewidth,
height=\figureheight,
scale only axis,
xmin=1,
xmax=20,
xlabel={$N$ (\# samples)},
ymin=0,
ymax=1,
ylabel={NRMSE},
axis background/.style={fill=white}
]
\addplot [color=blue,solid,mark=*,mark options={solid},forget plot]
  table[row sep=crcr]{%
1	0.965853554057879\\
2	0.770805037259404\\
3	0.578875256076272\\
4	0.000308955268323639\\
5	5.31135894989929e-05\\
6	5.49640562599588e-05\\
7	5.77761464927063e-05\\
8	5.36274389435898e-05\\
9	5.66754078089743e-05\\
10	5.70045004672561e-05\\
11	5.72073273500332e-05\\
12	5.97846948051261e-05\\
13	5.62421546777863e-05\\
14	5.86225463162513e-05\\
15	5.93574247969205e-05\\
16	5.67265899351559e-05\\
17	5.90283231072981e-05\\
18	6.25577640537818e-05\\
19	5.71148616775324e-05\\
20	6.22752438820113e-05\\
21	5.75402820863904e-05\\
22	5.7943351638053e-05\\
23	5.72306694905311e-05\\
24	5.97911501209693e-05\\
25	6.11618910208437e-05\\
26	5.66519395715117e-05\\
27	5.7696515623838e-05\\
28	5.97559977493634e-05\\
29	5.79523122076399e-05\\
30	5.89759353131175e-05\\
31	5.8780300660555e-05\\
32	5.76992249547344e-05\\
33	5.71877649646798e-05\\
34	5.82133638816878e-05\\
35	5.64500562467608e-05\\
36	5.76440686500386e-05\\
37	5.7114725773268e-05\\
38	5.639371188301e-05\\
39	5.53793536499916e-05\\
40	5.55811828107237e-05\\
41	5.69019797739484e-05\\
42	5.68803300013567e-05\\
43	5.72534450009365e-05\\
44	5.69584202879275e-05\\
45	5.7082172992802e-05\\
46	5.76672745262964e-05\\
47	5.79400475656909e-05\\
48	5.40953416136952e-05\\
49	5.18747268794299e-05\\
};
\end{axis}
\end{tikzpicture}%
}
\end{tabular}
}
\hspace{-2em}
\subfloat[][Recovery of $n=8$ regions (including background)]{
\begin{tabular}{cccc}
 & $N=4$ & $N=8$ & \\
\hspace{2em}
& 
\includegraphics[height=0.18\textwidth]{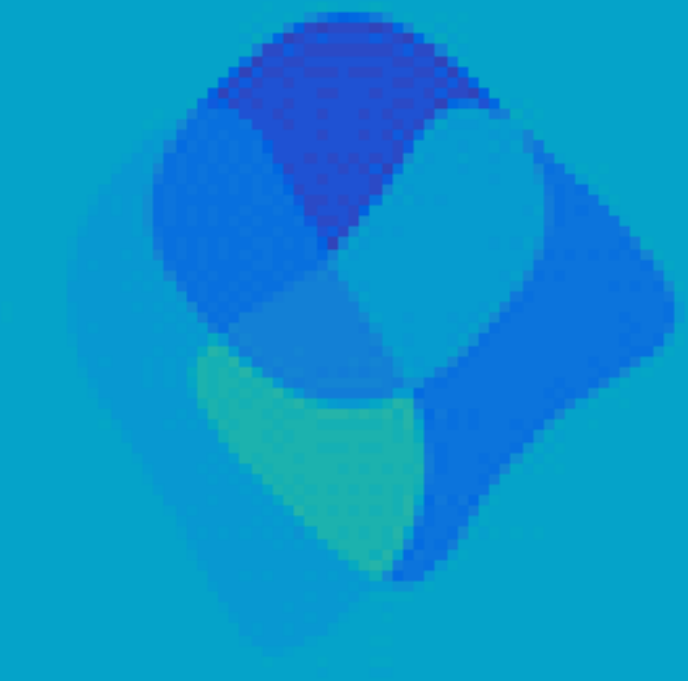}
\hspace{-1em}
& 
\includegraphics[height=0.18\textwidth]{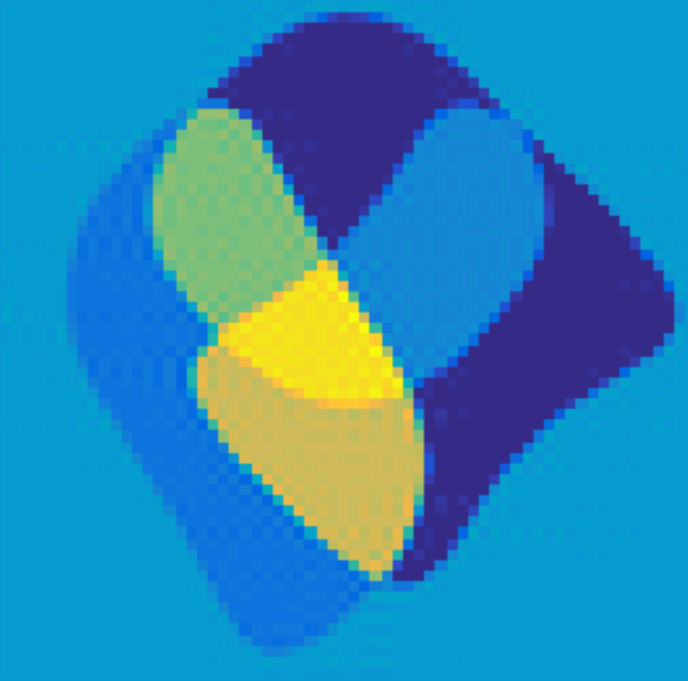}
& 
\raisebox{-0.5em}{
\hspace{-2em}
\setlength\figureheight{0.18\textwidth}
%
%
\begin{tikzpicture}

\pgfplotsset{
    colormap={parula}{
        rgb=(0.2081,0.1663,0.5292)
        rgb=(0.2116,0.1898,0.5777)
        rgb=(0.2123,0.2138,0.627)
        rgb=(0.2081,0.2386,0.6771)
        rgb=(0.1959,0.2645,0.7279)
        rgb=(0.1707,0.2919,0.7792)
        rgb=(0.1253,0.3242,0.8303)
        rgb=(0.0591,0.3598,0.8683)
        rgb=(0.0117,0.3875,0.882)
        rgb=(0.006,0.4086,0.8828)
        rgb=(0.0165,0.4266,0.8786)
        rgb=(0.0329,0.443,0.872)
        rgb=(0.0498,0.4586,0.8641)
        rgb=(0.0629,0.4737,0.8554)
        rgb=(0.0723,0.4887,0.8467)
        rgb=(0.0779,0.504,0.8384)
        rgb=(0.0793,0.52,0.8312)
        rgb=(0.0749,0.5375,0.8263)
        rgb=(0.0641,0.557,0.824)
        rgb=(0.0488,0.5772,0.8228)
        rgb=(0.0343,0.5966,0.8199)
        rgb=(0.0265,0.6137,0.8135)
        rgb=(0.0239,0.6287,0.8038)
        rgb=(0.0231,0.6418,0.7913)
        rgb=(0.0228,0.6535,0.7768)
        rgb=(0.0267,0.6642,0.7607)
        rgb=(0.0384,0.6743,0.7436)
        rgb=(0.059,0.6838,0.7254)
        rgb=(0.0843,0.6928,0.7062)
        rgb=(0.1133,0.7015,0.6859)
        rgb=(0.1453,0.7098,0.6646)
        rgb=(0.1801,0.7177,0.6424)
        rgb=(0.2178,0.725,0.6193)
        rgb=(0.2586,0.7317,0.5954)
        rgb=(0.3022,0.7376,0.5712)
        rgb=(0.3482,0.7424,0.5473)
        rgb=(0.3953,0.7459,0.5244)
        rgb=(0.442,0.7481,0.5033)
        rgb=(0.4871,0.7491,0.484)
        rgb=(0.53,0.7491,0.4661)
        rgb=(0.5709,0.7485,0.4494)
        rgb=(0.6099,0.7473,0.4337)
        rgb=(0.6473,0.7456,0.4188)
        rgb=(0.6834,0.7435,0.4044)
        rgb=(0.7184,0.7411,0.3905)
        rgb=(0.7525,0.7384,0.3768)
        rgb=(0.7858,0.7356,0.3633)
        rgb=(0.8185,0.7327,0.3498)
        rgb=(0.8507,0.7299,0.336)
        rgb=(0.8824,0.7274,0.3217)
        rgb=(0.9139,0.7258,0.3063)
        rgb=(0.945,0.7261,0.2886)
        rgb=(0.9739,0.7314,0.2666)
        rgb=(0.9938,0.7455,0.2403)
        rgb=(0.999,0.7653,0.2164)
        rgb=(0.9955,0.7861,0.1967)
        rgb=(0.988,0.8066,0.1794)
        rgb=(0.9789,0.8271,0.1633)
        rgb=(0.9697,0.8481,0.1475)
        rgb=(0.9626,0.8705,0.1309)
        rgb=(0.9589,0.8949,0.1132)
        rgb=(0.9598,0.9218,0.0948)
        rgb=(0.9661,0.9514,0.0755)
        rgb=(0.9763,0.9831,0.0538)
    }
}

\begin{axis}[
    hide axis,
    scale only axis,
    width=1em,
    height=3em,
    point meta min=-0.5,
    point meta max=1,
    colorbar,
    colorbar style={
        height=\figureheight,
        width=0.7em,
        ytick={-0.5,0,0.5,1},
        yticklabel style={
            xshift = -0.5ex,
            font = \tiny
        }
    }]
    \addplot [draw=none] coordinates {(0,0)};
\end{axis}
\end{tikzpicture}%
}\\
\multicolumn{4}{c}{
\hspace{-1.5em}
\setlength\figureheight{0.15\textwidth}
\setlength\figurewidth{0.35\textwidth} 
%
%
\begin{tikzpicture}

\begin{axis}[%
width=\figurewidth,
height=\figureheight,
scale only axis,
xmin=1,
xmax=20,
xlabel={$N$ (\# samples)},
ymin=0,
ymax=1,
ylabel={NRMSE},
axis background/.style={fill=white}
]
\addplot [color=blue,solid,mark=*,mark options={solid},forget plot]
  table[row sep=crcr]{%
1	0.996371830447585\\
2	0.914173975925766\\
3	0.933495046890445\\
4	0.796994889155684\\
5	0.666813922748149\\
6	0.603303343556904\\
7	0.538333920617707\\
8	0.0235525819205099\\
9	0.0146544257473361\\
10	0.00908077923694146\\
11	0.0046283905379797\\
12	0.00402390365894079\\
13	0.00557619204965219\\
14	0.00431493121825542\\
15	0.00346402849868462\\
16	0.00313457513397094\\
17	0.00267649801819359\\
18	0.00285940002811256\\
19	0.00275482929286738\\
20	0.00278466741149466\\
};
\end{axis}
\end{tikzpicture}%
}
\end{tabular}
}
\caption{\small \emph{Edge set aware recovery of synthetic data}. Piecewise constant images with $n=4$ and $n=8$ distinct connected regions (including the background region) were generated with edge set given by a known minimal annihilating polynomial having $7$$\times$$7$ coefficients (as shown in Fig.\ \ref{fig:edgerecovery}). Choosing $N$ samples uniformly at random from the $7$$\times$$7$ center Fourier coefficients of the image, we extrapolate the image in Fourier domain to $65$$\times$$65$ coefficients by solving \eqref{eq:finiterecovery} in a least-squares fashion. We plot the normalized root mean square error (NRMSE) of the recovery averaged over ten trials. Pictured above the plots is the inverse DFT of a reconstruction obtained from the indicated number of samples $N$. Note that in both cases we are able to recover the image nearly exactly from the necessary minimum number of samples, $N=4$ and $N=8$ respectively, corresponding to the number of regions. Below this number, we still recover the Fourier coefficients of a piecewise constant image, however one with the incorrect region amplitudes.}
\label{fig:amp_recovery}
\end{figure*}
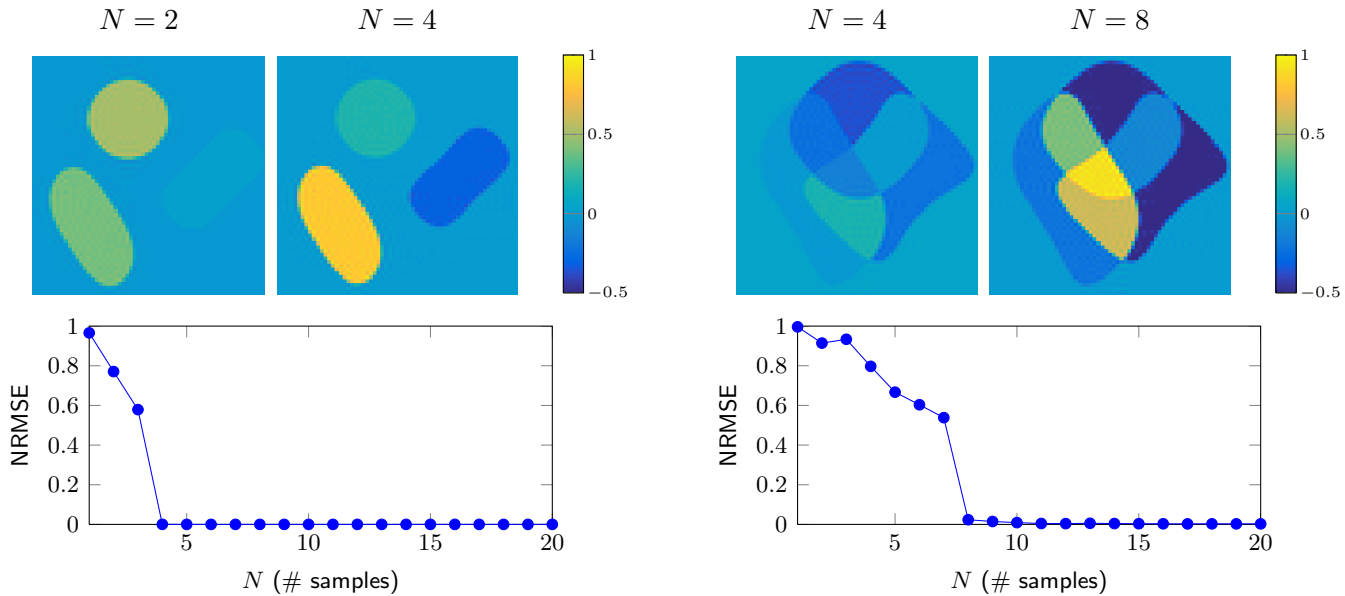

\section{Algorithms}
\label{sec:alg}
We rely on a two step algorithm similar to Prony's method, which involves (1) the estimation of the edge set of an image (2) the edge set aware recovery of the image. Theorems \ref{thm:unique2} and \ref{thm:unique_3} together guarantee the unique recovery of the edge set from the linear annihilation relation \eqref{eq:annsys}, while Theorem \ref{thm:uniqueamplitudes} guarantees the unique recovery of the full piecewise constant image given its edge set. However, these guarantees only hold under ideal conditions, making direct estimation of the signal from \eqref{eq:annsys} infeasible in practice. Here we introduce robust refinements of this procedure to enable recovery of the image in the presence of noise or model-mismatch.

\subsection{Step 1: Estimation of an annihilating subspace}
Directly estimating the edge set via the linear system \eqref{eq:annsys} in practical applications is challenging due to the following problems:
\begin{enumerate}
	\item \emph{Model-order selection}: Theorems \ref{thm:unique2} and \ref{thm:unique_3} assume the knowledge of the degree of the minimal polynomial describing the edge set of the image. However, this is often unknown in practical applications. 
	\item \emph{Noise and model-mismatch}: The available Fourier samples are often corrupted by noise in practical settings. In addition, the assumption that the underlying image is piecewise constant with edge set described by a trigonometric curve will only be approximately true in practice.
\end{enumerate}
We now introduce efficient approaches to overcome the above problems. The key idea is that rather than estimating a single annihilating filter, we estimate an entire \emph{annihilating subspace}---a collection of linearly independent annihilating filters which jointly encode the edge set. This can be thought of as an extension of \emph{subspace methods} in spectral estimation to the 2-D setting. Specifically, we draw parallels between our algorithm and the well-known MUSIC algorithm \cite{schmidt1986multiple}. This gives us an efficient means to address the model-order selection problem. We also introduce an efficient algorithm for the robust estimation of the annihilating subspace in \S\ref{denoising}, which address the problem of model-mismatch and noisy measurements. This scheme is related to Cadzow denoising methods \cite{cadzow1988signal,condat2015cadzow} used for robust spectral estimation, and recent low-rank structured matrix completion approaches in MRI \cite{haldar2014low,jin2015general}. Finally, once the annihilating subspace is estimated, in \S\ref{sec:edgeawarerecovery} we propose recovering the piecewise constant signal in Fourier domain by constraining its extrapolated Fourier coefficients to be orthogonal to the annihilating subspace. 

For convenience, we rewrite the annihilation system \eqref{eq:annsys} in matrix notation. First, we fix rectangular index sets $\Gamma,\Lambda\subseteq \mathbb{Z}^2$, which represent the sampling set, and annihilating filter support, respectively, and set $N = |\Gamma\,{:}\,\Lambda|$, $M = |\Lambda|$. Define $\mathcal{T}_x$ and $\mathcal{T}_y$ to be the linear operators mapping samples $(\widehat{f}[\mbf k] : \mbf k \in \Gamma)$ to matrices $\mathcal{T}_x(\widehat{f}) \in \mathbb{C}^{N\times M}$ and $\mathcal{T}_y(\widehat{f})\in \mathbb{C}^{N\times M}$ which represent the discrete convolution of $\widehat{\partial_x f}$ and $\widehat{\partial_y f}$ with filters $(c[\mbf k]: \mbf k\in\Lambda)$, whose output is restricted to the index set $\Gamma\,{:}\,\Lambda$. In particular, $\mathcal{T}_x(\widehat{f})$ and $\mathcal{T}_y(\widehat{f})$ are block Toeplitz matricies with Toeplitz blocks, where the rows of $\mathcal{T}_x(\widehat{f})$ indexed by $\bs \ell \in \Gamma\,{:}\,\Lambda$ are given by $(\widehat{\partial_x f}[\bs \ell -\mbf k] : \mbf k \in \Lambda)$ in vectorized form, and similarly for $\mathcal{T}_y(\widehat{f})$; see Fig.\ \ref{fig:annmatrix}.  Finally, we define the linear operator $\mathcal{T}$ as the vertical concatenation of $\mathcal{T}_x$ and $\mathcal{T}_y$, i.e. 
\begin{equation}
\mathcal{T}(\widehat{f}) =
\begin{bmatrix}
\mathcal{T}_x(\widehat{f})\\
\mathcal{T}_y(\widehat{f})
\end{bmatrix} \in \mathbb{C}^{2N\times M}.
\label{eq:annmatrix}
\end{equation}
Treating $\mbf c = (c[\mbf k] : \mbf k \in \Lambda)$ as vector in $\mathbb{C}^{M}$ we may reproduce the annihilation system \eqref{eq:annsys} by writing
\[
\mathcal{T}(\widehat{f})\,\mbf c = \bs 0
\]
Therefore, recovering an annihilating polynomial for $f$ is equivalent to finding a vector in the nullspace of the structured matrix $\mathcal{T}(\widehat{f})$, which we call the \emph{annihilation matrix}. 

We note that our annihilation matrix is structurally similar to the to the multi-fold Hankel matrix construction in the matrix enhancement pencil matrix (MEMP) method, used for 2-D harmonic retrieval \cite{hua1992estimating,chen2014robust}. The main difference in our construction is that we apply weights $j2\pi \mbf k$ to the Fourier data before creating the structured matrix. Also, similar multi-fold Toeplitz/Hankel liftings of Fourier data have been considered in MRI reconstruction problems to model spatially limited and phase constrained images \cite{sake,haldar2014low}, and to model transform sparse images \cite{jin2015general}. Similar to us, \cite{jin2015general} also applies weights to Fourier data before constructing the structured matrix. However, our construction \eqref{eq:annmatrix} is unique in that we consider joint annihilation of both partial derivatives with respect to a common filter.
\begin{figure*}[ht!]
\centering
\subfloat[][]{
       \includegraphics[height=0.3\linewidth]{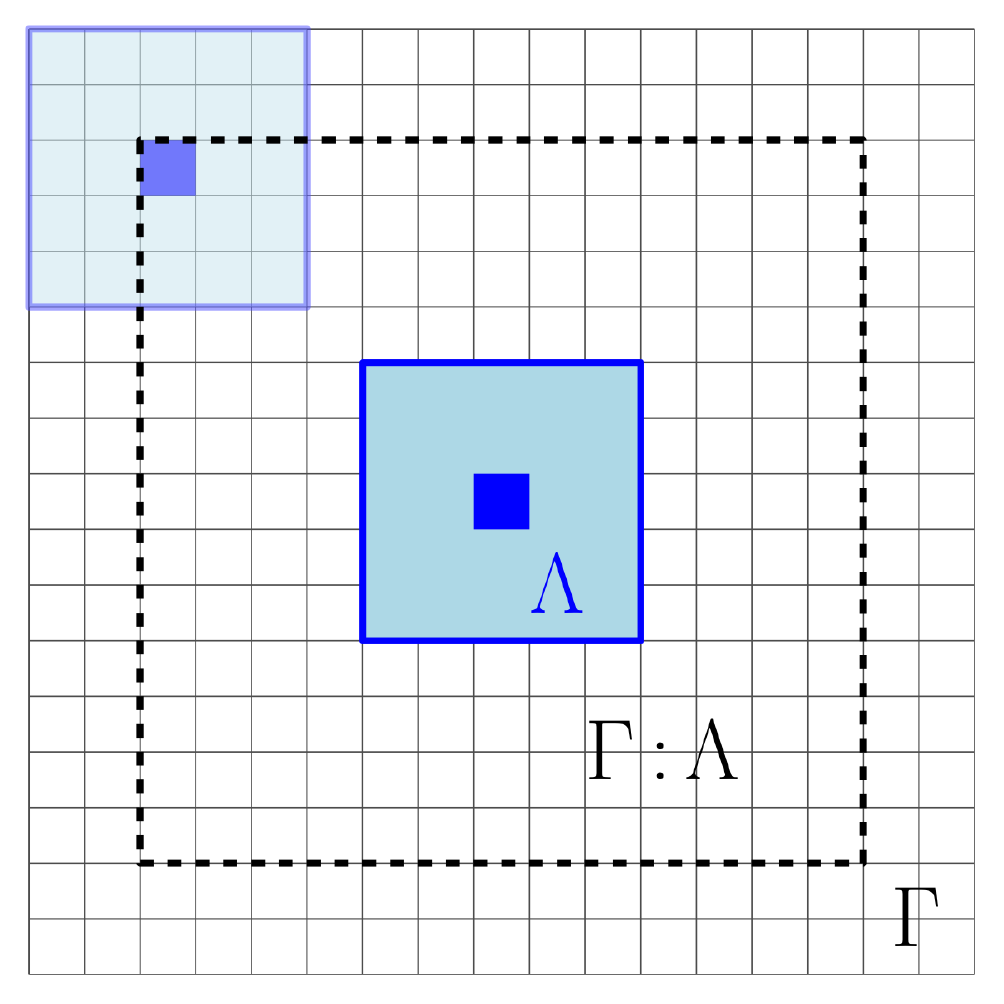}
}
\hspace{3em}
\subfloat[][]{
       \includegraphics[height=0.3\linewidth]{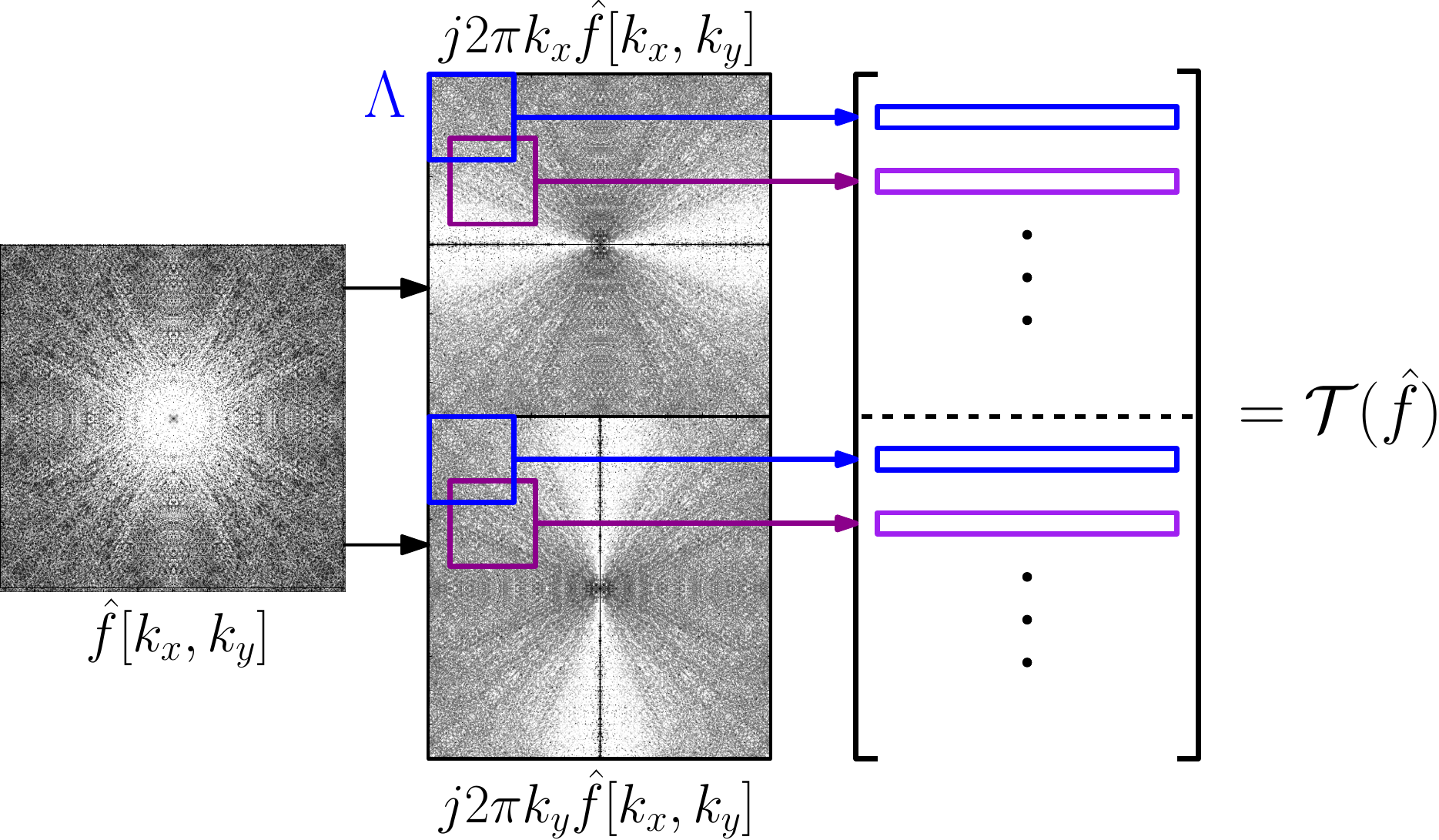}	
}	
\caption{\small\emph{Construction of the annihilation matrix}. Shown in (a) are the index sets in $\mathbb{Z}^2$ used to construct the annihilation matrix $\mathcal{T}(\widehat{f})$ defined in \eqref{eq:annmatrix}. Here $\Gamma$ in is the set of Fourier sampling locations, $\Lambda$ (in blue) is the assumed annihilating filter support, and $\Gamma\,{:}\,\Lambda$ (interior of dashed line) is the set of all integer shifts $\bs \ell$ such that $\Lambda$ shifted by $\bs \ell$ is contained in $\Gamma$. The construction of $\mathcal{T}(\widehat{f})$ is shown schematically in (b). From known samples $\widehat{f}[k_x,k_y]$ for $\mbf (k_x,k_y) \in\Gamma$ we compute $\widehat{\partial_x f}[k_x,k_y] = j2\pi k_x \hat{f}[k_x,k_y]$ and $\widehat{\partial_y f}[k_x,k_y] = j2\pi k_y \hat{f}[k_x,k_y]$ by a Fourier domain multiplication. Then each row of $\mathcal{T}(\widehat{f})$ is obtained by vectorizing a rectangular Fourier domain patch of size $\Lambda$ from the computed samples of $\widehat{\partial_x f}$ and $\widehat{\partial_y f}$.}
\label{fig:annmatrix}
\end{figure*}

\subsubsection{Annihilating subspace estimation}
When the degree of the minimal polynomial associated with the edge set is unknown in advance---which will typically be the case---we cannot build the annihilation matrix $\mathcal{T}(\widehat{f})$ with the correct filter support to ensure its null space is one-dimensional. For instance, if the assumed filter support $\Lambda'$ is strictly larger than support $\Lambda$ of the minimal filter $(c[\mbf k] : \mbf k \in \Lambda)$, the annihilation matrix will have a larger null space. This is because any filter of the form $\mbf c\ast \mbf d$ having Fourier support contained in $\Lambda'$ is also an annihilating filter by the associativity of convolution. In fact, Theorem \ref{thm:unique_3} says that under certain conditions on $f$, \emph{every} null space vector of $\mathcal{T}(\widehat{f})$ is of this form. Hence, under these conditions, $\mathcal{T}(\widehat{f})$ is rank deficient with the following bound:
\begin{proposition}
\label{prop:rank}
Let $f$ be a piecewise constant image with edge set $C = \{\mu=0\}$, where $\mu$ is the minimal polynomial with coefficients in $\Lambda$. If the assumed annihilating filter support $\Lambda'$ strictly contains $\Lambda$, then 
\begin{equation}
\text{rank}~\mathcal{T}(\widehat{f}) \leq |\Lambda'| - |\Lambda'{:}\,\Lambda|.
\label{eq:rankbound}
\end{equation}
Furthermore, equality holds if $f$ satisfies the conditions of Thm.\ \ref{thm:unique2} and if the sampling set $\Gamma \subseteq 2\Lambda  + \Lambda'$.
\end{proposition}
We call the null space of $\mathcal{T}(\widehat{f})$ the \emph{annihilating subspace}. Since identifying the minimal annihilating filter within the annihilating subspace when it has large dimension is challenging, we propose to leverage the entire annihilating subspace for both edge set estimation and image reconstruction.

Observe that the edge set is still uniquely recoverable from the annihilating subspace under ideal conditions. To see this, let $\mbf D = [\mbf d_1,\ldots,\mbf d_R]$ be any orthonormal basis for the annihilating subspace. Define $\mu_i$ to be the trigonometric polynomial having coefficients given by $\mbf d_i$. By the preceding discussion $\mu_i$ must always have the form $\mu_i = \gamma_i\,\mu$ for $\mu$ the minimal polynomial and $\gamma_i$ another trigonometric polynomial. We define a function $\overline{\mu}$ as the following sum-of-squares average
\begin{equation}
\overline{\mu}(\mbf r) = \sqrt{\sum_{i=1}^R |\mu_i (\mbf r) |^2} = \sqrt{\sum_{i=1}^R |{\gamma}_i (\mbf r) \mu(\mbf r) |^2}.
\label{eq:mubar}
\end{equation}
Note we always have $\overline{\mu}(\mbf r) = 0$ when $\mu(\mbf r) = 0$. We can  also prove the following:
\begin{proposition} The function $\overline{\mu}$ has at most finitely many zeros on $[0,1]^2/\{\mu = 0\}$.
\label{prop:finitezeros}
\end{proposition}
Therefore, we can identify the edge set by finding the zeros of $\overline{\mu}(\mbf r)$, and removing any isolated points. This method of estimating the edge set can also be interpreted as a higher dimensional analogue of the technique used in spectral MUSIC to estimate the locations of isolated spectral peaks \cite{schmidt1986multiple}. Recall that with spectral MUSIC one performs this estimate by computing the distance of pure exponential signals from a particular eigenspace of the signal covariance matrix. The following result shows we may also interpret $\overline{\mu}(\mbf r)$ in an analogous way:
\begin{proposition}
Let $\mbf D\in\mathbb{C}^{M\times R}$ be an orthonormal basis of the nullspace of $\mathcal{T}(\widehat{f})$, and for all $\mbf r \in [0,1]^2$, define the filter $\mbf e_{\mbf r} \in \mathbb{C}^{M}$ by $\mbf e_{\mbf r}[\mbf k] = e^{j2\pi \mbf k\cdot \mbf r}$ for all $\mbf k \in \Lambda$. Then we have
\[
\overline{\mu}(\mbf r) = \|\mbf D\mbf D^H \mbf e_{\mbf r} \|_2
\] 
that is, $\overline{\mu}(\mbf r)$ is the $\ell^2$-norm of the projection of $\mbf e_{\mbf r}$ onto the annihilating subspace.
\label{prop:MUSIC}
\end{proposition}

While we do not directly estimate the edge set in this fashion, the function $\overline{\mu}$ will play an important role in the second step of our proposed algorithm in \S\ref{sec:edgeawarerecovery}.

\subsubsection{Annihilation subspace estimation in the presence of noise}
\label{denoising}
When the measurements are corrupted by noise or there is model mismatch, we propose performing a denoising step to better identify the annihilating subspace. This denoising step is similar to Cadzow methods used in FRI algorithms \cite{blu2008sparse,pan2013sampling}, as well as recently proposed algorithms for structured low-rank matrix completion problems \cite{haldar2014low,chen2014robust,pal2014grid,jin2015general}. Given noisy Fourier samples $b[\mbf k] = \widehat{f}[\mbf k] + n[\mbf k]$, $\mbf k \in \Gamma$, where we assume $n[\mbf k]$ is a vector of i.i.d.\ mean-zero Gaussian noise, we attempt find the $\widehat{f}^*$ closest in Euclidean norm that satisfies the rank bound in Proposition \ref{prop:rank}:
\begin{equation}
\label{eq:cadzow}
\widehat{f}^* = \argmin_{\widehat{g}[\mbf k],\mbf k\in\Gamma} \quad \sum_{\mbf k \in \Gamma} \left|\widehat{g}[\mbf k] - b[\mbf k]\right|^2~~\text{such that}\quad \text{rank}[\mathcal{T}(\widehat{g})] \leq r.
\end{equation}
Here the parameter $r$ is estimated according to \eqref{eq:rankbound}. This can be viewed as a maximum likelihood estimation, where the rank constraint on $\mathcal{T}(\widehat{g})$ enforces our modeling assumption; see \cite{condat2015cadzow} for a similar approach in 1-D. Note the rank constraint in \eqref{eq:cadzow} is non-convex, which means in general we are only able to compute an local minimizer of \eqref{eq:cadzow}. One common heuristic for solving problems of the type \eqref{eq:cadzow} is to perform a Cadzow denoising on a lifted matrix. However, as observed in \cite{pan2013sampling}, this procedure is slow to converge in this context. Instead, we attempt to solve \eqref{eq:cadzow} more directly with a variable splitting approach adopted from \cite{haldar2014low}:
\begin{equation}
\label{eq:cadzow_new2}
\min_{\substack{\widehat{g}[\mbf k],\mbf k\in\Gamma\\\mbf T: \text{rank} \mbf T \leq r}} \quad \sum_{\mbf k \in \Gamma} \left|\widehat{g}[\mbf k] - b[\mbf k]\right|^2 + \lambda\, \|\mathcal{T}(\widehat{g})-\mbf T\|_F^2
\end{equation}
where $\lambda$ is a regularization parameter balancing data fidelity and the low-rank penalty term. This suggests the following alternating minimization scheme: set $\widehat{f}^{(0)} = b$ and for $n\geq 1$, solve
\begin{align}
\label{eq:cadzow_step1}
\mbf T^{(n+1)}  &= \argmin_{\mbf T: \text{rank}\mbf T \leq r} \|\mathcal{T}(\widehat{f}^{(n)})-\mbf T\|_F^2 \\
\widehat{f}^{(n+1)} &= \argmin_{\widehat{g}[\mbf k],\mbf k\in\Gamma} \sum_{\mbf k \in \Lambda} \left|\widehat{g}[\mbf k] - b[\mbf k]\right|^2 + \lambda \|\mathcal{T}(\widehat{g})-\mbf T^{(n+1)}\|_F^2
\label{eq:cadzow_step2}
\end{align}
The exact solution to subproblem \eqref{eq:cadzow_step1} is easily computed by taking the singular value decomposition of $\mathcal{T}(\widehat{f}^{(n)}) = \mbf U\bs \Sigma \mbf V^*$, and setting $\mbf T^{n+1} = \mbf U\bs \Sigma_r \mbf V^*$, where $\bs \Sigma_r$ is the diagonal matrix of singular values $\sigma_i$ with $\sigma_i = 0$ for all $i > r$. Also, subproblem \eqref{eq:cadzow_step2} is quadratic in $\widehat{g}$, and has an exact solution given by $\widehat{g} = \left({b + \mathcal{T}^*(\mbf T)}\right)/\left({1 + \lambda\,\text{diag}(\mathcal{T}^*\mathcal{T})}\right)$,
where the operations above are understood element-wise. Here $\mathcal{T}^*$ denotes the adjoint of $\mathcal{T}$, which computes a weighted sum of entries of $\mbf T \in \mathbb{C}^{2N\times M}$. In particular, one can show $\mathcal{T}^*\mathcal{T}$ is a diagonal matrix. Due to the alternating scheme, the cost in \eqref{eq:cadzow_new2} is guaranteed to monotonically decrease with updates \eqref{eq:cadzow_step2}. A detailed convergence analysis of the algorithm is beyond the scope of this work; further analysis can be found in \cite{haldar2014low}. However, we observe the algorithm converges well in practice. In all our experiments conducted in \S\ref{sec:exp} a suitable solution was reached in 5--10 iterations.

\subsection{Step 2: Edge set aware image recovery}
\label{sec:edgeawarerecovery}
We now focus on the edge-aware recovery of the image using \eqref{eq:recovery_Fourier}, which amounts to the Fourier domain extrapolation of the measured samples. This recovery can still be considered as ``off-the-grid'', in the sense that we reconstruct the exact Fourier series coefficients of the underlying continuous domain image. We will also briefly consider the case where we solve for the image ``on-the-grid'' in spatial domain similar to \cite{pan2013sampling}.

\subsubsection{Fourier domain recovery: Extrapolation by linear prediction}
Suppose that we have an estimate of a basis for the annihilating subspace $\mbf D = [\mbf d_1,\ldots,\mbf d_R]$. Motivated by the Fourier domain recovery formulations \eqref{eq:recovery_Fourier} and \eqref{eq:finiterecovery}, we pose the recovery of the Fourier domain of the signal within an arbitrary rectangular domain $\Delta \subseteq \mathbb{Z}^2$ as
\begin{equation}
\label{opt:ell2}
\min_{\widehat{g}[\mbf k],\mbf k \in \Delta} \sum_{i=1}^R\widehat{\|\nabla g}* \mbf d_i\|^2\quad \text{subj. to}\quad \widehat{g}[\mbf k ] = \widehat{f}[\mbf k], \mbf k\in \Gamma,
\end{equation}
i.e.\ we extrapolate the signal by enforcing its Fourier derivatives to be annihilated by filters in the annihilating subspace. Here we restrict the convolution to the valid region $\Delta\,{:}\,\Lambda\subset\mathbb{Z}^2$. When the measurements are corrupted by noise (i.e.\ $b[\mbf k] = \widehat{f}[\mbf k] + n[\mbf k]$, $\mbf k\in\Gamma$, where $n$ is a vector i.i.d.\ complex white Gaussian noise), we relax the equality constraint in \eqref{opt:ell2} and instead solve
\begin{equation}
\label{opt:ell2_lam}
\min_{\widehat{g}[\mbf k],\mbf k \in \Delta} \sum_{i=1}^R\|\widehat{\nabla g}* \mbf d_i\|^2 + \lambda \|P_\Gamma\widehat{g} - \mbf b\|^2.
\end{equation}
Here, $\lambda$ is a tunable parameter balancing annihilation and data fidelity, and $P_\Gamma$ is projection onto the sampling set $\Gamma$. Note that \eqref{opt:ell2_lam} is a linear least-squares problem, and can be solved efficiently with a conjugate gradient iterative solver, provided we have an fast method of computing matrix-vector products with the gradient of the objective in \eqref{opt:ell2_lam}. Towards this end, we rewrite \eqref{opt:ell2_lam} in terms of a linear operator $\mathcal{Q}$ acting on $\widehat{g}$:
\begin{equation}
\label{opt:ell2_Q}
\min_{\widehat{g}[\mbf k],\mbf k \in \Delta} \|\mathcal{Q}\widehat{g}\|_F^2  + \lambda \|P_\Gamma\widehat{g} - \mbf b\|^2.
\end{equation}
where $\|\cdot\|_F$ denotes the Frobenius norm. We can express $\mathcal{Q}$ as
\begin{equation}
\mathcal{Q} = [\mathcal{Q}_{x,1},\mathcal{Q}_{y,1},...,\mathcal{Q}_{x,R},\mathcal{Q}_{y,R}];~~\mathcal{Q}_{x,i} = \mbf P \mbf C_{i} \mbf M_x,~~\mathcal{Q}_{y,i} = \mbf P \mbf C_{i} \mbf M_y\text{ for all } i = 1,...,R,
\label{eq:Qexpand}
\end{equation}
with $\mbf M_x$ and $\mbf M_y$ representing elementwise multiplication by $j2\pi k_x$ and $j2\pi k_y$, respectively, $\mbf C_i$ discrete convolution with the annihilating filter $\mbf d_i$, and $\mbf P$ projection onto the index set $\Delta : \Lambda$. Note that since $\widehat{g}$ is restricted to $\Delta$, we may take $\mbf C_i$ to be a 2-D circular convolution, provided we zero-pad the filter $\mbf d_i$ and the input $\widehat{g}$ sufficiently. Thus, we can write $\mbf C_i = \mbf F \mbf S_i \mbf F^H$ where $\mbf F$ is the 2-D DFT on any rectangular grid containing the set $\Delta+\Lambda$, and $\mbf S_i$ is elementwise multiplication by the inverse DFT of $\mbf d_i$, which is given by a spatial gridding of the annihilating polynomial $\mu_i(\mbf r)$ having coefficients $\mbf d_i$.

Taking the gradient of the objective in \eqref{opt:ell2_Q} and setting it to zero yields
\[
(\mathcal{Q}^*\mathcal{Q} + \lambda P_\Gamma^*P_\Gamma)\widehat{g} = P_\Gamma^*\mbf b.
\]
From \eqref{eq:Qexpand}, the operator $\mathcal{Q}^*\mathcal{Q}$ can be expanded as
\begin{equation}
\mathcal{Q}^*\mathcal{Q} = \sum_{i=1}^R(\mbf M_x^* \mbf C_i^* \mbf P^*\mbf P \mbf C_i \mbf M_x + \mbf M_y^* \mbf C_i^* \mbf P^*\mbf P \mbf C_i \mbf M_y)
\label{eq:QstarQ}
\end{equation}
Unfortunately, when the number of annihilating filters is large (i.e. $R$ is large), applying this operator directly will be computationally costly, requiring $8R$ two-dimensional FFT operations per application. However, if the resolution of the reconstruction grid $\Delta$ is large enough, the projection operator $\mbf P^*\mbf P$ will be close to the identity matrix $\mbf I$. Making the approximation $\mbf P^* \mbf P \approx \mbf I$ in \eqref{eq:QstarQ}, we obtain
\begin{equation}
\mathcal{Q}^*\mathcal{Q} \approx \mbf M_x^* \mbf F^* \mbf R \mbf F \mbf M_x + \mbf M_y^* \mbf F^* \mbf R \mbf F \mbf M_y.
\label{eq:Qapprox}
\end{equation}
where $\mbf R = \sum_{i=1}^R \mbf S_i^*\mbf S_i$ is diagonal and represents pointwise multiplication by spatially gridded samples of the sum-of-squares polynomial $\overline{\mu}^2 = \sum_{i=1}^R |\mu_i|^2$, as defined in \eqref{eq:mubar}. Note that after computing the entries of $\mbf R$, computing matrix-vector products with \eqref{eq:Qapprox} requires only four 2-D FFTs per iteration to evaluate, as opposed to $8R$. Working backwards, we can show approximating $\mathcal{Q}^*\mathcal{Q}$ in this manner is equivalent to minimizing
\begin{equation}
\label{opt:ell2_approx}
\min_{\widehat{g}[\mbf k],\mbf k \in \Delta} \|\overline{\mu}\,\nabla \mathcal{B}(g)\|^2_{L^2} + {\lambda} \|P_\Gamma\,\widehat{g} - \mbf b\|^2
\end{equation}
where here $\mathcal{B}(g)$ is the function obtained by bandlimiting $g$ to the Fourier extrapolation set $\Delta$, and $\|\cdot\|^2_{L^2}$ is the spatial domain $L^2$-norm of vector-valued functions on $[0,1]^2$. This shows that linear prediction using the entire annihilating subspace is roughly equivalent to \eqref{eq:recovery_space}. The main difference is that the minimal polynomial $\mu$ is replaced by the sum-of-squares average polynomial $\overline{\mu}$ defined in \eqref{eq:mubar}. We propose solving \eqref{opt:ell2_approx} as a surrogate for \eqref{opt:ell2_lam} when the dimension of the annihilating subspace is large, and call \eqref{opt:ell2_approx} the least-squares linear prediction (LSLP) approach. 

\subsubsection{Spatial domain recovery: Weighted total variation}
We also consider solving for the discrete samples of a function $g:[0,1]^2\rightarrow\mathbb{C}^2$ on a rectangular grid $\{\mbf r_{i,j}\} \subset [0,1]^2$ in spatial domain:
\begin{equation}
\label{opt:ell1}
\min_{g} \sum_{i,j} \overline{\mu}(\mbf r_{i,j})\,|\nabla g(\mbf r_{i,j})| + {\lambda} \|P_\Gamma\,\widehat{g} - \mbf b\|^2,
\end{equation}
This approach of solving for the image on discrete spatial grid is similar to the approach in \cite{pan2013sampling}, and can be thought of as a discrete approximation of \eqref{eq:recovery_space}. To obtain a fast algorithm, we use finite differences to approximate the gradient $\nabla g$ in \eqref{opt:ell1}, and replace the Fourier projection $P_\Gamma\,\widehat{g}$ with its DFT counterpart. In this case, \eqref{opt:ell1} can be cast as a weighted total variation minimization which can be solved efficiently with a primal-dual algorithm, such as \cite{zhu2008efficient}. We call \eqref{opt:ell1} the weighted total variation approach (WTV). 

We note that several edge-aware recovery formulations similar to our WTV \eqref{opt:ell1} and LSLP approach \eqref{opt:ell2_approx} have been proposed previously in the literature; see e.g.\ \cite{xiang2005accelerating,jacob2007improved,haldar2008anatomically,vaswani2010modified,luo2012mri,gong2015promise}. The difference of the proposed approach lies in the estimation of the edge weights, which are obtained from the low-resolution data itself, rather than from \emph{a priori} known high resolution edge information, as is assumed in the above cited works.

\begin{figure*}
\centering
\input{results/phantom_SL.tex}
\subfloat[][\scriptsize Fully sampled]{
\begin{minipage}{0.16\linewidth}
\begin{tikzpicture}[      
        every node/.style={anchor=south west,inner sep=0pt},
        x=1mm, y=1mm,
      ]   
     \node (fig1) at (0,0)
       {\includegraphics[width=\linewidth,trim = 20mm 5mm 40mm 55mm, clip]{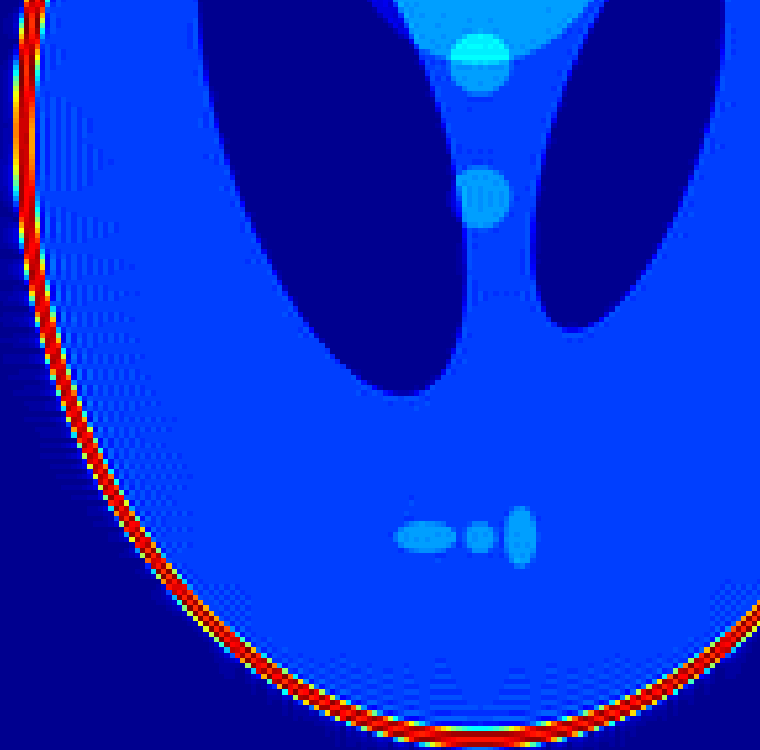}};
     \node [draw=white, thick] (fig2) at (0,0)
       {\includegraphics[width=0.4\linewidth]{results/phantom_SL_orig_jet-eps-converted-to.pdf}};  
\end{tikzpicture}

\includegraphics[width=\linewidth]{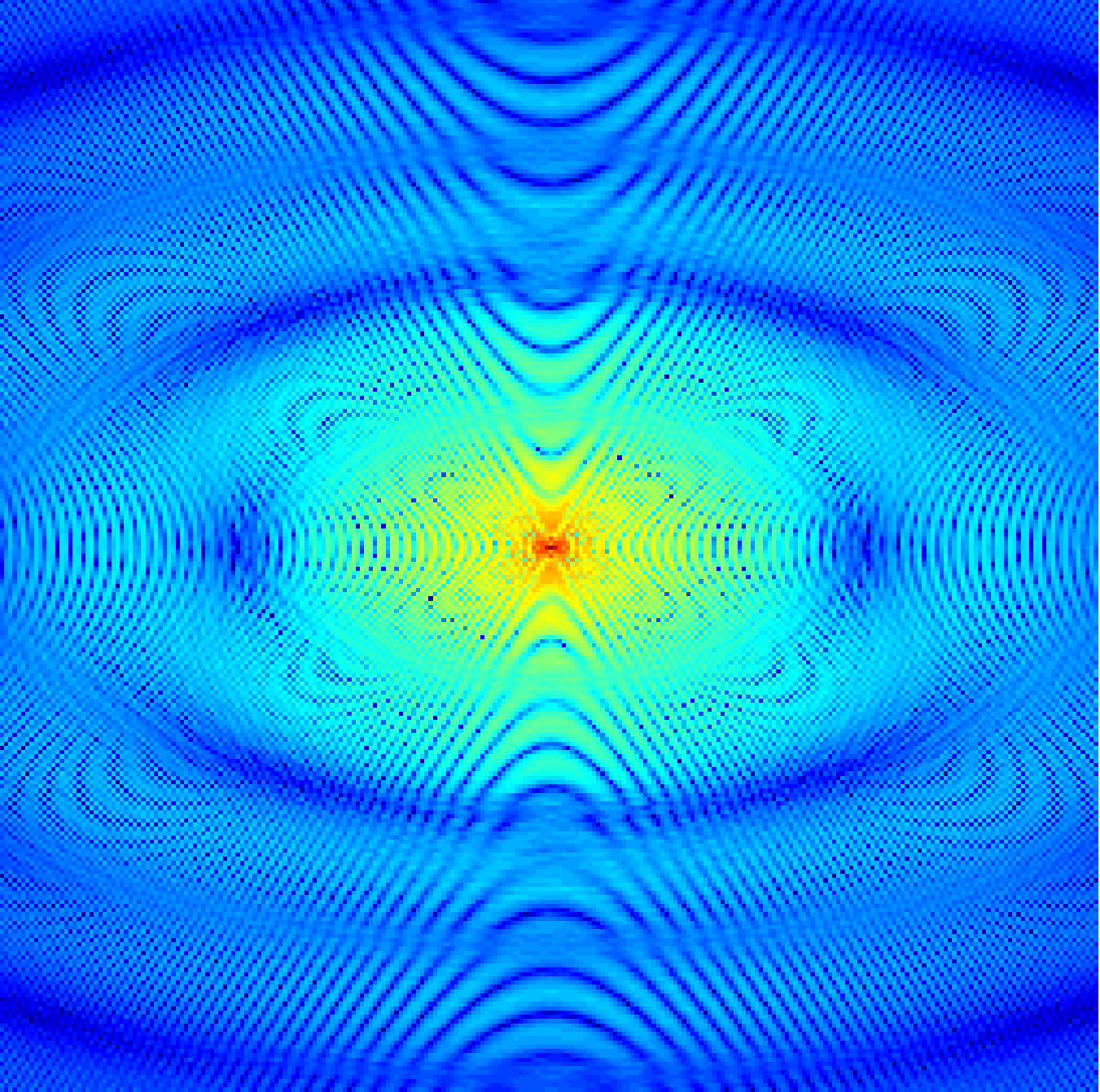}
\end{minipage}
}
\subfloat[][\scriptsize IFFT\\SNR=\SNRifft dB]{
\hspace{-1em}
\begin{minipage}{0.16\linewidth}
\begin{tikzpicture}[      
        every node/.style={anchor=south west,inner sep=0pt},
        x=1mm, y=1mm,
      ]   
     \node (fig1) at (0,0)
       {\includegraphics[width=\linewidth,trim = 20mm 5mm 40mm 55mm, clip]{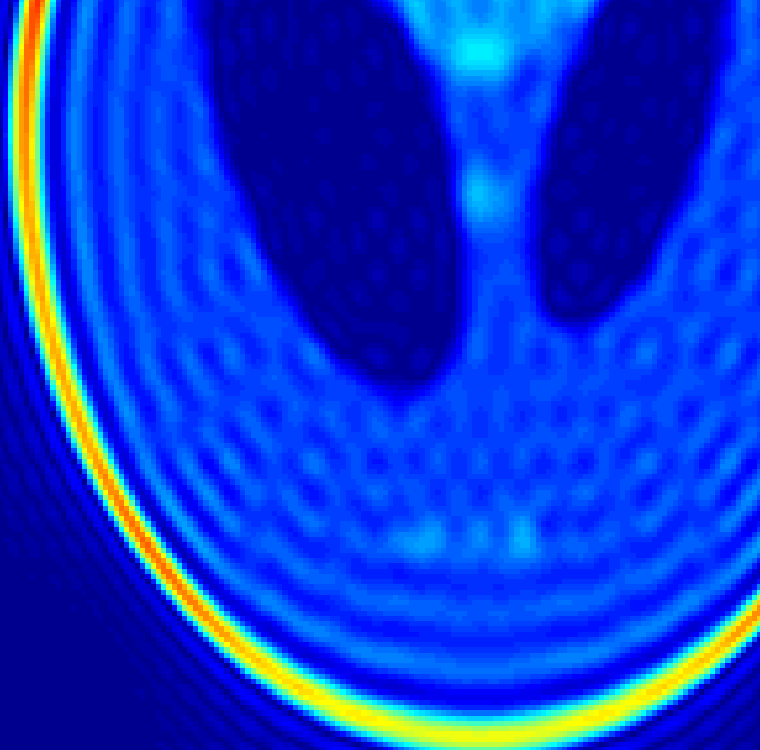}};
     \node [draw=white, thick] (fig2) at (0,0)
       {\includegraphics[width=0.4\linewidth]{results/phantom_SL_ifft_jet-eps-converted-to.pdf}};  
\end{tikzpicture}

\includegraphics[width=\linewidth]{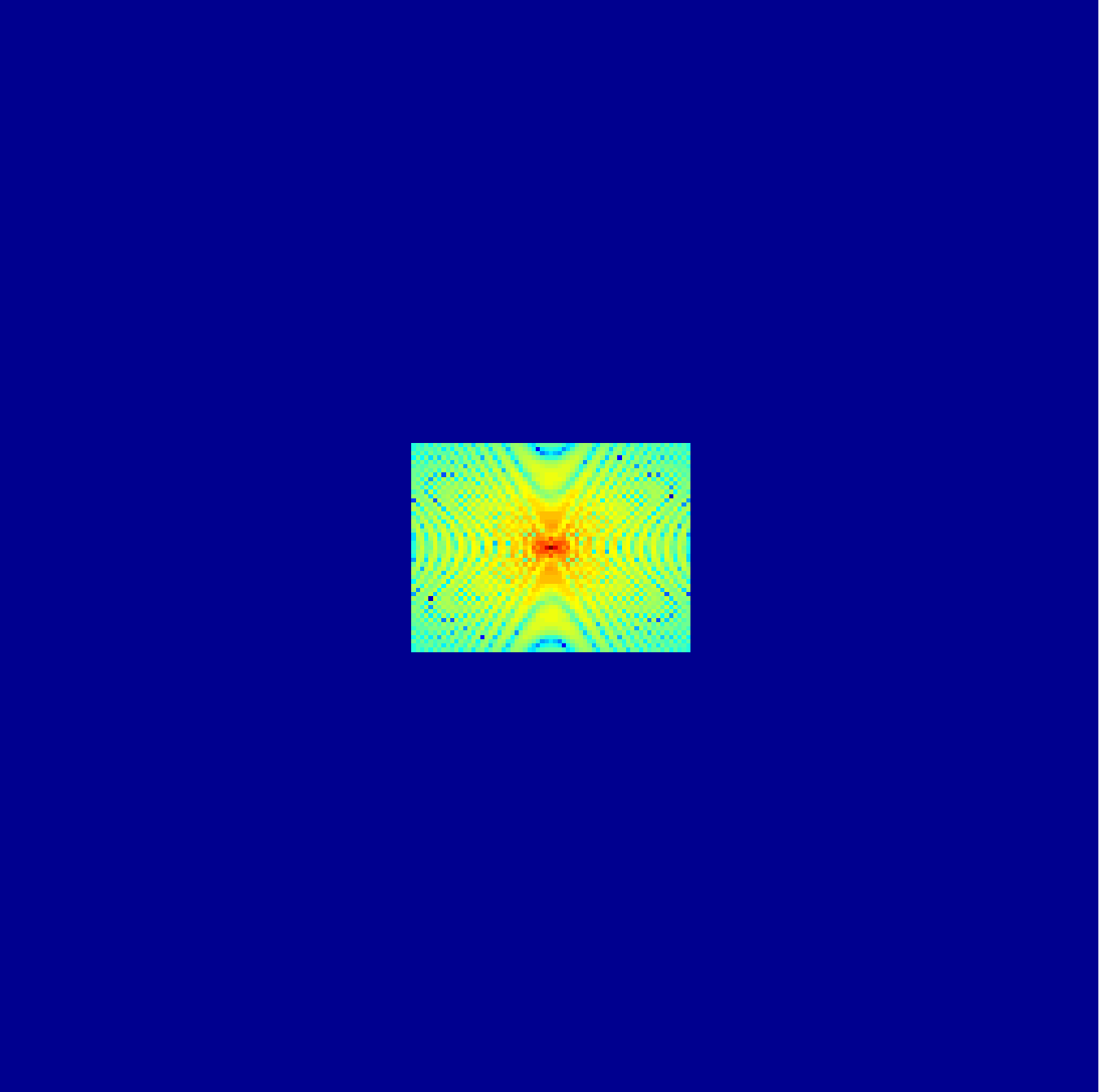}
\end{minipage}
}
\subfloat[][\scriptsize TV\\SNR=\SNRtv dB]{
\hspace{-1em}
\begin{minipage}{0.16\linewidth}
\begin{tikzpicture}[      
        every node/.style={anchor=south west,inner sep=0pt},
        x=1mm, y=1mm,
      ]   
     \node (fig1) at (0,0)
       {\includegraphics[width=\linewidth,trim = 20mm 5mm 40mm 55mm, clip]{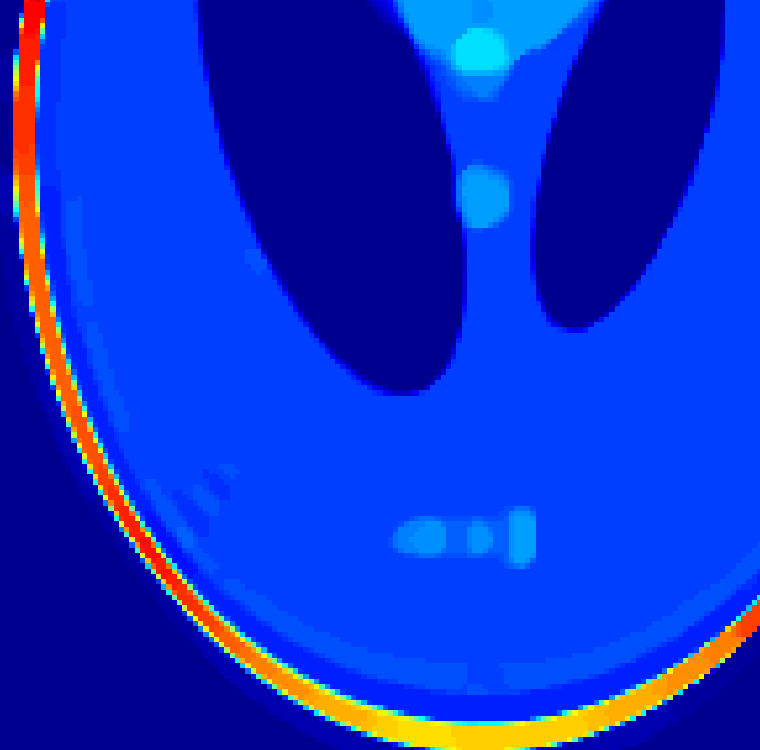}};
     \node [draw=white, thick] (fig2) at (0,0)
       {\includegraphics[width=0.4\linewidth]{results/phantom_SL_tv_jet-eps-converted-to.pdf}};  
\end{tikzpicture}

\includegraphics[width=\linewidth]{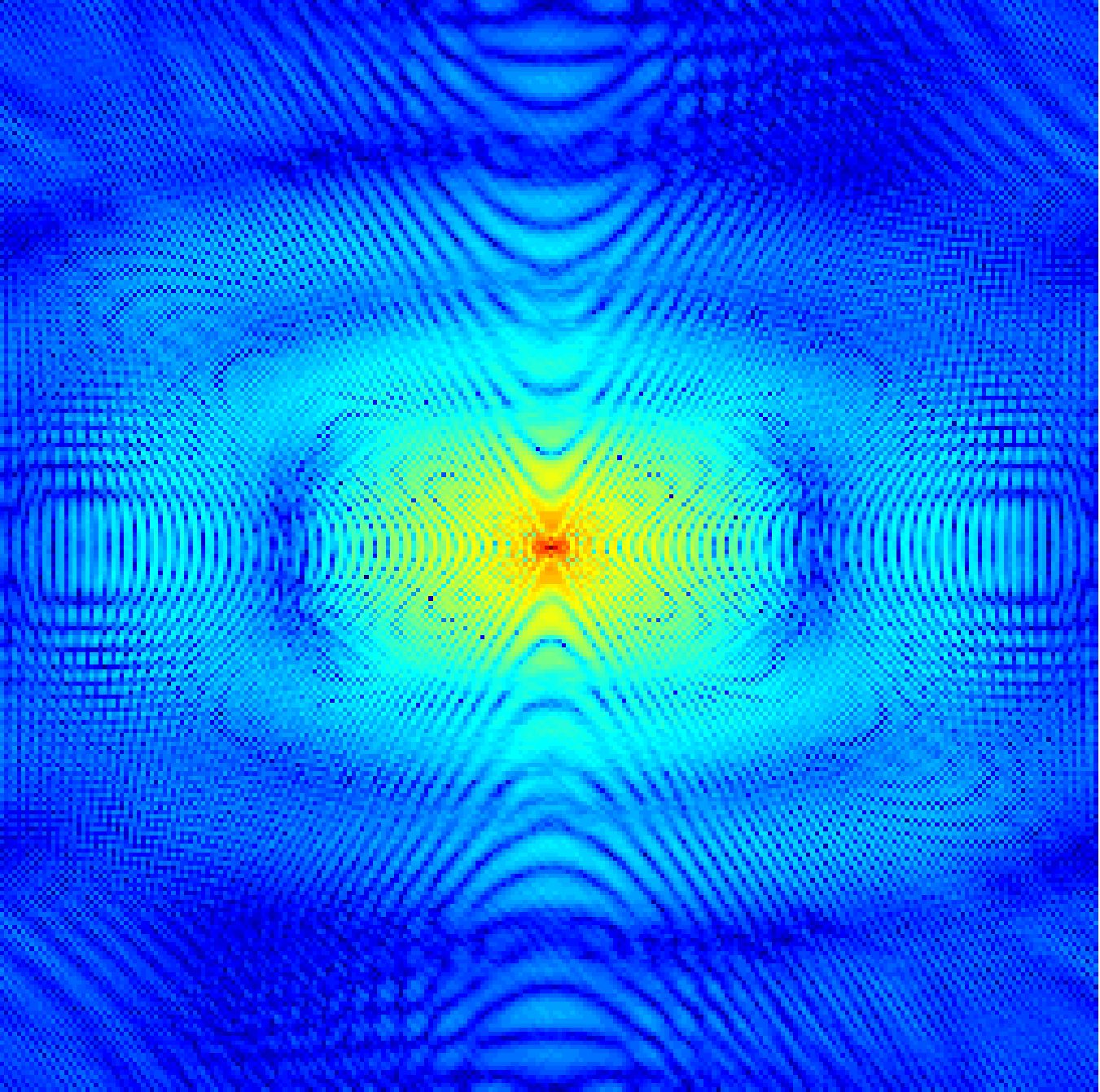}
\end{minipage}
}
\subfloat[][\scriptsize ncvx-TV\\SNR=\SNRncvxtv dB]{
\hspace{-1em}
\begin{minipage}{0.16\linewidth}
\begin{tikzpicture}[      
        every node/.style={anchor=south west,inner sep=0pt},
        x=1mm, y=1mm,
      ]   
     \node (fig1) at (0,0)
       {\includegraphics[width=\linewidth,trim = 20mm 5mm 40mm 55mm, clip]{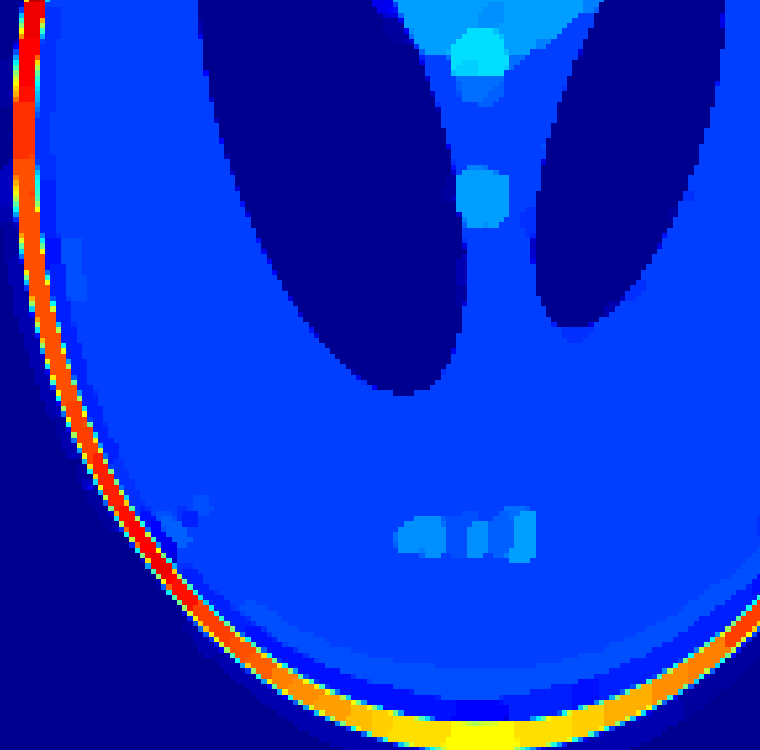}};
     \node [draw=white, thick] (fig2) at (0,0)
       {\includegraphics[width=0.4\linewidth]{results/phantom_SL_ncvxtv_jet-eps-converted-to.pdf}};  
\end{tikzpicture}

\includegraphics[width=\linewidth]{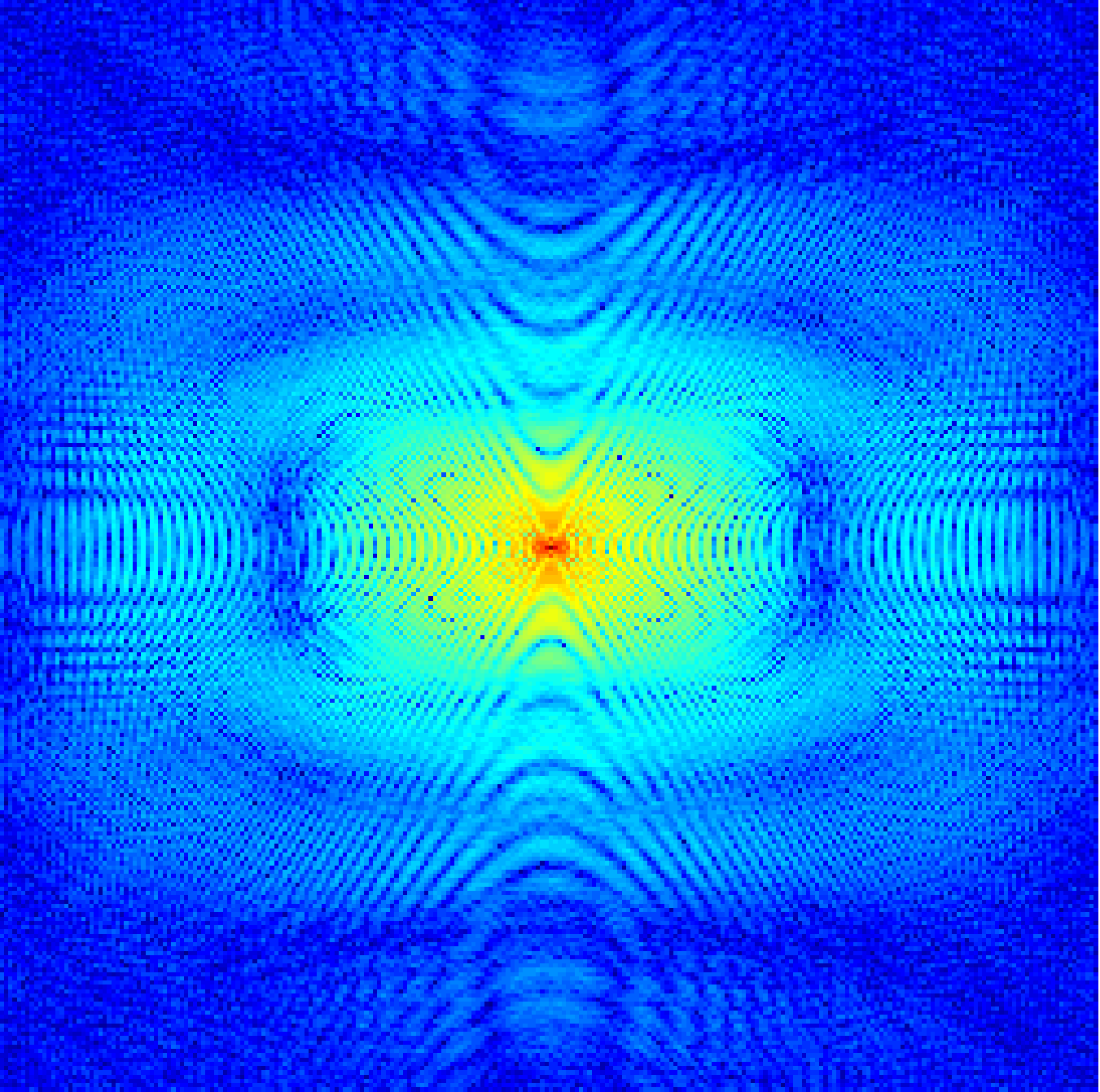}
\end{minipage}
}
\subfloat[][\scriptsize LSLP\\SNR=\SNRlslp dB]{
\hspace{-1em}
\begin{minipage}{0.16\linewidth}
\begin{tikzpicture}[      
        every node/.style={anchor=south west,inner sep=0pt},
        x=1mm, y=1mm,
      ]   
     \node (fig1) at (0,0)
       {\includegraphics[width=\linewidth,trim = 20mm 5mm 40mm 55mm, clip]{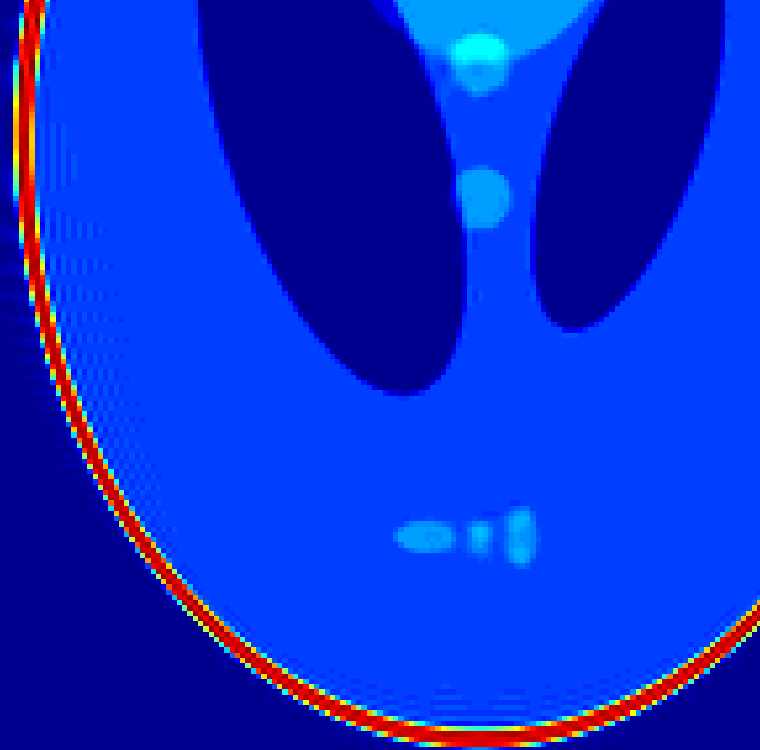}};
     \node [draw=white, thick] (fig2) at (0,0)
       {\includegraphics[width=0.4\linewidth]{results/phantom_SL_lslp_jet-eps-converted-to.pdf}};  
\end{tikzpicture}

\includegraphics[width=\linewidth]{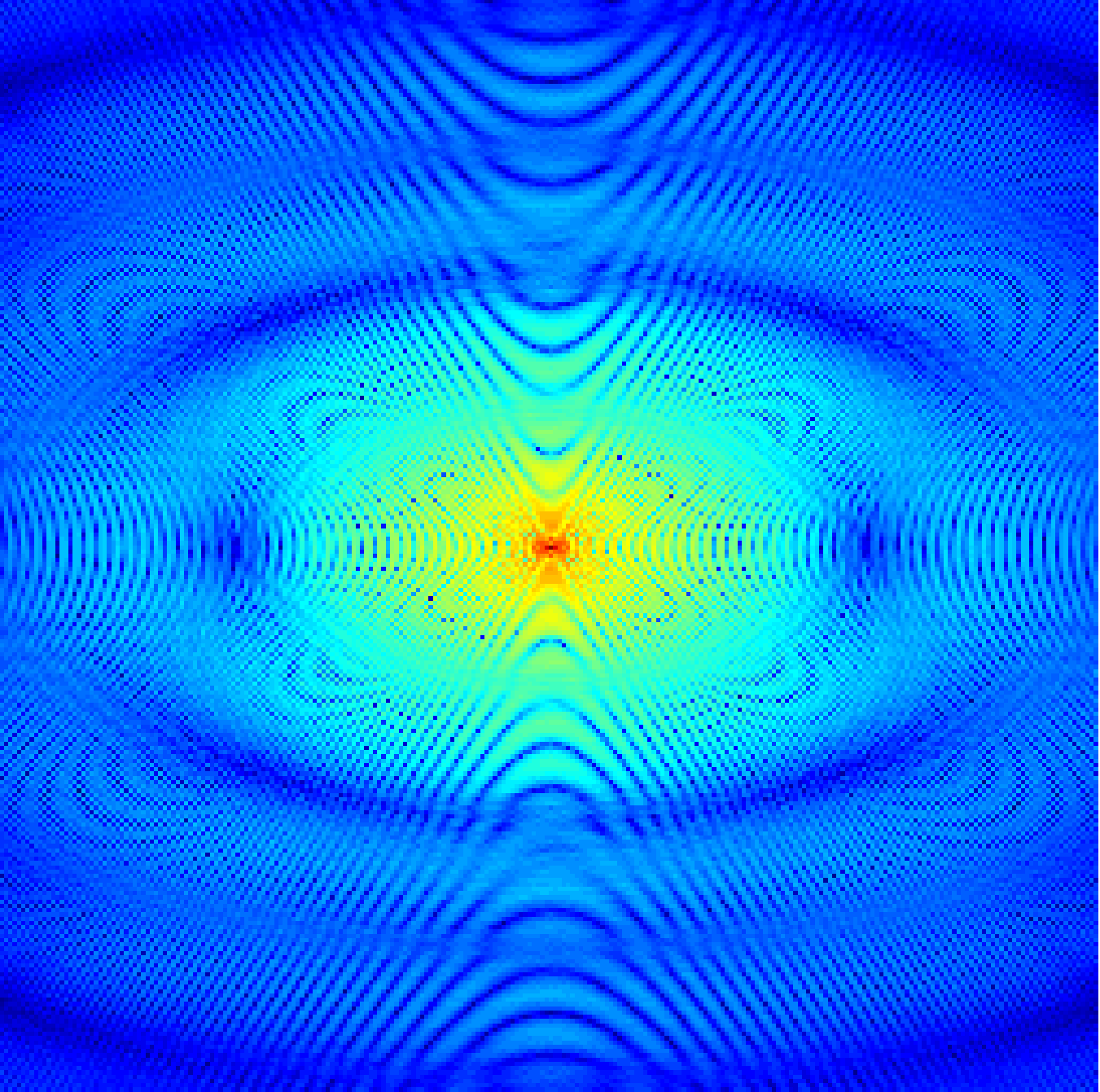}
\end{minipage}
}
\subfloat[][\scriptsize WTV\\SNR=\SNRwtv dB]{
\hspace{-1em}
\begin{minipage}{0.16\linewidth}
\begin{tikzpicture}[      
        every node/.style={anchor=south west,inner sep=0pt},
        x=1mm, y=1mm,
      ]   
     \node (fig1) at (0,0)
       {\includegraphics[width=\linewidth,trim = 20mm 5mm 40mm 55mm, clip]{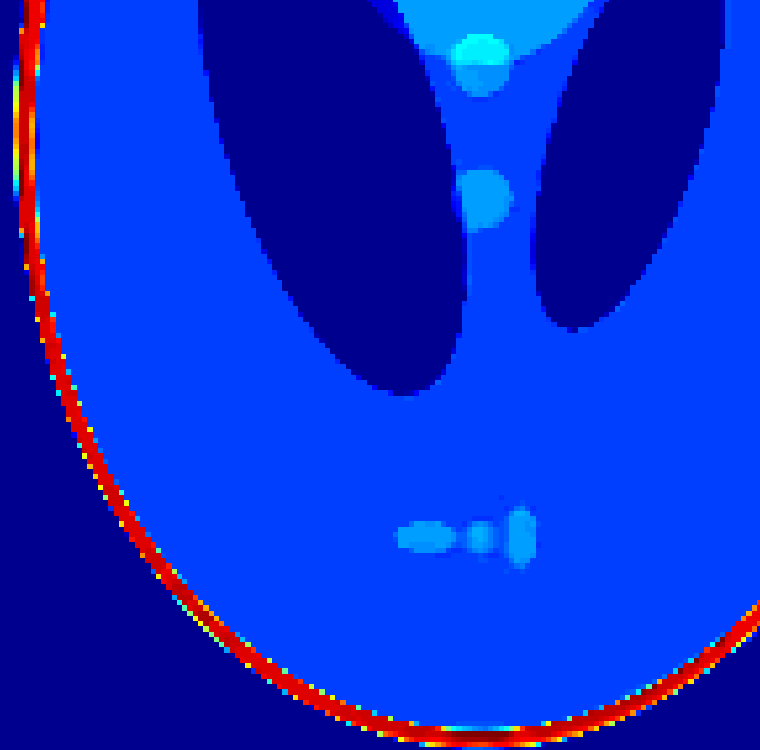}};
     \node [draw=white, thick] (fig2) at (0,0)
       {\includegraphics[width=0.4\linewidth]{results/phantom_SL_wtv_jet-eps-converted-to.pdf}};  
\end{tikzpicture}

\includegraphics[width=\linewidth]{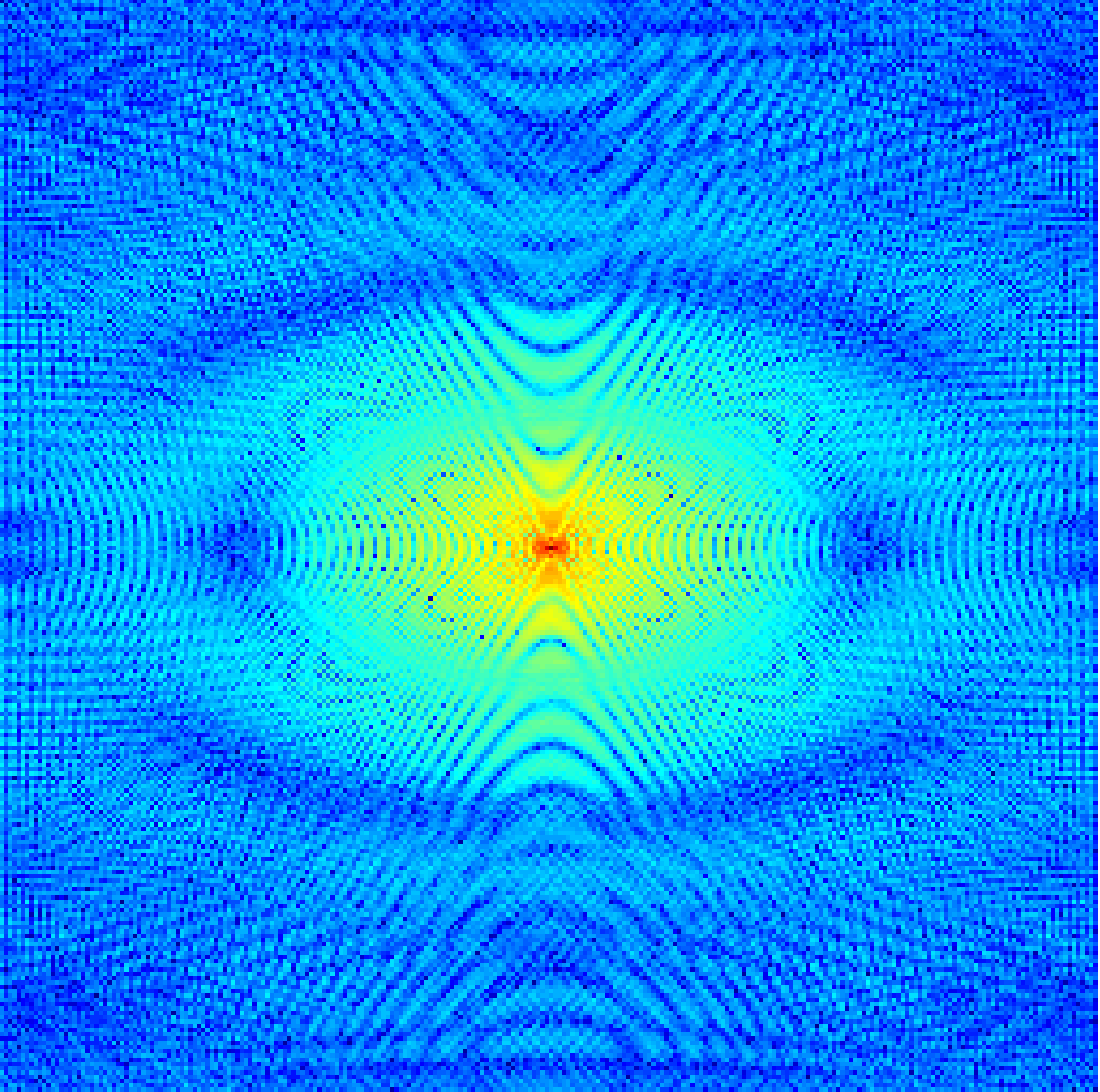}
\end{minipage}
}
\begin{minipage}{0.5cm}
\raisebox{0cm}{
\hspace{-1.5em}
\setlength\figureheight{2.2cm}
%
%
\begin{tikzpicture}

\begin{axis}[
    hide axis,
    scale only axis,
    width=1em,
    height=3em,
    point meta min=0,
    point meta max=1,
    colorbar,
    colormap/jet,
    colorbar style={
        height=\figureheight,
        width=0.7em,
        ytick={0,0.2,...,1},
        yticklabel style={
            xshift = -0.5ex,
            font = \tiny
        }
    }]
    \addplot [draw=none] coordinates {(0,0)};
\end{axis}
\end{tikzpicture}%
}

\raisebox{0.2cm}{
\hspace{-1.5em}
\setlength\figureheight{2.2cm}
%
%
\begin{tikzpicture}

\begin{axis}[
    hide axis,
    scale only axis,
    width=1em,
    height=3em,
    point meta min=0,
    point meta max=9,
    colorbar,
    colormap/jet,
    colorbar style={
        height=\figureheight,
        width=0.7em,
        ytick={0,1,...,9},
        yticklabel style={
            xshift = -0.5ex,
            font = \tiny
        }
    }]
    \addplot [draw=none] coordinates {(0,0)};
\end{axis}
\end{tikzpicture}%
}
\end{minipage}
\\
\input{results/phantom_brain.tex}
\subfloat[Fully sampled]{
\begin{minipage}{0.16\linewidth}
\begin{tikzpicture}[      
        every node/.style={anchor=south west,inner sep=0pt},
        x=1mm, y=1mm,
      ]   
     \node (fig1) at (0,0)
       {\includegraphics[width=\linewidth,height=\linewidth,trim = 45mm 10mm 40mm 75mm, clip]{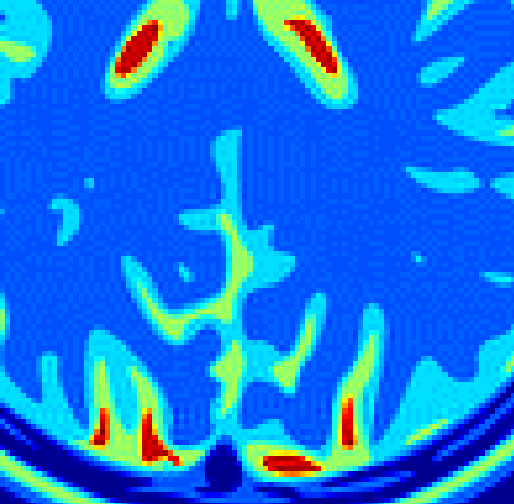}};
     \node [draw=white, thick] (fig2) at (0,0)
       {\includegraphics[width=0.4\linewidth,height=0.4\linewidth]{results/phantom_brain_orig_jet-eps-converted-to.pdf}};  
\end{tikzpicture}

\includegraphics[width=\linewidth,height=\linewidth]{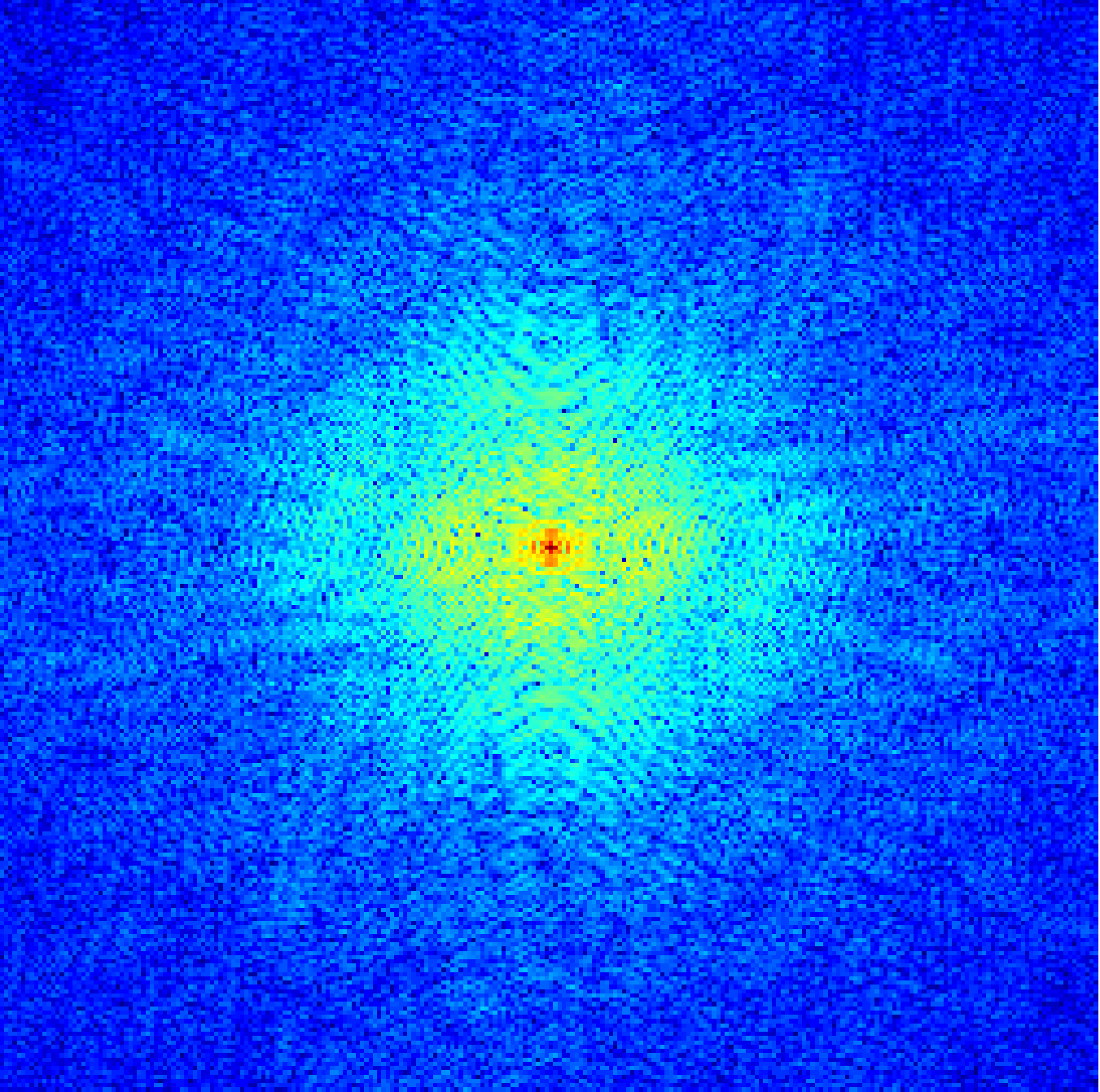}
\end{minipage}
}
\subfloat[][\scriptsize IFFT\\SNR=\SNRifft dB]{
\hspace{-1em}
\begin{minipage}{0.16\linewidth}
\begin{tikzpicture}[      
        every node/.style={anchor=south west,inner sep=0pt},
        x=1mm, y=1mm,
      ]   
     \node (fig1) at (0,0)
       {\includegraphics[width=\linewidth,height=\linewidth,trim = 45mm 10mm 40mm 75mm, clip]{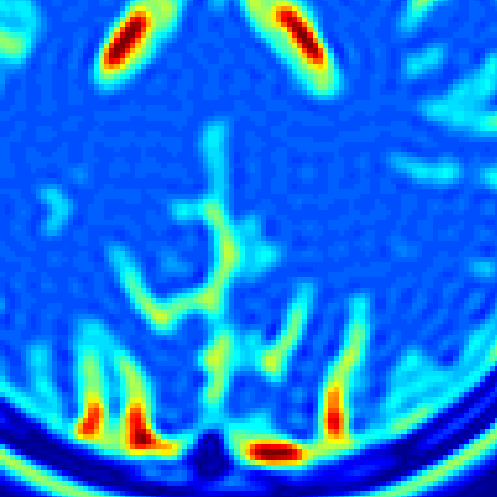}};
     \node [draw=white, thick] (fig2) at (0,0)
       {\includegraphics[width=0.4\linewidth,height=0.4\linewidth]{results/phantom_brain_ifft_jet-eps-converted-to.pdf}};  
\end{tikzpicture}

\includegraphics[width=\linewidth,height=\linewidth]{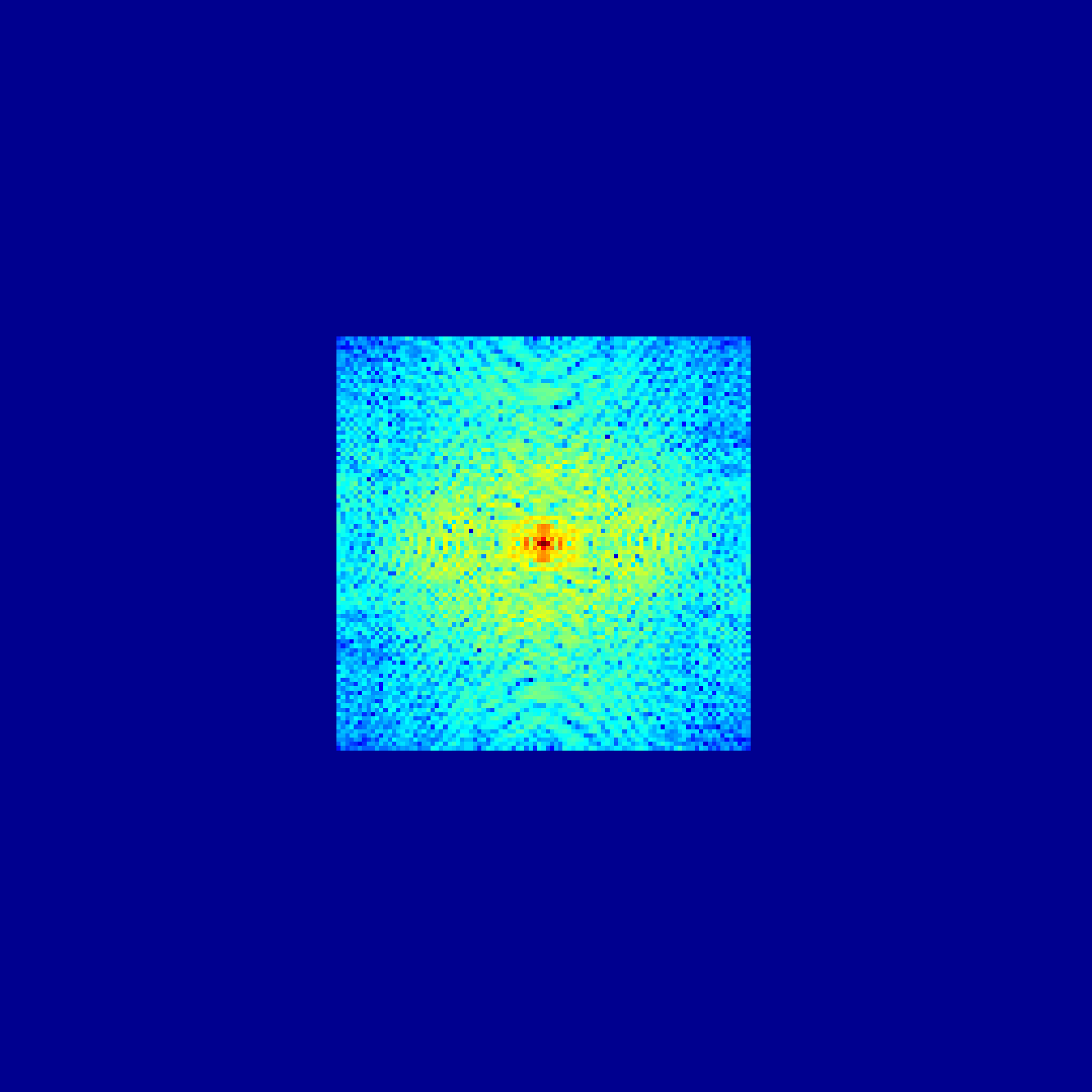}
\end{minipage}
}
\subfloat[][\scriptsize TV\\SNR=\SNRtv dB]{
\hspace{-1em}
\begin{minipage}{0.16\linewidth}
\begin{tikzpicture}[      
        every node/.style={anchor=south west,inner sep=0pt},
        x=1mm, y=1mm,
      ]   
     \node (fig1) at (0,0)
       {\includegraphics[width=\linewidth,height=\linewidth,trim = 45mm 10mm 40mm 75mm, clip]{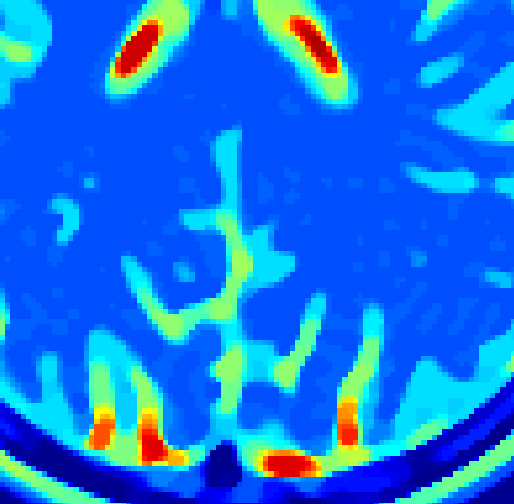}};
     \node [draw=white, thick] (fig2) at (0,0)
       {\includegraphics[width=0.4\linewidth,height=0.4\linewidth]{results/phantom_brain_tv_jet-eps-converted-to.pdf}};  
\end{tikzpicture}

\includegraphics[width=\linewidth,height=\linewidth]{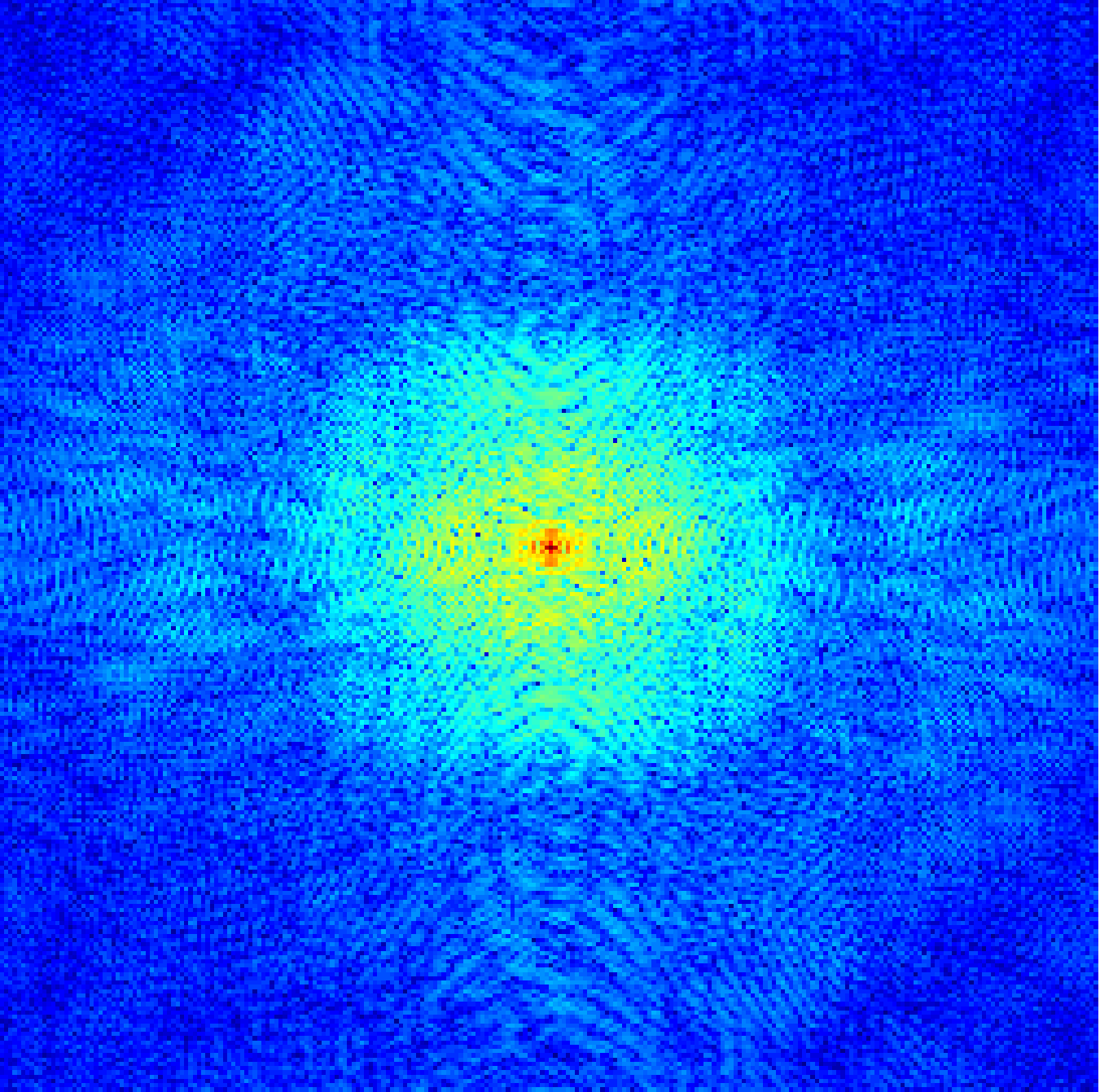}
\end{minipage}
}
\subfloat[][\scriptsize ncvx-TV\\SNR=\SNRncvxtv dB]{
\hspace{-1em}
\begin{minipage}{0.16\linewidth}
\begin{tikzpicture}[      
        every node/.style={anchor=south west,inner sep=0pt},
        x=1mm, y=1mm,
      ]   
     \node (fig1) at (0,0)
       {\includegraphics[width=\linewidth,height=\linewidth,trim = 45mm 10mm 40mm 75mm, clip]{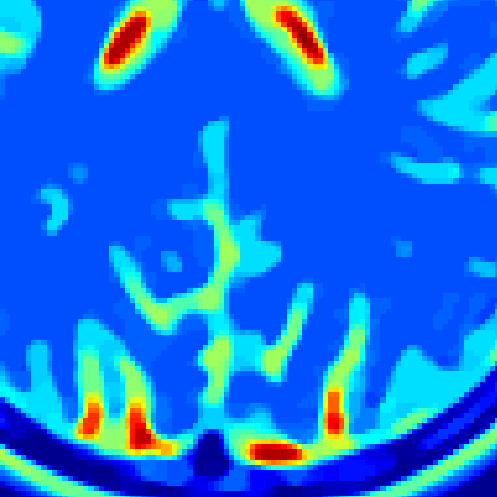}};
     \node [draw=white, thick] (fig2) at (0,0)
       {\includegraphics[width=0.4\linewidth,height=0.4\linewidth]{results/phantom_brain_ncvxtv_jet-eps-converted-to.pdf}};  
\end{tikzpicture}

\includegraphics[width=\linewidth,height=\linewidth]{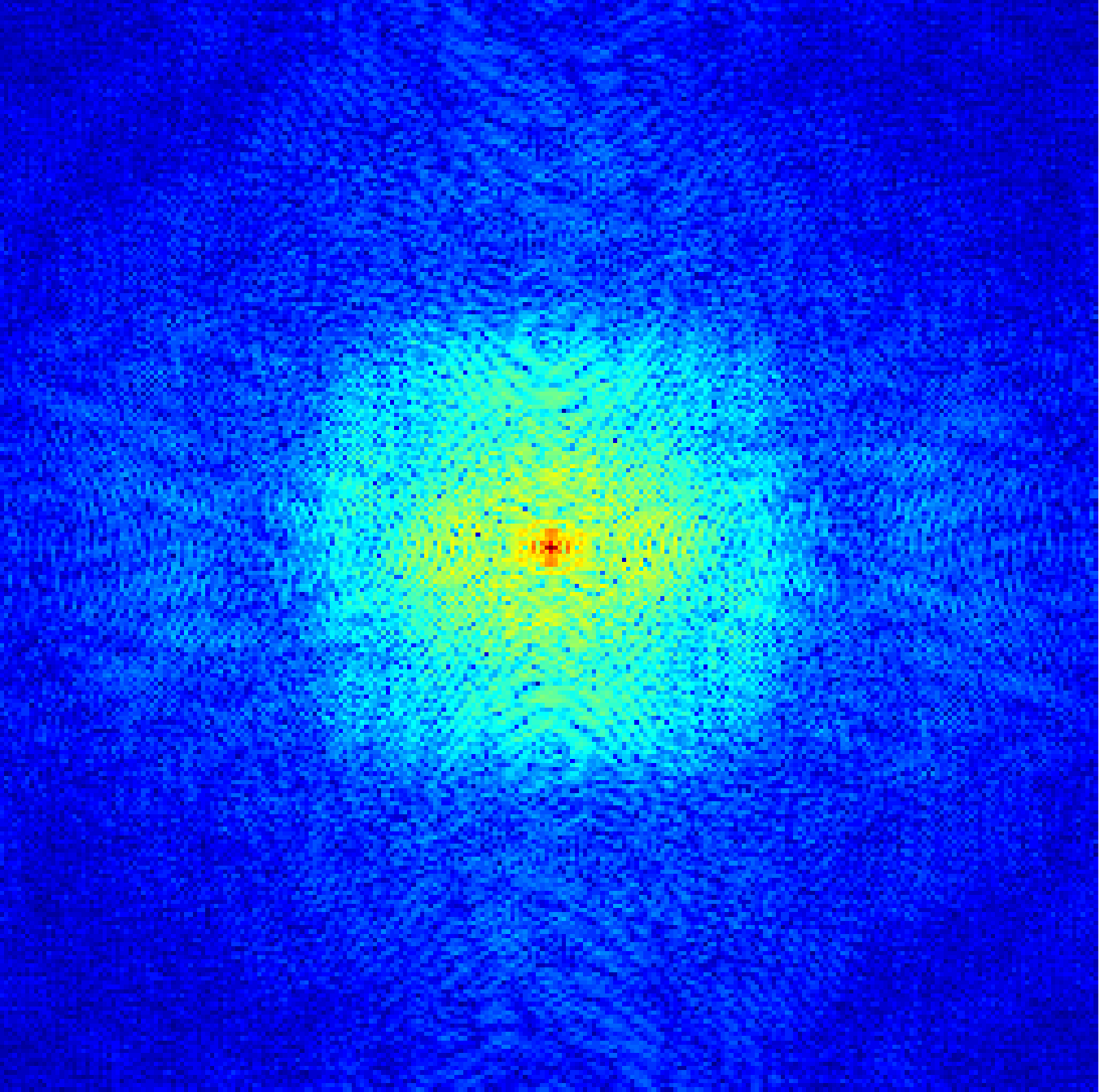}
\end{minipage}
}
\subfloat[][\scriptsize LSLP\\SNR=\SNRlslp dB]{
\hspace{-1em}
\begin{minipage}{0.16\linewidth}
\begin{tikzpicture}[      
        every node/.style={anchor=south west,inner sep=0pt},
        x=1mm, y=1mm,
      ]   
     \node (fig1) at (0,0)
       {\includegraphics[width=\linewidth,height=\linewidth,trim = 45mm 10mm 40mm 75mm, clip]{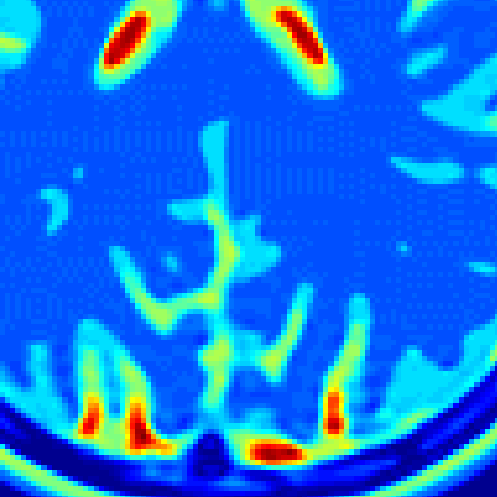}};
     \node [draw=white, thick] (fig2) at (0,0)
       {\includegraphics[width=0.4\linewidth,height=0.4\linewidth]{results/phantom_brain_lslp_jet-eps-converted-to.pdf}};  
\end{tikzpicture}

\includegraphics[width=\linewidth,height=\linewidth]{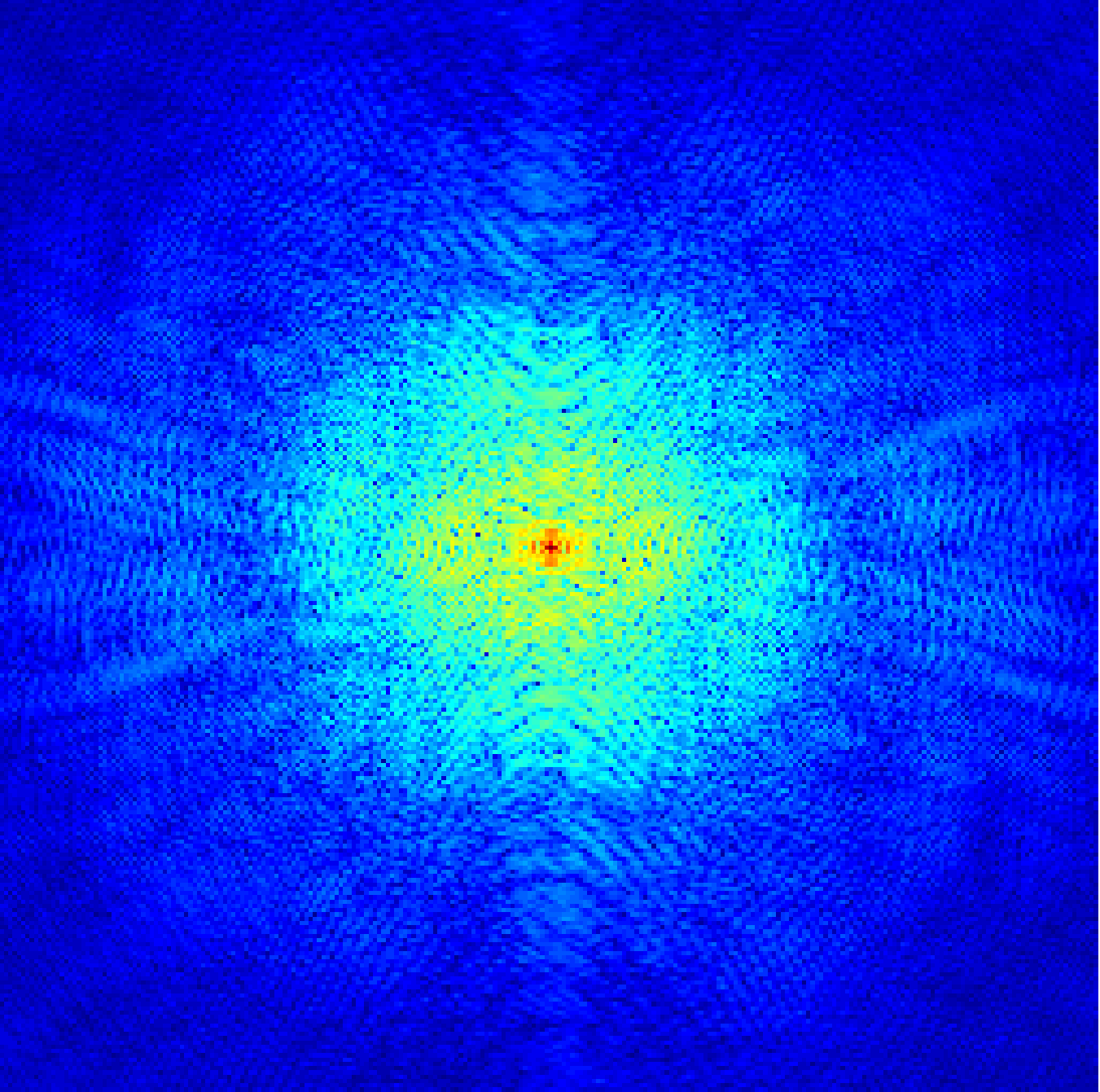}
\end{minipage}
}
\subfloat[][\scriptsize WTV\\SNR=\SNRwtv dB]{
\hspace{-1em}
\begin{minipage}{0.16\linewidth}
\begin{tikzpicture}[      
        every node/.style={anchor=south west,inner sep=0pt},
        x=1mm, y=1mm,
      ]   
     \node (fig1) at (0,0)
       {\includegraphics[width=\linewidth,height=\linewidth,trim = 45mm 10mm 40mm 75mm, clip]{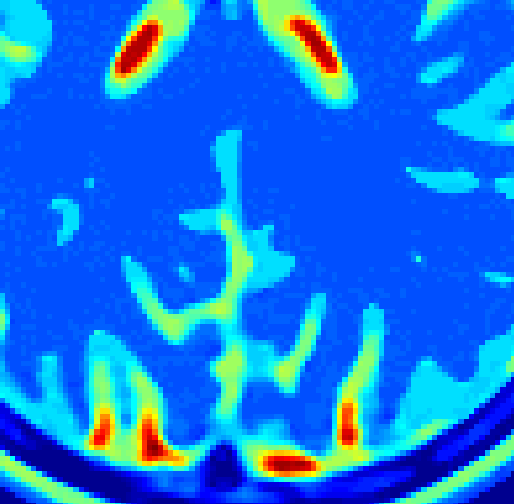}};
     \node [draw=white, thick] (fig2) at (0,0)
       {\includegraphics[width=0.4\linewidth,height=0.4\linewidth]{results/phantom_brain_wtv_jet-eps-converted-to.pdf}};  
\end{tikzpicture}

\includegraphics[width=\linewidth]{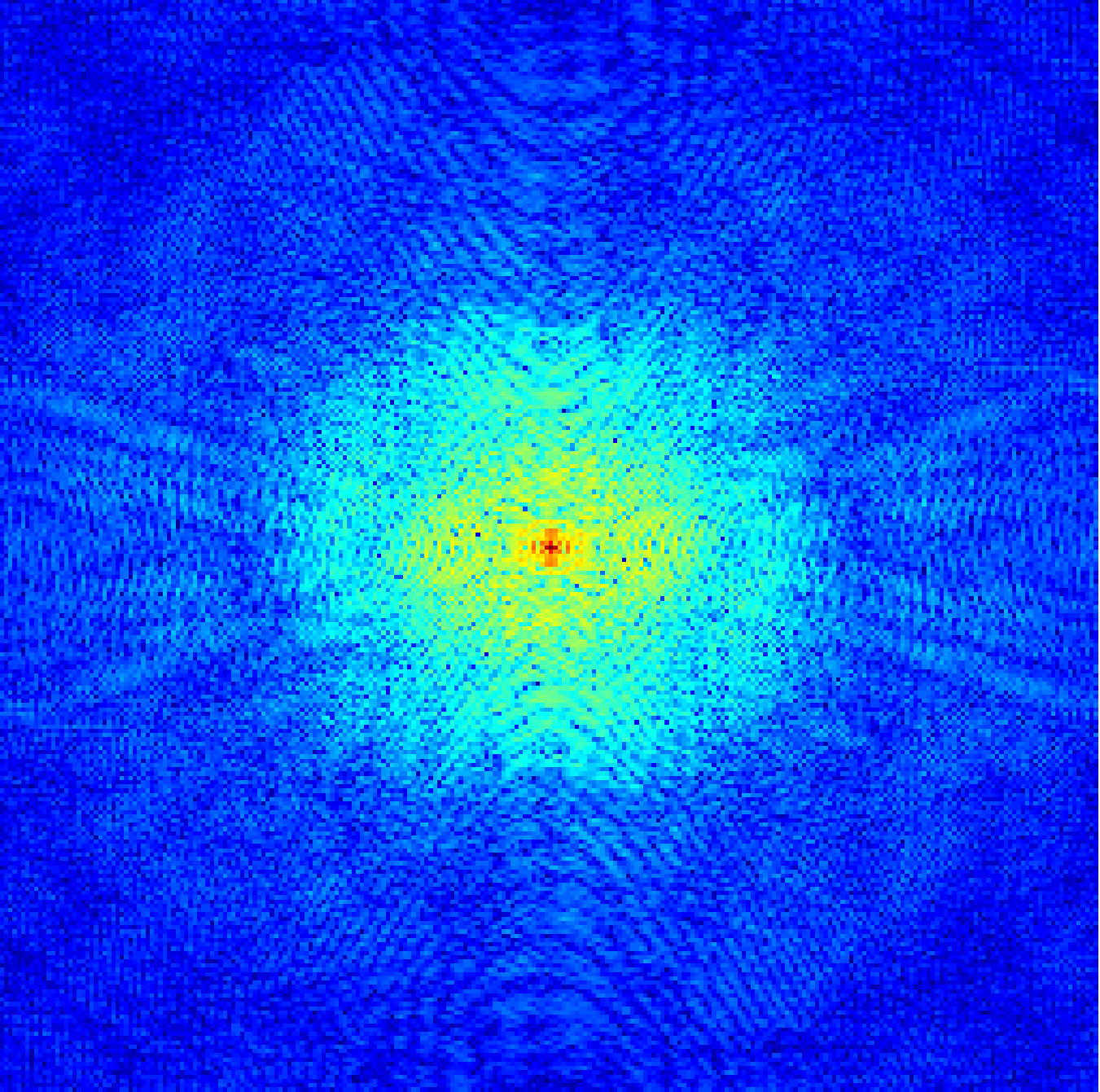}
\end{minipage}
}
\begin{minipage}{0.5cm}
\raisebox{0cm}{
\hspace{-1.5em}
\setlength\figureheight{2.2cm}
%
%
\begin{tikzpicture}

\begin{axis}[
    hide axis,
    scale only axis,
    width=1em,
    height=3em,
    point meta min=0,
    point meta max=1,
    colorbar,
    colormap/jet,
    colorbar style={
        height=\figureheight,
        width=0.7em,
        ytick={0,0.2,...,1},
        yticklabel style={
            xshift = -0.5ex,
            font = \tiny
        }
    }]
    \addplot [draw=none] coordinates {(0,0)};
\end{axis}
\end{tikzpicture}%
}

\raisebox{0.2cm}{
\hspace{-1.5em}
\setlength\figureheight{2.2cm}
%
%
\begin{tikzpicture}

\begin{axis}[
    hide axis,
    scale only axis,
    width=1em,
    height=3em,
    point meta min=0,
    point meta max=9,
    colorbar,
    colormap/jet,
    colorbar style={
        height=\figureheight,
        width=0.7em,
        ytick={0,1,...,9},
        yticklabel style={
            xshift = -0.5ex,
            font = \tiny
        }
    }]
    \addplot [draw=none] coordinates {(0,0)};
\end{axis}
\end{tikzpicture}%
}
\end{minipage}
\\
\subfloat[][\scriptsize Edge set estimate, $33$$\times$$25$ coefficients]{
\hspace{2em}
\begin{tikzpicture}
  \node[anchor=south west,inner sep=0] (image) at (0,0) {\includegraphics[height=0.25\textwidth]{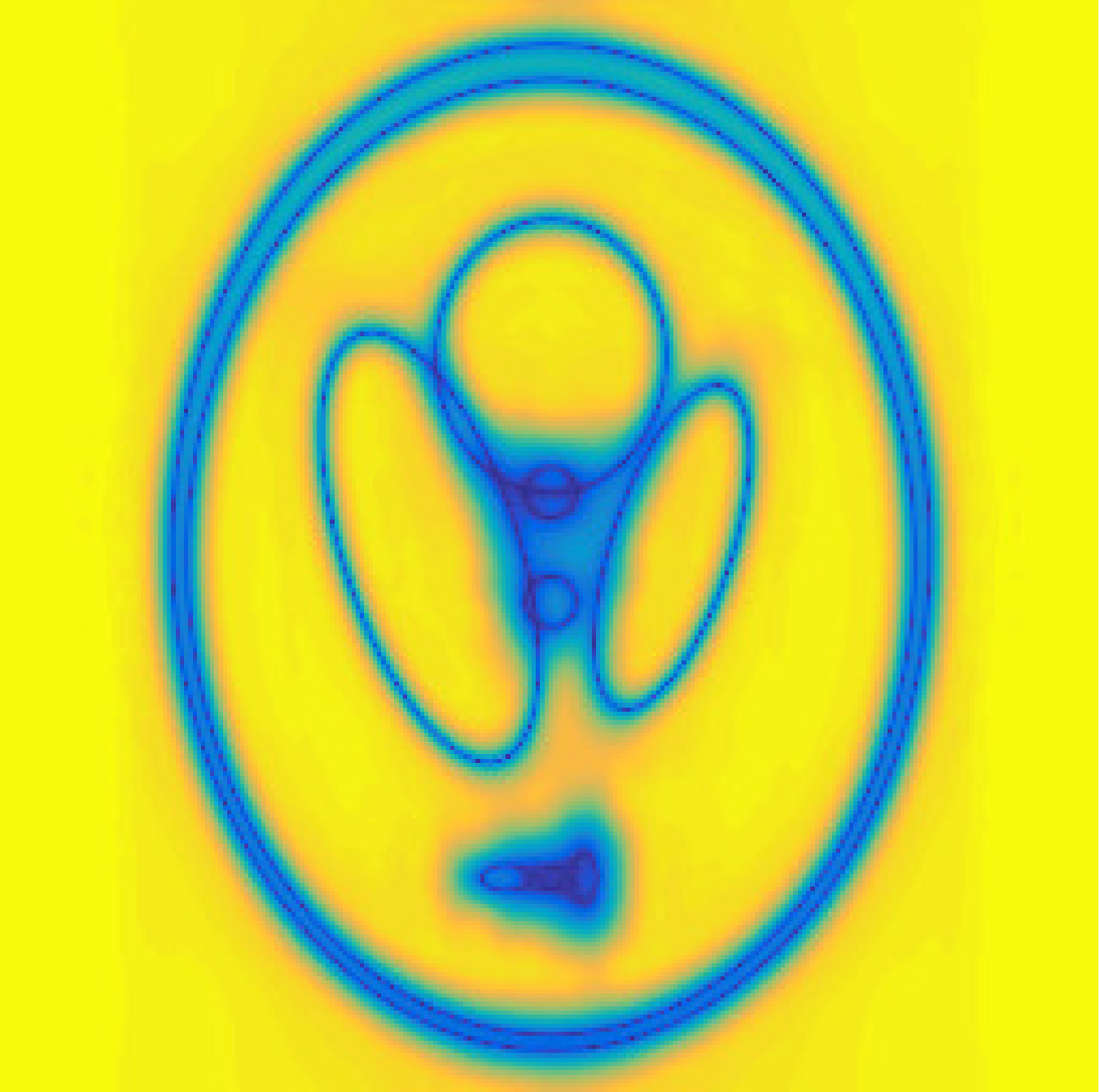}};
\end{tikzpicture}
\raisebox{-0.5em}{
\hspace{-1.5em}
\setlength\figureheight{0.25\textwidth}
%
%
\begin{tikzpicture}

\pgfplotsset{
    colormap={parula}{
        rgb=(0.2081,0.1663,0.5292)
        rgb=(0.2116,0.1898,0.5777)
        rgb=(0.2123,0.2138,0.627)
        rgb=(0.2081,0.2386,0.6771)
        rgb=(0.1959,0.2645,0.7279)
        rgb=(0.1707,0.2919,0.7792)
        rgb=(0.1253,0.3242,0.8303)
        rgb=(0.0591,0.3598,0.8683)
        rgb=(0.0117,0.3875,0.882)
        rgb=(0.006,0.4086,0.8828)
        rgb=(0.0165,0.4266,0.8786)
        rgb=(0.0329,0.443,0.872)
        rgb=(0.0498,0.4586,0.8641)
        rgb=(0.0629,0.4737,0.8554)
        rgb=(0.0723,0.4887,0.8467)
        rgb=(0.0779,0.504,0.8384)
        rgb=(0.0793,0.52,0.8312)
        rgb=(0.0749,0.5375,0.8263)
        rgb=(0.0641,0.557,0.824)
        rgb=(0.0488,0.5772,0.8228)
        rgb=(0.0343,0.5966,0.8199)
        rgb=(0.0265,0.6137,0.8135)
        rgb=(0.0239,0.6287,0.8038)
        rgb=(0.0231,0.6418,0.7913)
        rgb=(0.0228,0.6535,0.7768)
        rgb=(0.0267,0.6642,0.7607)
        rgb=(0.0384,0.6743,0.7436)
        rgb=(0.059,0.6838,0.7254)
        rgb=(0.0843,0.6928,0.7062)
        rgb=(0.1133,0.7015,0.6859)
        rgb=(0.1453,0.7098,0.6646)
        rgb=(0.1801,0.7177,0.6424)
        rgb=(0.2178,0.725,0.6193)
        rgb=(0.2586,0.7317,0.5954)
        rgb=(0.3022,0.7376,0.5712)
        rgb=(0.3482,0.7424,0.5473)
        rgb=(0.3953,0.7459,0.5244)
        rgb=(0.442,0.7481,0.5033)
        rgb=(0.4871,0.7491,0.484)
        rgb=(0.53,0.7491,0.4661)
        rgb=(0.5709,0.7485,0.4494)
        rgb=(0.6099,0.7473,0.4337)
        rgb=(0.6473,0.7456,0.4188)
        rgb=(0.6834,0.7435,0.4044)
        rgb=(0.7184,0.7411,0.3905)
        rgb=(0.7525,0.7384,0.3768)
        rgb=(0.7858,0.7356,0.3633)
        rgb=(0.8185,0.7327,0.3498)
        rgb=(0.8507,0.7299,0.336)
        rgb=(0.8824,0.7274,0.3217)
        rgb=(0.9139,0.7258,0.3063)
        rgb=(0.945,0.7261,0.2886)
        rgb=(0.9739,0.7314,0.2666)
        rgb=(0.9938,0.7455,0.2403)
        rgb=(0.999,0.7653,0.2164)
        rgb=(0.9955,0.7861,0.1967)
        rgb=(0.988,0.8066,0.1794)
        rgb=(0.9789,0.8271,0.1633)
        rgb=(0.9697,0.8481,0.1475)
        rgb=(0.9626,0.8705,0.1309)
        rgb=(0.9589,0.8949,0.1132)
        rgb=(0.9598,0.9218,0.0948)
        rgb=(0.9661,0.9514,0.0755)
        rgb=(0.9763,0.9831,0.0538)
    }
}

\begin{axis}[
    hide axis,
    scale only axis,
    width=1em,
    height=3em,
    point meta min=0,
    point meta max=1,
    colorbar,
    colorbar style={
        height=\figureheight,
        width=0.7em,
        ytick={0,0.2,...,1},
        yticklabel style={
            xshift = -0.5ex,
            font = \tiny
        }
    }]
    \addplot [draw=none] coordinates {(0,0)};
\end{axis}
\end{tikzpicture}%
}
}
\subfloat[][\scriptsize Edge set estimate, $49$$\times$$49$ coefficients]{
\hspace{2em}
\begin{tikzpicture}
  \node[anchor=south west,inner sep=0] (image) at (0,0) {\includegraphics[height=0.25\textwidth]{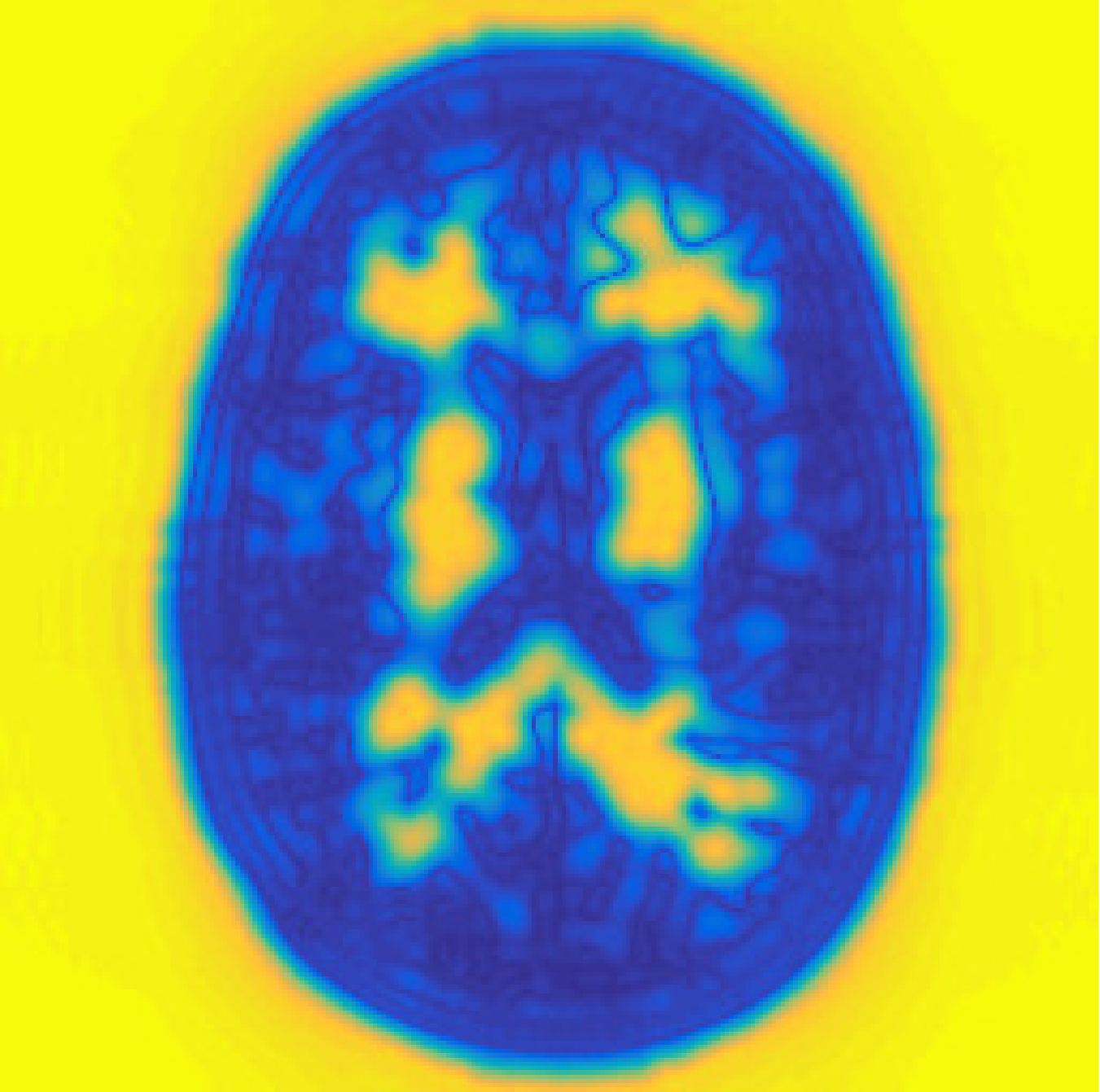}};
\end{tikzpicture}
\raisebox{-0.5em}{
\hspace{-1.5em}
\setlength\figureheight{0.25\textwidth}
%
%
\begin{tikzpicture}

\pgfplotsset{
    colormap={parula}{
        rgb=(0.2081,0.1663,0.5292)
        rgb=(0.2116,0.1898,0.5777)
        rgb=(0.2123,0.2138,0.627)
        rgb=(0.2081,0.2386,0.6771)
        rgb=(0.1959,0.2645,0.7279)
        rgb=(0.1707,0.2919,0.7792)
        rgb=(0.1253,0.3242,0.8303)
        rgb=(0.0591,0.3598,0.8683)
        rgb=(0.0117,0.3875,0.882)
        rgb=(0.006,0.4086,0.8828)
        rgb=(0.0165,0.4266,0.8786)
        rgb=(0.0329,0.443,0.872)
        rgb=(0.0498,0.4586,0.8641)
        rgb=(0.0629,0.4737,0.8554)
        rgb=(0.0723,0.4887,0.8467)
        rgb=(0.0779,0.504,0.8384)
        rgb=(0.0793,0.52,0.8312)
        rgb=(0.0749,0.5375,0.8263)
        rgb=(0.0641,0.557,0.824)
        rgb=(0.0488,0.5772,0.8228)
        rgb=(0.0343,0.5966,0.8199)
        rgb=(0.0265,0.6137,0.8135)
        rgb=(0.0239,0.6287,0.8038)
        rgb=(0.0231,0.6418,0.7913)
        rgb=(0.0228,0.6535,0.7768)
        rgb=(0.0267,0.6642,0.7607)
        rgb=(0.0384,0.6743,0.7436)
        rgb=(0.059,0.6838,0.7254)
        rgb=(0.0843,0.6928,0.7062)
        rgb=(0.1133,0.7015,0.6859)
        rgb=(0.1453,0.7098,0.6646)
        rgb=(0.1801,0.7177,0.6424)
        rgb=(0.2178,0.725,0.6193)
        rgb=(0.2586,0.7317,0.5954)
        rgb=(0.3022,0.7376,0.5712)
        rgb=(0.3482,0.7424,0.5473)
        rgb=(0.3953,0.7459,0.5244)
        rgb=(0.442,0.7481,0.5033)
        rgb=(0.4871,0.7491,0.484)
        rgb=(0.53,0.7491,0.4661)
        rgb=(0.5709,0.7485,0.4494)
        rgb=(0.6099,0.7473,0.4337)
        rgb=(0.6473,0.7456,0.4188)
        rgb=(0.6834,0.7435,0.4044)
        rgb=(0.7184,0.7411,0.3905)
        rgb=(0.7525,0.7384,0.3768)
        rgb=(0.7858,0.7356,0.3633)
        rgb=(0.8185,0.7327,0.3498)
        rgb=(0.8507,0.7299,0.336)
        rgb=(0.8824,0.7274,0.3217)
        rgb=(0.9139,0.7258,0.3063)
        rgb=(0.945,0.7261,0.2886)
        rgb=(0.9739,0.7314,0.2666)
        rgb=(0.9938,0.7455,0.2403)
        rgb=(0.999,0.7653,0.2164)
        rgb=(0.9955,0.7861,0.1967)
        rgb=(0.988,0.8066,0.1794)
        rgb=(0.9789,0.8271,0.1633)
        rgb=(0.9697,0.8481,0.1475)
        rgb=(0.9626,0.8705,0.1309)
        rgb=(0.9589,0.8949,0.1132)
        rgb=(0.9598,0.9218,0.0948)
        rgb=(0.9661,0.9514,0.0755)
        rgb=(0.9763,0.9831,0.0538)
    }
}

\begin{axis}[
    hide axis,
    scale only axis,
    width=1em,
    height=3em,
    point meta min=0,
    point meta max=1,
    colorbar,
    colorbar style={
        height=\figureheight,
        width=0.7em,
        ytick={0,0.2,...,1},
        yticklabel style={
            xshift = -0.5ex,
            font = \tiny
        }
    }]
    \addplot [draw=none] coordinates {(0,0)};
\end{axis}
\end{tikzpicture}%
}
}
\caption{\small \emph{Super-resolution of piecewise constant phantoms from ideal Fourier samples}. Recovery of Shepp-Logan phantom from $65\times 49$ ideal low-pass Fourier samples, and a brain phantom reconstructed onto $256$$\times$$256$ spatial grid reconstructed from $97\times 97$ samples. In (a)-(l), spatial domain is shown above, and Fourier domain below (in log scale). Shown in (a)\&(g) are the fully sampled images, (b)\&(h) is the IFFT of the zero-padded samples, (c)\&(i) is a standard TV regularized recovery, (d)\&(j) is the non-convex TV used in \cite{chartrand2011frequency}, (e)\&(k) is the recovery using the proposed least squares linear prediction (LSLP) method, and (f)\&(l) the recovery using the proposed weighted-TV (WTV) formulation. Note the the LSLP more accurately extrapolates the true Fourier coefficients of the image. The sum-of-squares polynomial used as the edge set estimate for the proposed schemes in shown (m) and (n).
}
\label{fig:phantoms}
\end{figure*}

\section{Experiments}
\label{sec:exp}
In this section we demonstrate the utility of our proposed algorithms for super-resolution in magnetic resonance (MR) imaging. A single-coil MR acquisition restricted to a single plane with sampling locations on a uniform Cartesian grid can be accurately modeled as the Fourier coefficients of the underlying 2-D image  \cite{fessler2010model}. This suggests we can apply our proposed two-stage recovery scheme to recover MR data from low-pass Fourier samples at high resolution, provided the underlying image is well-approximated as piecewise constant.

First, to test the validity of the proposed scheme, we experiment on simulated data generated from piecewise constant phantoms. To ensure resolution-independence, we compute ideal Fourier samples from analytical Fourier domain expressions of the MR phantoms, as done in \cite{guerquin2012realistic}. Taking low-pass samples located within a centered rectangular window, we solve for an annihilating filter of the maximum possible size allowable using the procedure outlined in Section \ref{sec:alg}. We then recover the image using the proposed algorithms in \S\ref{sec:alg}, and compare against a standard total variation (TV) regularized recovery, and a non-convex version of TV (ncvx-TV). For ncvx-TV, we adopt the approach in \cite{chartrand2011frequency,chartrand2009fast} using a non-convex $\ell^p$, $p < 1$, penalty in place of the $\ell^1$ penalty, and fix $p=1/2$ in our experiments. For all recovery schemes, we tune regularization parameters to optimize the signal-to-noise ratio (SNR), defined as $\text{SNR} = 20\log_{10}(\|\mbf x_0\|_2/\|\mbf x - \mbf x_0\|_2)$ where $\mbf x$ is the reconstructed image, and $\mbf x_0$ is the ground truth image.

\begin{figure*}[ht!]
\centering
\input{results/realdata_cor2_cc_noisy2.tex}
\begin{minipage}{0.30\linewidth}
\subfloat[][Fully-sampled]{
\begin{tikzpicture}
    \node[anchor=south west,inner sep=0] (image) at (0,0) {\includegraphics[height=\linewidth,width=1.2\linewidth,angle=-90,origin=c]{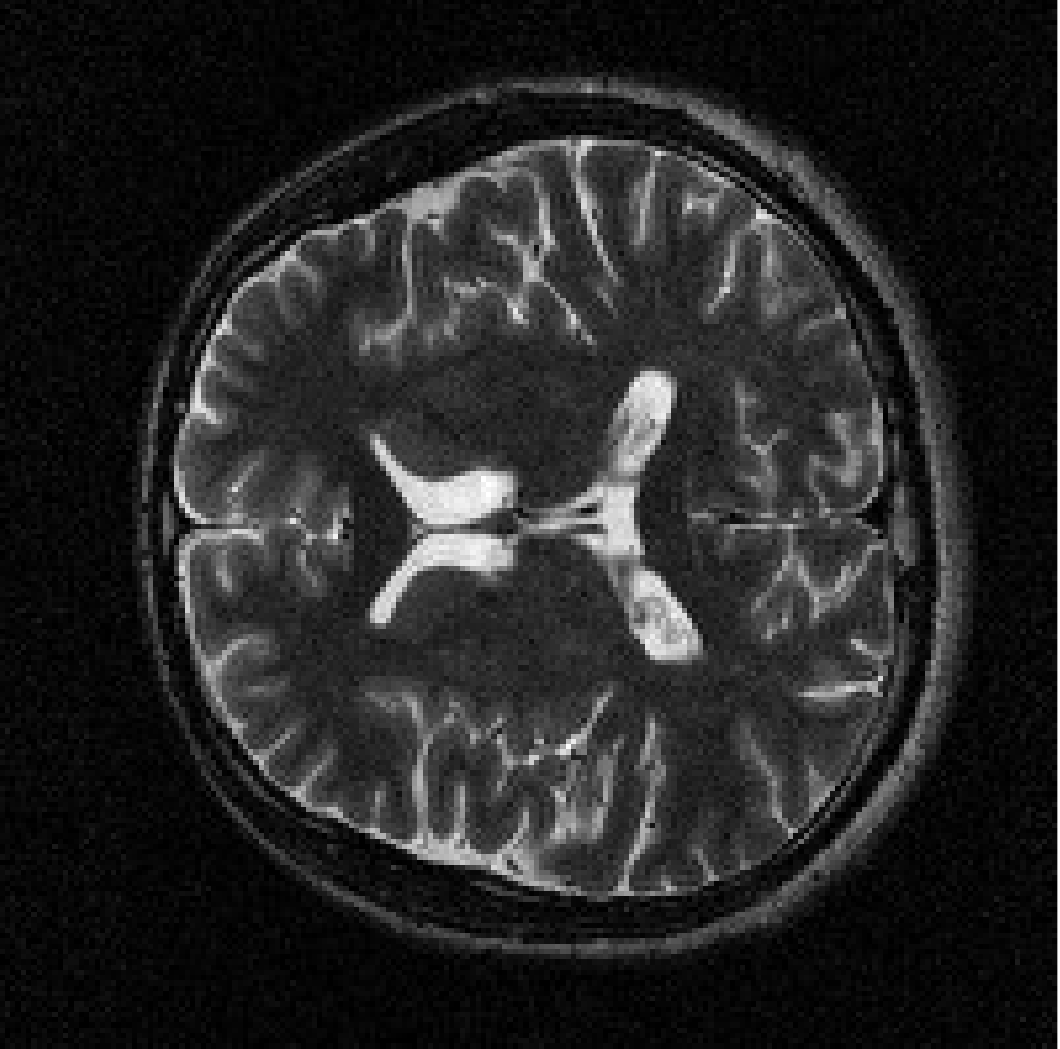}};
    \begin{scope}[x={(image.south east)},y={(image.north west)}]
        \draw[cyan,thick] (0.17,0.30) rectangle (0.76,0.80);
    \end{scope}   
\end{tikzpicture}
}

\subfloat[][Edge set estimate\\($50$$\times$$50$ coefficients)]{
	\includegraphics[width=0.96\linewidth,height=0.8\linewidth,angle=-90,origin=c]{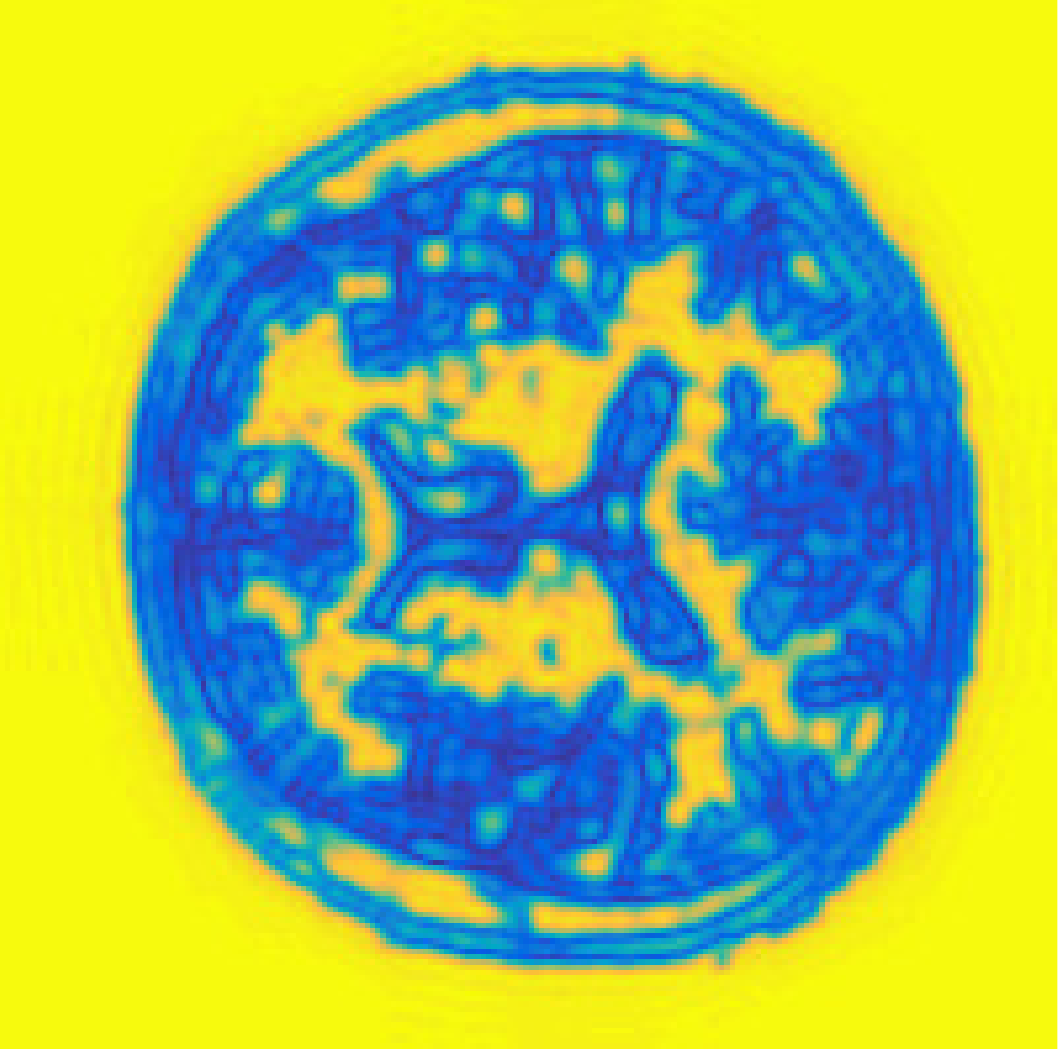} 
	\raisebox{-1.5em}{
	\hspace{-1.5em}
	\setlength\figureheight{0.96\linewidth}
%
%
\begin{tikzpicture}

\pgfplotsset{
    colormap={parula}{
        rgb=(0.2081,0.1663,0.5292)
        rgb=(0.2116,0.1898,0.5777)
        rgb=(0.2123,0.2138,0.627)
        rgb=(0.2081,0.2386,0.6771)
        rgb=(0.1959,0.2645,0.7279)
        rgb=(0.1707,0.2919,0.7792)
        rgb=(0.1253,0.3242,0.8303)
        rgb=(0.0591,0.3598,0.8683)
        rgb=(0.0117,0.3875,0.882)
        rgb=(0.006,0.4086,0.8828)
        rgb=(0.0165,0.4266,0.8786)
        rgb=(0.0329,0.443,0.872)
        rgb=(0.0498,0.4586,0.8641)
        rgb=(0.0629,0.4737,0.8554)
        rgb=(0.0723,0.4887,0.8467)
        rgb=(0.0779,0.504,0.8384)
        rgb=(0.0793,0.52,0.8312)
        rgb=(0.0749,0.5375,0.8263)
        rgb=(0.0641,0.557,0.824)
        rgb=(0.0488,0.5772,0.8228)
        rgb=(0.0343,0.5966,0.8199)
        rgb=(0.0265,0.6137,0.8135)
        rgb=(0.0239,0.6287,0.8038)
        rgb=(0.0231,0.6418,0.7913)
        rgb=(0.0228,0.6535,0.7768)
        rgb=(0.0267,0.6642,0.7607)
        rgb=(0.0384,0.6743,0.7436)
        rgb=(0.059,0.6838,0.7254)
        rgb=(0.0843,0.6928,0.7062)
        rgb=(0.1133,0.7015,0.6859)
        rgb=(0.1453,0.7098,0.6646)
        rgb=(0.1801,0.7177,0.6424)
        rgb=(0.2178,0.725,0.6193)
        rgb=(0.2586,0.7317,0.5954)
        rgb=(0.3022,0.7376,0.5712)
        rgb=(0.3482,0.7424,0.5473)
        rgb=(0.3953,0.7459,0.5244)
        rgb=(0.442,0.7481,0.5033)
        rgb=(0.4871,0.7491,0.484)
        rgb=(0.53,0.7491,0.4661)
        rgb=(0.5709,0.7485,0.4494)
        rgb=(0.6099,0.7473,0.4337)
        rgb=(0.6473,0.7456,0.4188)
        rgb=(0.6834,0.7435,0.4044)
        rgb=(0.7184,0.7411,0.3905)
        rgb=(0.7525,0.7384,0.3768)
        rgb=(0.7858,0.7356,0.3633)
        rgb=(0.8185,0.7327,0.3498)
        rgb=(0.8507,0.7299,0.336)
        rgb=(0.8824,0.7274,0.3217)
        rgb=(0.9139,0.7258,0.3063)
        rgb=(0.945,0.7261,0.2886)
        rgb=(0.9739,0.7314,0.2666)
        rgb=(0.9938,0.7455,0.2403)
        rgb=(0.999,0.7653,0.2164)
        rgb=(0.9955,0.7861,0.1967)
        rgb=(0.988,0.8066,0.1794)
        rgb=(0.9789,0.8271,0.1633)
        rgb=(0.9697,0.8481,0.1475)
        rgb=(0.9626,0.8705,0.1309)
        rgb=(0.9589,0.8949,0.1132)
        rgb=(0.9598,0.9218,0.0948)
        rgb=(0.9661,0.9514,0.0755)
        rgb=(0.9763,0.9831,0.0538)
    }
}

\begin{axis}[
    hide axis,
    scale only axis,
    width=1em,
    height=3em,
    point meta min=0,
    point meta max=1,
    colorbar,
    colorbar style={
        height=\figureheight,
        width=0.7em,
        ytick={0,0.2,...,1},
        yticklabel style={
            xshift = -0.5ex,
            font = \tiny
        }
    }]
    \addplot [draw=none] coordinates {(0,0)};
\end{axis}
\end{tikzpicture}%
	}
}
\end{minipage}
\subfloat[][Fully-sampled]{
\begin{minipage}{0.20\linewidth}
	\includegraphics[width=\linewidth,height=\linewidth,trim = 60 50 90 70,clip,angle=-90,origin=c]{results/realdata_cor2_cc_noisy2_orig_gray-eps-converted-to.pdf}
\includegraphics[width=\linewidth,height=\linewidth,angle=-90,origin=c]{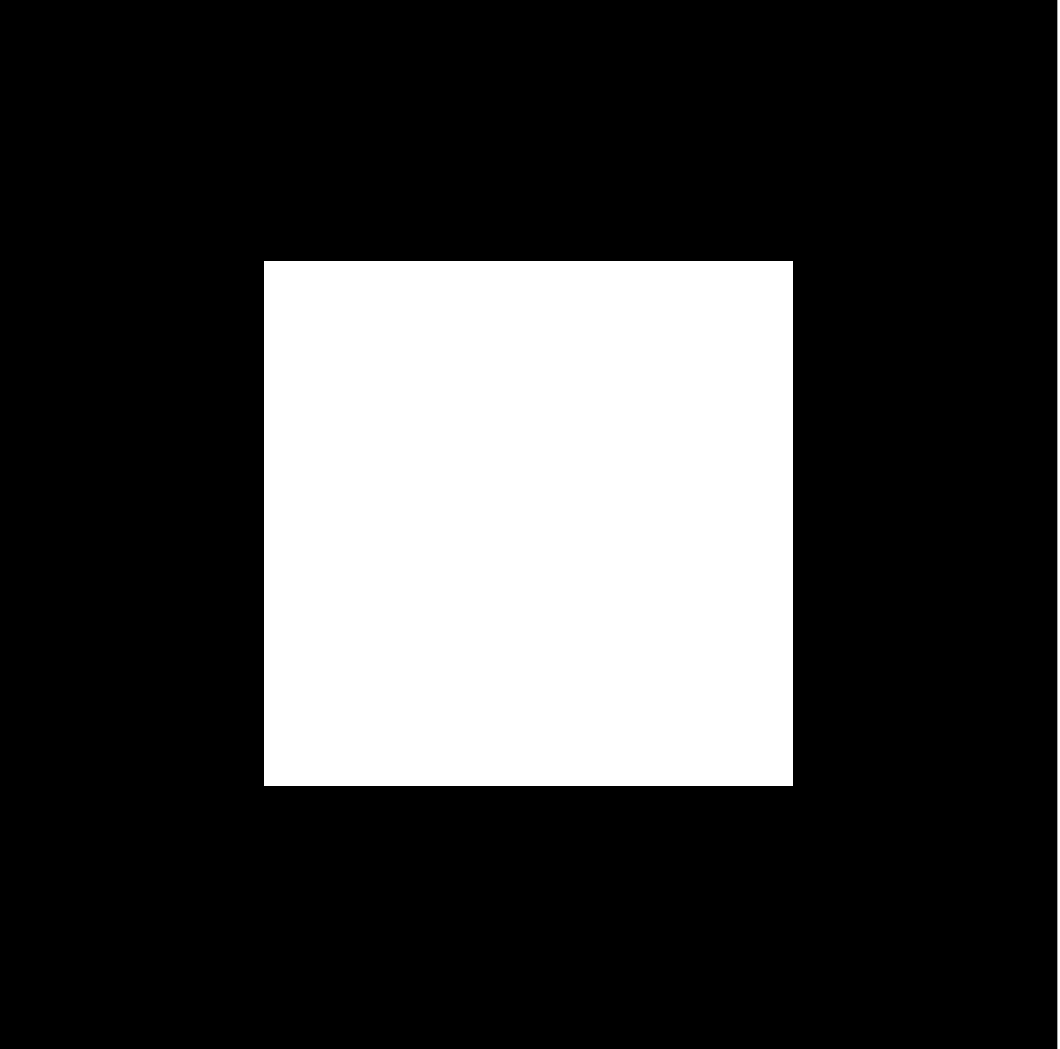}
\includegraphics[width=\linewidth,height=\linewidth,angle=-90,origin=c]{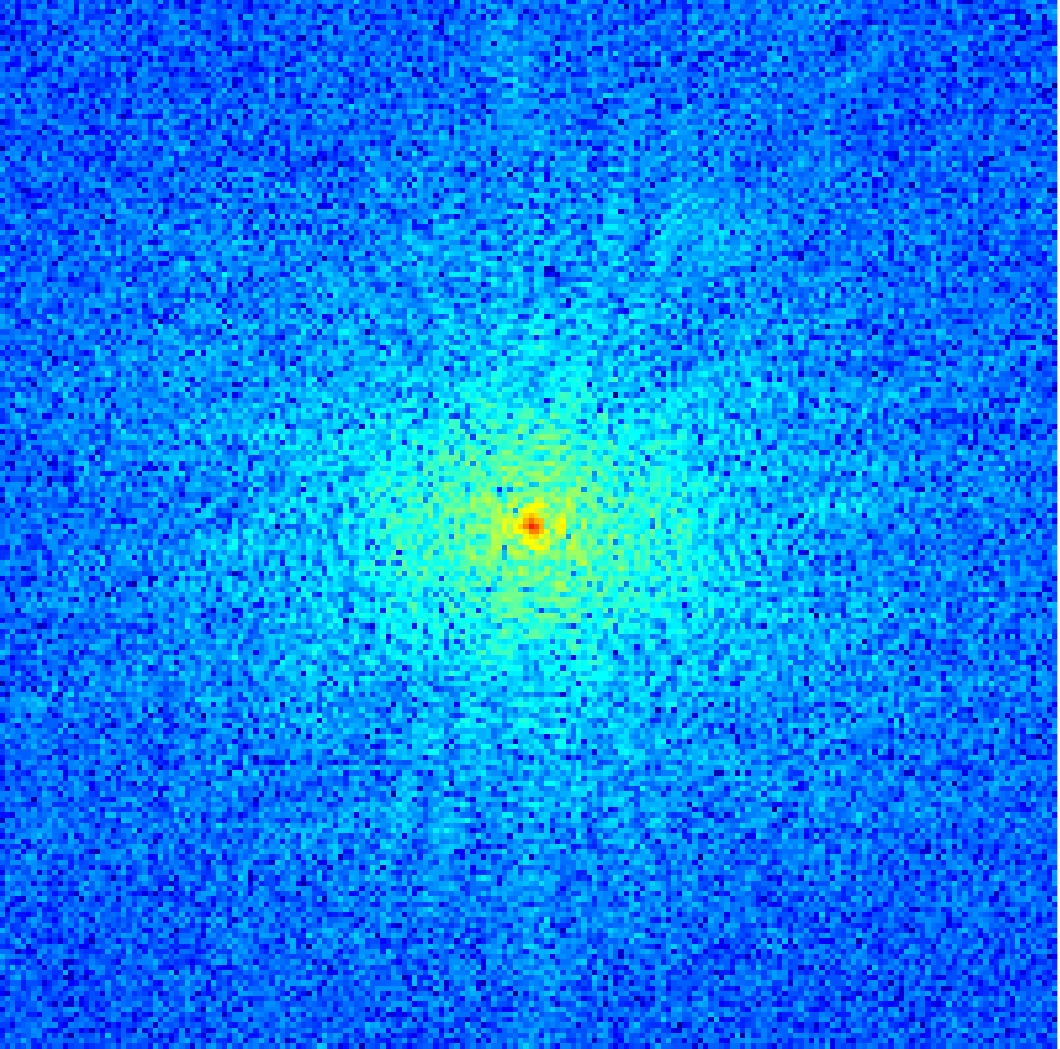}
\end{minipage}
}
\subfloat[][ncvx-TV\\SNR=\SNRncvxtv dB\\SSIM=\SSIMncvxtv]{
\begin{minipage}{0.20\linewidth}
\includegraphics[width=\linewidth,height=\linewidth,trim = 60 50 90 70,clip,angle=-90,origin=c]{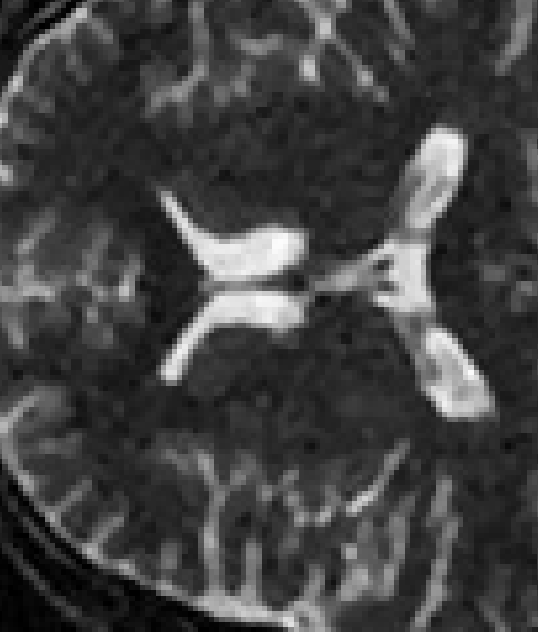}

\includegraphics[width=\linewidth,height=\linewidth,trim = 60 50 90 70,clip,angle=-90,origin=c]{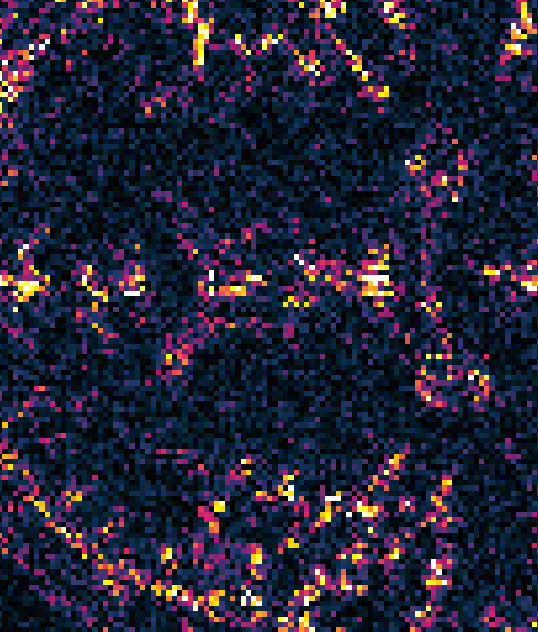}

\includegraphics[width=\linewidth,height=\linewidth,angle=-90,origin=c]{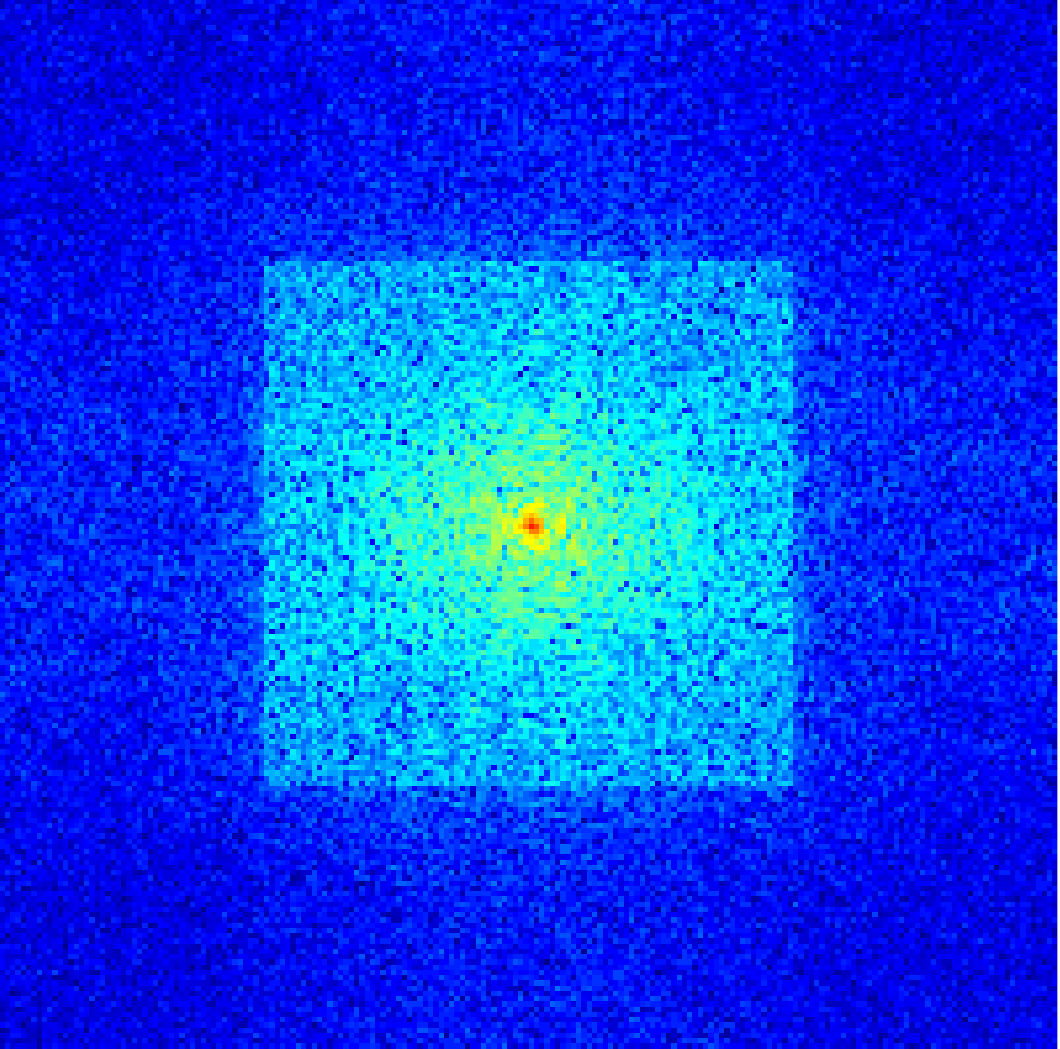}
\end{minipage}
}
\subfloat[][Proposed, LSLP\\SNR=\SNRlslp dB\\SSIM=\SSIMlslp]{
\begin{minipage}{0.20\linewidth}
\includegraphics[width=\linewidth,height=\linewidth,trim = 60 50 90 70,clip,angle=-90,origin=c]{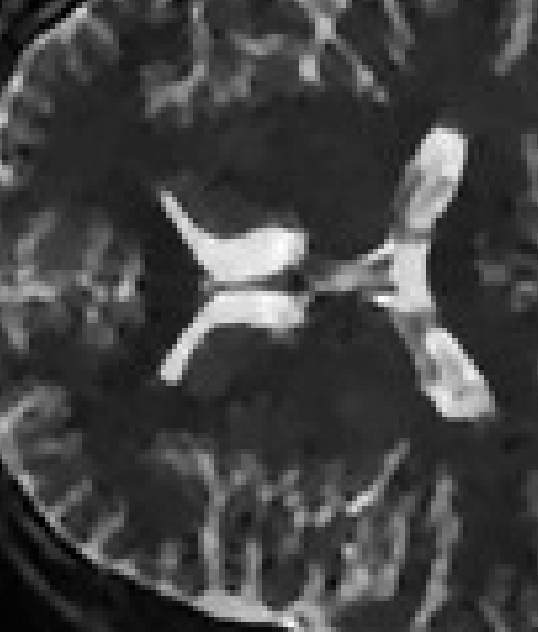}

\includegraphics[width=\linewidth,height=\linewidth,trim = 60 50 90 70,clip,angle=-90,origin=c]{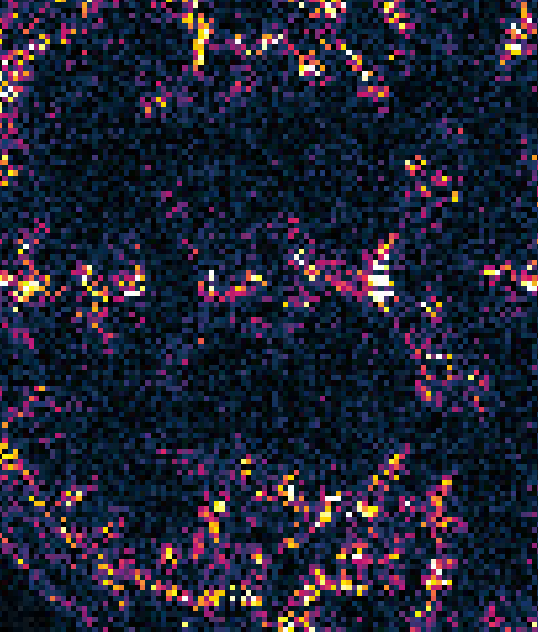}

\includegraphics[width=\linewidth,height=\linewidth,angle=-90,origin=c]{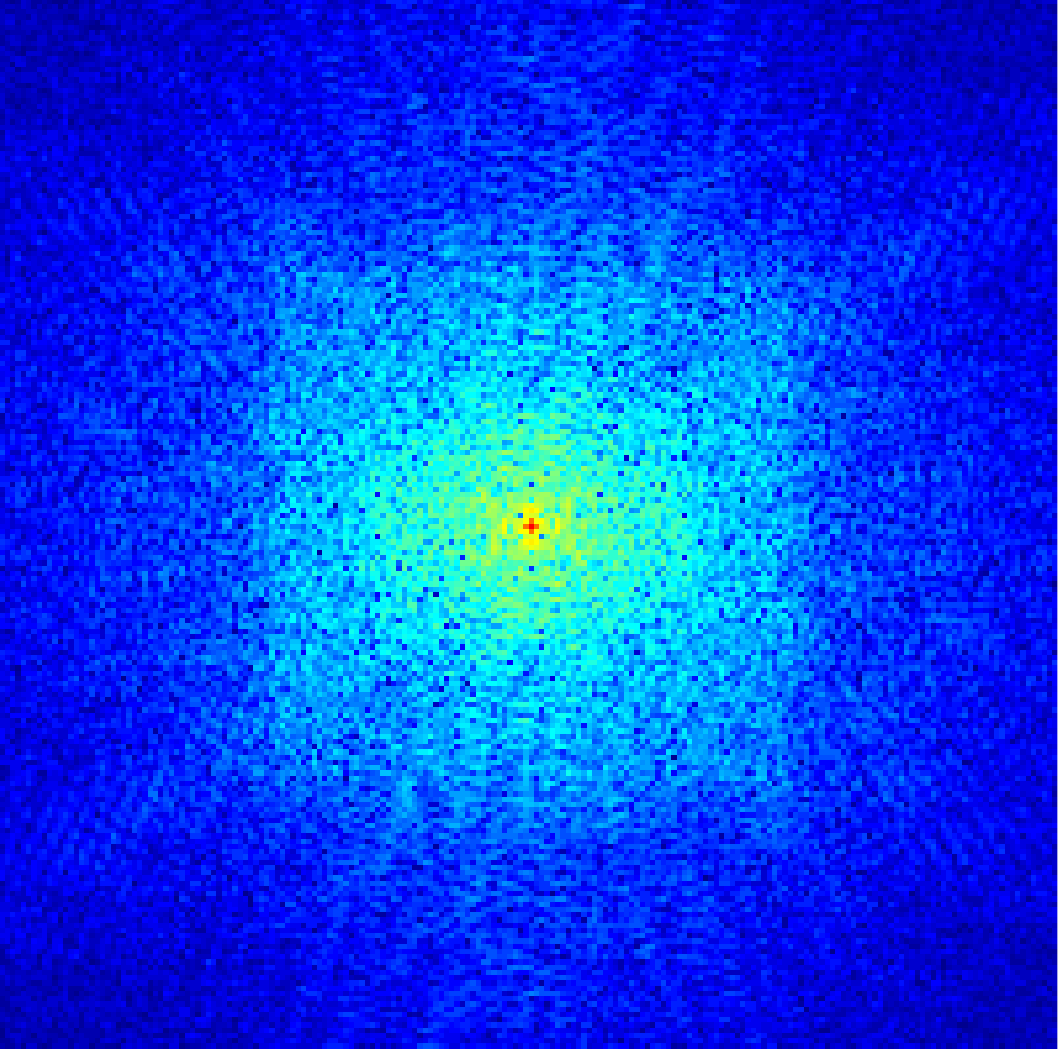}
\end{minipage}
}
\begin{minipage}{0.5cm}
\vspace{9em}
\raisebox{0cm}{
\hspace{-1.5em}
\setlength\figureheight{2.9cm}
%
%
\begin{tikzpicture}

\pgfplotsset{
    colormap={morgenstemning}{
        rgb=(0,0,0)
        rgb=(0.0980,0.2078,0.3725)
        rgb=(0.7529,0.1059,0.4353)
        rgb=(0.9882,0.8980,0)
        rgb=(1,1,1)
    }
}

\begin{axis}[
    hide axis,
    scale only axis,
    width=1em,
    height=3em,
    point meta min=0,
    point meta max=0.2,
    colorbar,
    colorbar style={
        height=\figureheight,
        width=0.7em,
        ytick={0,0.1,0.2},
        yticklabel style={
            xshift = -0.5ex,
            font = \tiny
        }
    }]
    \addplot [draw=none] coordinates {(0,0)};
\end{axis}
\end{tikzpicture}%
}

\raisebox{0.3cm}{
\hspace{-1.5em}
\setlength\figureheight{2.9cm}
%
%
\begin{tikzpicture}

\begin{axis}[
    hide axis,
    scale only axis,
    width=1em,
    height=3em,
    point meta min=0,
    point meta max=9,
    colorbar,
    colormap/jet,
    colorbar style={
        height=\figureheight,
        width=0.7em,
        ytick={0,1,...,9},
        yticklabel style={
            xshift = -0.5ex,
            font = \tiny
        }
    }]
    \addplot [draw=none] coordinates {(0,0)};
\end{axis}
\end{tikzpicture}%
}
\end{minipage}
\caption{\small Super-resolution of a single-coil MR brain image from retrospectively undersampled data. Reconstruction from $100\times 100$ center Fourier samples taken from a $200\times 200$ acquisition reconstructed onto a $200$$\times$$200$ spatial grid. The original fully sampled data is shown in (a), and zoomed in (c). A nonconvex TV regularized recovery is shown in (d), and the proposed LSLP recovery is shown in (e) which uses the sum-of-squares polynomial shown in (b) as an edge set estimate. Error images are shown in middle row of (d) and (e), and the Fourier domain recovery in log scale is shown in the bottom row. Note the proposed scheme more faithfully extrapolates the true Fourier data.}
\label{fig:realbrain_cor}
\end{figure*}

The results of this first experiment are shown in Fig.\ \ref{fig:phantoms}. Our proposed scheme is able to recover the Shepp-Logan phantom with high accuracy from as few as $65\times 49$ low-pass Fourier samples, corresponding to a 20-fold undersampling. A TV-regularized recovery from the same samples shows significant edge blurring, broadening of image features, and ringing-like artifacts in piecewise constant regions. We observe the non-convex TV approach is better able to reduce artifacts in the constant regions of the image, but still distorts the boundaries of the piecewise constant regions. We see similar benefits in the case of a brain phantom having a more complicated edge set geometry. Visually, the WTV reconstruction appears to show some advantages over LSLP in this case, in that it is better able to recovery some of the finer image features and suppress ringing artifacts. However, the LSLP recovery more faithfully recovers the true Fourier coefficients, as measured by the SNR. 

We also performed another experiment validating the method on real MR data; see Figure \ref{fig:realbrain_cor}. The dataset was obtained from a fully sampled 4-coil acquisition, which we compressed into a single virtual coil using an SVD-based technique \cite{zhang2013coil}. We then retrospectively collected low-pass Fourier samples from the single virtual coil. To simulate a noisy acquisition, we also added i.i.d.\ complex white Gaussian noise to the low-pass Fourier samples, such that the resulting SNR of the samples was approximately 30 dB. 

We observed that the data from the single virtual coil was complex-valued in image domain with smoothly varying phase over its support, as is typical with MR data. Applying our edge-set estimation procedure directly on the unprocessed low-pass samples gave poor results, since our approach assumes the underlying image is piecewise constant. Therefore, prior to estimating the edge-set, we performed a phase correction step. This was done by taking the inverse DFT of the zero-padded low-pass Fourier samples, canceling out the phase in image domain, and passing back to Fourier domain. We note this pre-processing step only requires access to the low-pass samples. To ensure fair comparisons, when computing image metrics we compared the magnitude images of the reconstructions against the magnitude image obtained from the fully sampled virtual coil. 

Additional noise and model-mismatch in real MR data makes estimating the true model order of the edge-set challenging. As a rule of thumb, we chose the annihilating filter dimensions to be half the size of the low-resolution input. We then performed the annihilating subspace denoising step \eqref{eq:cadzow_new2} with the rank threshold $r$ set to half the number of the coefficients in the filter. Empirically, we find this procedure is still able to accurately recover the significant edges in the image, as shown in Fig. \ref{fig:realbrain_cor}(b). We then pass the ``denoised'' Fourier samples that result from \eqref{eq:cadzow_new2} as the input to the second stage of our recovery scheme. 

For the real MR data we find our proposed LSLP approach to be superior to our WTV approach, hence we restrict our comparison to this case. In Fig. \ref{fig:realbrain_cor} we compare the proposed approach with ncvx-TV, which was found to outperform standard TV (not shown). In addition to SNR, we also use the structural similarity index (SSIM) as a measure of the perceptual-quality \cite{wang2004image}. We observe that on this dataset our proposed LSLP scheme shows some advantages over ncvx-TV in terms of SNR and SSIM. Visually, the LSLP recovery shows fewer noise artifacts in the constant regions of the image, and better preservation of edge features. Additionally, the Fourier domain images in Fig.\ \ref{fig:realbrain_cor} indicate the LSLP approach better extrapolates the low-pass Fourier data beyond the sampling region.

Finally, we remark that the computation time of the proposed scheme is primarily dominated by the cost in computing the SVD of the annihilation matrix in the edge-recovery stage. The dimensions of the annihilation matrix grow with the size of the annihilation filter and the number of low-resolution samples. For example, given an $128\times 128$ array of low-resolution samples and a filter size of $64\times 64$, the annihilation matrix has dimensions $8450 \times 4096$. The second stage of the recovery using the LSLP and WTV approaches were similar in run time to a standard TV regularized recovery, and a MATLAB implementation running on CPU took under a minute for the datasets considered here.

\section{Discussion}
In this work, we developed sampling guarantees for the unique recovery of a continuous domain piecewise constant image from few of its low-pass Fourier samples, assuming the edge set of the image is described by the zero set of a trigonometric polynomial. This is achieved by showing the distributional derivatives of such an image satisfy a linear annihilation relation in Fourier domain. One of our main contributions is to prove that, under ideal conditions, unique recovery of the image is possible from the annihilation relation. 

We also proposed a two step algorithm for the recovery of images, which extrapolates the known Fourier samples of the image in Fourier domain. We argue such a recovery scheme deserves to be called ``off-the-grid'', since we recover the exact Fourier coefficients of the underlying continuous domain image, and at no point do we need to discretize in spatial domain. In the first step of the algorithm, we estimate a collection of annihilating filters from the low-pass samples, which jointly encode the edge set of the image. In the second step, we perform a least squares linear prediction using these filters to extrapolate the signal Fourier domain. We show the proposed scheme has advantages over traditional single-image super-resolution schemes in MRI in its ability to suppress noise artifacts and preserve strong image edges.

While our experiments focused on the problem of super-resolution MRI, we note that this work has potentially wide-ranging applications beyond those presented here. For instance, the first stage our algorithm estimates a continuous domain representation of the edge set of an image, which can be useful for image segmentation. Additionally, the super-resolution problem we considered is not limited to MRI, and will work in any setting where one has a good estimate of the low-pass Fourier samples of image, such as in the image zoom problem studied in \cite{pan2013sampling}.

In future work, we plan to resolve the conjectures put forth in \S\ref{sec:trigcurves}, which would extend our sampling theorems to piecewise constant images with edge set given by an arbitrary trigonometric curve, not just those that are non-intersecting. We also hope to give further insight into the robustness of this scheme with respect to noise and model mismatch, especially with regard to the rank properties of the annihilation matrix.

Finally, we observe that most of the theory presented here extends directly to higher dimensions. For example, in three dimensions we could consider piecewise constant signals whose gradient is supported on a two-dimensional surface given as the zero set of a trigonometric polynomial in three variables. However, the computational burden in this case will be much more significant. For instance, the edge set estimation step would require the singular value decomposition of an extremely large matrix, which we may not even be able to hold in memory. Therefore a goal of future research is to come up with efficient algorithms to further enable the recovery of higher dimensional signals.

\appendix
\section{Algebraic Properties of Trigonometric Polynomials and Curves}
\label{sec:appendix_trig}
Here we develop the algebraic properties of trigonometric polynomials and trigonometric curves in some detail. We will slightly expand our definition of a \emph{trigonometric polynomial} \eqref{eq:trigpoly} to include those with frequencies $\mbf k$ on a shifted lattice $\mathbb{Z}^2 + \bs\sigma$ for some fixed $\bs\sigma = (\sigma_1,\sigma_2)$ with $\sigma_i\in \{0,\frac{1}{2}\}$. 
This is to ensure the results below extend to trigonometric polynomials having an even number of coefficients in one or both dimensions.

We define the \emph{degree} of $\mu$, denoted as $deg(\mu) = (K,L)$, to be the linear dimensions of the smallest rectangle $R$ whose closure contains the frequency support set, and we say $\mu$ is \emph{centered} if the frequency support rectangle $R$ is symmetric about origin $(k,l) = (0,0)$. For example, 
\begin{align*}
\mu_1(x,y) & = 4\sin(2\pi x)\cos(4\pi y)\\ 
& = je^{j2\pi(-x-2y)}+je^{j2\pi(x-2y)}-je^{j2\pi (-x+2y)}-je^{j2\pi(x+2y)}
\end{align*}
has frequency support contained in the rectangle $[-1,1]\times[-2,2]$, hence is centered with degree $(2,4)$. As another example,
\[
\mu_2(x,y) = 2 e^{-j2\pi x} + 1+ e^{j2\pi x} - e^{j2\pi(x+y)}
\]
has frequency support contained in $[-1,1]\times[0,1]$, hence degree $(2,1)$, but is \emph{not} centered. However, we can make $\mu_2$ centered by multiplying by a phase factor $\widetilde{\mu}_2(x,y) = e^{-j\pi y}\mu_2(x,y)$, which shifts the frequency support to $[-1,1]\times[-\frac{1}{2},\frac{1}{2}]$. Note $\mu_2$ and $\widetilde{\mu}_2$ have the same degree and the same zero set, since these properties are invariant under multiplication by a phase factor.

For any trigonometric polynomial $\mu$ with $deg(\mu) = (K,L)$ we can associate a unique complex polynomial $\mathcal{P}[\mu]$ in $\mathbb{C}[z,w]$ by making the substitutions $e^{j2\pi x}\mapsto z, e^{j2\pi y} \mapsto w$, and multiplying by powers of $z$ and $w$ such that $\mathcal{P}[\mu]$ has degree $K$ as a polynomial in $z$, and degree $L$ as a polynomial in $w$. Similarly, we may define an inverse mapping acting on polynomials $q(z,w)\in \mathbb{C}[z,w]$ as $\mathcal{P}^{-1}[q](x,y) := e^{-j2\pi Sx}e^{-j2\pi Ty}q(e^{j2\pi x},e^{j2\pi y})$, where $S,T$ are chosen so that the resulting trigonometric polynomial is centered. For example, using the same $\mu_1$ and $\mu_2$ as above, we have
\begin{align*}
\mathcal{P}[\mu_1](z,w) & = zw^2(jz^{-1}w^{-2} + jzw^{-2}-jz^{-1}w^{2}-zw^{2}) = j(1 + z^2 - w^4-z^2w^4),\\
\mathcal{P}[\mu_2](z,w) & = z(2z^{-1} + 1 + z - zw) = 2 + z + z^2 -z^2w
\end{align*} 
and $\mathcal{P}^{-1}[\mathcal{P}[\mu_1]](x,y) = \mu_1(x,y)$ whereas $\mathcal{P}^{-1}[\mathcal{P}[\mu_2]](x,y) = e^{-j\pi x}\mu_2(x,y)$. 

We also define $E(x,y) = (e^{j2\pi x},e^{j2\pi y})$ for all $(x,y)\in [0,1]^2$, which is a continuous bijective mapping of $[0,1)^2$ onto the complex unit torus $\mathbb{CT}^2 = \{(z,w)\in\mathbb{C}^2: |z|=|w|=1\}$. Observe that $\mu(x,y) = 0$ if and only if $\mathcal{P}[\mu](z,w) = 0$ for $(z,w) = E(x,y)$, so that $\mu = 0$ on $C \subseteq [0,1)^2$ if and only if $\mathcal{P}[\mu] = 0$ on $E(C)$. This shows we may study the algebraic properties of trigonometric polynomials and their zero sets by studying their corresponding complex polynomials under the mapping $\mathcal{P}$. Accordingly, for trigonometric polynomials $\mu$ and $\nu$ we will say $\mu$ \emph{divides} $\nu$, denoted $\mu~|~\nu$, when $\mathcal{P}[\mu]$ divides $\mathcal{P}[\nu]$ as polynomials in $\mathbb{C}[z,w]$, or equivalently, if $\nu = \mu\,\gamma$ for some trigonometric polynomial $\gamma$; $\mu$ and $\nu$ are said to have a \emph{common factor} if there is a $\gamma$ with positive degree such that $\gamma~|~\mu$ and $\gamma~|~\nu$; and we say $\mu$ is \emph{irreducible} when $\mathcal{P}[\mu]$ is irreducible as a polynomial in $\mathbb{C}[z,w]$. 

We will need the following well-known result regarding the maximum number of intersections of algebraic curves (see e.g. \cite{shafarevich1994basic}):
\begin{theorem}[B\'ezout's Theorem]
Let $p_1$ and $p_2$ be non-constant polynomials in $\mathbb{C}[z,w]$ with $deg(p_1) = d_1$ and $deg(p_2) = d_2$. If $p_1$ and $p_2$ have no common factor then the system of equations $p_1(z,w) = p_2(z,w) = 0$ has at most $d_1\,d_2$ solutions.
\label{lem1}
\end{theorem}
One important consequence of B\'ezout's theorem is that if $p(z,w)$ is irreducible and $q(z,w)$ is any other polynomial such that $p(z,w)=q(z,w)=0$ has infinitely many solutions, then $p$ must divide $q$.

We can readily translate these results to trigonometric polynomials $\mu$ and $\nu$ by mapping to the polynomials $p = \mathcal{P}[\mu]$ and $q = \mathcal{P}[\nu]$, and mapping back again. If $deg(\mu) = (K,L)$, then $p = \mathcal{P}[\mu]$ has total degree at most $K+L$ as a polynomial in $\mathbb{C}[z,w]$, and so we have the following corollary that we will also call ``B\'ezout's theorem'':
\begin{corollary}
Let $\mu_1$ and $\mu_2$ be trigonometric polynomials with $deg(\mu_1) = (K_1,L_1)$ and $deg(\mu_2) = (K_2,L_2)$. If $\mu_1$ and $\mu_2$ have no common factor then the system of equations $\mu_1(x,y) = \mu_2(x,y) = 0$ has at most $(K_1+L_1)(K_2+L_2)$ solutions in $[0,1)^2$.
\label{bezout_trig}
\end{corollary}
Consequently, if the zero sets of trigonometric polynomials $\mu$ and $\nu$ have infinite intersection, and $\mu$ is irreducible, this implies $\mu~|~\nu$.

\subsection{Existence of a minimal polynomial (Proof of Proposition \ref{prop:minpoly})}
Here we prove a slightly stronger version of Proposition \ref{prop:minpoly}, which shows the minimal polynomial satisfies an additional unique divisibility property.
 Recall that we say the zero set of a trigonometric polynomial $C = \{\mu = 0\}$ is a \emph{trigonometric curve} if $C$ infinite and has no isolated points, i.e.\ if every connected component of $C$ is one-dimensional. 
\begin{proposition}
For every trigonometric curve $C$ there is unique (up to scaling) real-valued trigonometric polynomial $\mu_0$ such that $C = \{\mu_0 = 0\}$ and for any other trigonometric polynomial $\mu$ with $\mu = 0$ on $C$ we have $deg(\mu_0) \leq deg (\mu)$ componentwise, and $\mu_0~|~\mu$.
\label{prop:minpolystrong}
\end{proposition}
\begin{proof}
Let $\mu$ be any trigonometric polynomial with $C=\{\mu = 0\}$. Set $p = \mathcal{P}[\mu] \in \C[z,w]$. Then $p$ has a unique factorization (up to scaling) into irreducible factors: $p = p_1^{n_1} \cdots p_r^{n_r}$. Create a polynomial $p_0 = p_{i_1}\cdots p_{i_k}$, $k \leq r$, omitting any factors $p_i$ whose zero sets do not intersect $\mathbb{CT}^2$, or intersect at only finitely many points. Define $\mu_0 = \mathcal{P}^{-1}[p_0]$, which by construction satisfies $\mu_0 = 0$ on $C/A$ where $A$ is possibly some finite set. In fact, it must be the case that $\mu_0 = 0$ on all of $C$, since $\mu_0$ is continuous and $C$ contains no isolated points by assumption. We will show $\mu_0$ is the desired minimal polynomial.

Note the algebraic curve $X=\{p_0 = 0\}$ is the union of irreducible curves $X_j=\{p_{i_j} = 0\} \subseteq \mathbb{C}^2$. Define $C_j = E^{-1}(X_j \cap \mathbb{CT}^2)$, then likewise we have a decomposition of $C$ as the union of the curves $C_j$. If $\nu$ is any other trigonometric polynomial with $C=\{\nu = 0\}$, then $\nu=0$ on each $C_j$, and for $q=P[\nu]$ we have $q=0$ on each infinite set $E(C_j) = X_j \cap \mathbb{CT}^2$. Therefore by B\'ezout's theorem, $q = \mathcal{P}[\nu]$ must factor as $q = p_{i_1}^{m_1}\cdots p_{i_k}^{m_k}\,r$ for some $r\in \mathbb{C}[z,w]$, hence $deg(p_0) \leq deg(q)$ and $p_0~|~q$, which implies $deg(\mu_0) \leq deg(\nu)$ and $\mu_0~|~\nu$. 

We now show $\mu_0$ is real-valued under an appropriate choice of scaling. Without loss of generality assume $Re[\mu_0] \neq 0$. Since $Re[\mu_0] = \frac{1}{2}(\mu_0 + \mu_0^*)$ and $\mu_0$ is centered, we have $deg(Re[\mu_0]) \leq deg(\mu_0)$. Also, since $\{\mu_0 = 0\} \subseteq \{Re[\mu_0] = 0\}$ we have $\mu_0 | Re[\mu_0]$ and so $Re[\mu_0] = c\cdot \mu_0$ for some $c\in\mathbb{C}$ by the uniqueness of $\mu_0$. Therefore $\frac{1}{c}\mu_0$ is real-valued.

Finally, to show uniqueness suppose $\nu_0$ is another centered trigonometric polynomial satisfying $C=\{\nu_0 = 0\}$ and $deg(\nu_0)= deg(\mu_0)$. Then $\mu_0~|~\nu_0$, and so by degree constraints, $\nu_0 = c\cdot \mu_0$ for some $c\in\mathbb{C}$. 
\end{proof}

\subsection{Additional properties of minimal polynomials}
Our uniqueness theorems also make use of the following result, which shows the zero set of the minimal polynomial and the zero set of its gradient have at most a finite intersection:
\begin{proposition}
\label{prop:musquared}
Let $C$ be a trigonometric curve with minimal polynomial $\mu$. Then $\nabla\mu = 0$ for at most finitely many points on $C$.
\end{proposition}
\begin{proof}
Note the partial derivatives $\partial_x \mu$ and $\partial_y \mu$ are trigonometric polynomials with $deg(\partial_x \mu)\leq deg(\mu)$ and ${deg(\partial_y \mu)\leq deg(\mu)}$. First consider the case where $\mu$ is irreducible. If $\{\nabla \mu = 0\}$ and $C$ have an infinite intersection, then by B\'ezout's theorem, $\mu$ must divide $\partial_x \mu$ and $\partial_y \mu$, and so by degree considerations, $\partial_x \mu = A\, \mu$ and $\partial_y \mu = B\, \mu$ for some non-zero constants $A,B$. By comparing coefficients, we see this is only possible if $\mu$ is the zero polynomial.

The general case follows by induction on the number of irreducible factors in $\mu$: Suppose $\mu=\mu_1\cdots\mu_{n-1}\mu_n$ where each $\mu_i$ is a distinct irreducible factor (that $\mu$ must have this form follows from the proof of Prop.\ \ref{prop:minpoly}). Let $\nu = \mu_1\cdots\mu_{n-1}$, then we may write $C=C_1\cup C_2$ where $C_1=\{\nu=0\}$ and $C_2 = \{\mu_n=0\}$. By the product rule $\nabla \mu = \mu_n \nabla \nu + \nu \nabla \mu_n$. Therefore, $\nabla\mu = 0$ on $C_1$ if only if $\mu_n = 0$ or $\nabla\nu = 0$ on $C_1$. However, $\{\mu_n=0\}\cap C_1$ is finite by B\'ezout's theorem, and $\{\nabla \nu =0\}\cap C_1$ is finite by inductive hypothesis, hence $\{\nabla\mu = 0\} \cap C_1$ is finite. A similar argument shows $\{\nabla\mu = 0\} \cap C_2$ is finite, which proves the induction.
\end{proof}

The following proposition shows that minimal polynomials must always change sign across an infinite segment of its zero set. 

\begin{proposition}
Let $C$ be a trigonometric curve with minimal polynomial $\mu$. If $U_1$ and $U_2$ are distinct connected components of $[0,1]^2/C$ whose boundaries share an infinite segment of $C$, then the sign of $\mu$ must be different on $U_1$ and $U_2$.
\label{prop:coloring}
\end{proposition}
\begin{proof}
Let $U_1$ and $U_2$ be any two connected components of $[0,1]^2/C$ whose boundaries share an infinite segment $S\subseteq C$, and suppose, without loss of generality, $\mu$ is positive on both $U_1$ and $U_2$. Then for all points in $S$, $\mu$ has a local minimum, and since $\mu$ is smooth this implies $\nabla\mu = 0$ on all of $S$, which violates Proposition \ref{prop:musquared}. Therefore $\mu$ must have different signs on $U_1$ and $U_2$.
\end{proof}

Note that this result puts certain topological constraints on the geometry of trigonometric curves. For instance, if $C$ is any trigonometric curve, and $\mbf r_0 \in C$, then in any sufficiently small neighborhood of $r_0$, the set $C/\{\mbf r_0\}$ must consist of an \emph{even} number of branches, since otherwise the minimal polynomial would violate the parity rule in Prop.\ \ref{prop:coloring}. This implies, for example, that within the region $[0,1]^2$ trigonometric curves must always be closed. It also rules out such features such as triple junctions, where locally three curve segments terminate at a common intersection point.

\subsection{Bound on Connected Components (Proof of Proposition \ref{prop:connectedcomp})}
\label{sec:appendix_cc}
To bound the number of connected components in the complement of a trigonometric curve, we make use of Bernstein's theorem, also known as the BKK bound, which is a well-known refinement of B\'ezout's theorem (see, e.g., \cite{sturmfels1998polynomial,fulton1993introduction}). Bernstein's theorem relates the maximum number of isolated zeros of a Laurent polynomial system to the geometry of the coefficient supports. Here we use the notation $conv(\Lambda)$ to denote the convex hull of an index set $\Lambda \subset \mathbb{Z}^2$ treated as a subset of $\mathbb{R}^2$. For any two sets $P_1,P_2 \subset \mathbb{R}^2$ we also use $\mathcal{M}(P_1,P_2)$ to denote their \emph{mixed area}:
\[
\mathcal{M}(P_1,P_2):= area(P_1 + P_2) - area(P_1) - area(P_2)
\]
where $area(\cdot)$ is the usual Euclidean area and $P_1+P_2$ is the Minkowski sum, $\{p_1+p_2: p_1\in P_1, p_2 \in P_2\}$.
\begin{theorem}[Bernstein's Theorem]
Let $p_1(z,w)$ and $p_2(z,w)$ be Laurent polynomials with coefficient supports $\Lambda_1$ and $\Lambda_2$, and let $P_1 = \text{conv}(\Lambda_1)$ and $P_2 = \text{conv}(\Lambda_1)$. The number of isolated solutions of system $p_1(z,w) = p_2(z,w)=0$ in $(\mathbb{C}/\{0\})^2$ is at most $\mathcal{M}(P_1,P_2)$.
\end{theorem}
Since a trigonometric polynomial is simply the restriction of a Laurent polynomial to the unit complex torus, we can use this result to bound the number of isolated zeros of a system of trigonometric polynomials. In particular, using that fact that the mixed area satisfies $\mathcal{M}(P_1,P_2) \leq \mathcal{M}(R_1,R_2)$ if $P_i \subseteq R_i$, $i=1,2$ (\cite{fulton1993introduction}, pg.\ 117), we have the following direct corollary:
\begin{corollary}
Let $\mu_1$ and $\mu_2$ be trigonometric polynomials with $deg(\mu_1)\leq(K_1,L_1)$ and $deg(\mu_2)\leq(K_2,L_2)$. The number of isolated solutions of $\mu_1(x,y) = \mu_2(x,y) = 0$ in $[0,1)^2$ is at most $K_1L_2 + K_2L_1$.
\label{bkk_trig}
\end{corollary}
Now we prove Proposition \ref{prop:connectedcomp}:

Let $C$ be a trigonometric curve of degree $(K,L)$, i.e.\ the minimal polynomial $\mu$ for $C$ has degree $(K,L)$, and let $U_1,...,U_n$ be the connected components of $[0,1]^2/C$. For each $i=1,...,n$, $\mu$ is continuous with $\mu\neq 0$ on $U_i$, $\mu = 0$ on $\partial U$, and $\overline{U} = U\cup\partial U$ is compact, therefore $\mu$ must attain its maximum or minimum at some point $\mbf r_i$ in the interior $U$, where $\nabla\mu(\mbf r_i) = 0$. Hence, the number of components $n$ is bounded by the cardinality of $\{\nabla \mu = 0\} = \{\partial_x\mu = 0\} \cap \{\partial_y\mu = 0\}$. When $\{\nabla \mu = 0\}$ is finite, the result $n\leq 2KL$ follows from Corollary \ref{bkk_trig} since $deg(\partial_x\mu),deg(\partial_y\mu) \leq (K,L)$. 

If $\{\nabla \mu = 0\}$ is not finite, i.e.\ $\{\partial_x\mu = 0\}$ and $\{\partial_y\mu=0\}$ share a common component, we may obtain the same bound by a standard perturbation argument, which we summarize here (see, e.g., the appendix \cite{solymosi2012incidence} for further details). Let $\mbf v \in \mathbb{R}^2$ be a sufficiently small vector. By continuity, the function $\widetilde{\mu}(\mbf r) = \mu(\mbf r) - \langle \mbf r, \mbf v\rangle$ also attains a maximum or minimum in each $U_i$, and $\nabla \widetilde{\mu} = \nabla \mu - \mbf v$ vanishes at these points, which shows $n$ is also bounded by the cardinality of $\{\nabla\mu = \mbf v\}$. Finally, note that since $\nabla\mu$ is a smooth map whose domain and range have the same dimension, for a generic vector $\mbf v$ the set $\{\nabla\mu = \mbf v\}$ is finite by the regular level set theorem (Cor.\ 8.10, \cite{lee2003smooth}), and the result again follows by Corollary \ref{bkk_trig}.
\hfill\ensuremath{\square}
\vspace{1em}
\\
We note that the bound in Proposition \ref{prop:connectedcomp} appears to be an overestimate. Empirically, we find that the maximum number of connected components of the complement to a trigonometric curve of degree $(K,L)$ never exceeds $KL$ (in the topology of the torus $\mathbb{T}^2 = (\mathbb{R}/\mathbb{Z})^2$). Moreover, this bound is only obtained in artificial cases, such as when the minimal polynomial is $\mu(x,y) = \sin(\pi K x)\sin(\pi L y)$, whose zero set is the union of $K$ vertical lines and $L$ horizontal lines.
\section{Proofs of Theorems}
\label{sec:appendix_proofs}
\subsection{Proof of Theorem \ref{thm:unique1}}
Suppose $f=1_U$, with $U$ simply connected and ${\partial U = \{\mu = 0\}}$. where $\mu$ is the minimal polynomial having coefficients $(c[\mbf k]: \mbf k\in {\Lambda})$. Let $(d[\mbf k] : \mbf k \in {\Lambda})$ be any non-trivial set of coefficients satisfying
\begin{align}
	\sum_{\mbf k \in {\Lambda}} d[\mbf k] \widehat{\nabla f}[\bs\ell - \mbf k] & = 0,~~\text{ for all } \bs\ell \in 2\Lambda.
	\label{eq:linsystem}
\end{align}
We will show the coefficients $d[\mbf k]$ must be a scalar multiple of $c[\mbf k]$. 

Let $\eta$ denote the trigonometric polynomial having coefficients $(d[\mbf k] : \mbf k \in {\Lambda})$. Observe that the sums in \eqref{eq:linsystem} represent the Fourier coefficients of $\eta \nabla f$ evaluated at $\bs \ell \in 2\Lambda$. Therefore, \eqref{eq:linsystem} can be expressed in spatial domain as
\[
\mathcal{P}_{2\Lambda}(\eta \nabla f) = 0
\]
where $\mathcal{P}_{2\Lambda}$ is Fourier projection onto the index set $2\Lambda$. If we let $\bs \varphi$ be any smooth test field, then by duality
\[
\langle \mathcal{P}_{2\Lambda}(\eta \nabla f), \bs \varphi \rangle = \langle \eta \nabla f, \mathcal{P}_{2\Lambda}\bs \varphi \rangle =0,
\]
or, equivalently,
\[
\langle \eta \nabla f, \bs \psi \rangle = 0, \text{ for all } \bs\psi \in B_{2\Lambda}^2,
\]
where $B_{2\Lambda}^2$ denotes the space all smooth vector fields $\bs \varphi = (\varphi_1,\varphi_2)$ whose components $\varphi_1,\varphi_2$ are bandlimited to $2\Lambda$. Since $f = 1_U$, using the distributional characterization of $\nabla f$ in \eqref{eq:duality} we have
\begin{equation}
\oint_{\partial U} \eta\,(\bs\psi \cdot \mbf n)\, ds  = 0, \text{ for all } \bs\psi \in B_{2\Lambda}^2.
\label{eq:forallphi}
\end{equation}
Because $\partial U$ is a level set of $\mu$, the outward normal $\mbf n$ is given pointwise by $\mbf n(\mbf r) = \pm\nabla \mu(\mbf r)/|\nabla \mu(\mbf r)|$ for all $\mbf r \in \partial U$, provided $\nabla \mu(\mbf r) \neq 0$. However, by Prop. \ref{prop:musquared}, the set of points for which $\nabla \mu = 0$ on $\partial U$ is at most finite, hence of measure zero, so the formula $\mbf n(\mbf r) = \pm\nabla \mu(\mbf r)/|\nabla \mu(\mbf r)|$ holds almost everywhere on $\partial U$. Now, without loss of generality, assume $\mu < 0$ on $U$. Since $\mu$ is a minimal polynomial, by Prop.\ \ref{prop:coloring} it must change sign across the any infinite segment of $\partial U$, hence $\mu > 0$ on $U^C$, so that $\mbf n(\mbf r) = \nabla \mu(\mbf r)/|\nabla \mu(\mbf r)|$ for all $\mbf r \in \partial U$. Setting $\bs \psi = \eta^* \nabla \mu \in B^2_{2\Lambda}$, from \eqref{eq:forallphi} we have
\[
\oint_{\partial U} |\eta|^2\,|\nabla \mu|\, ds  = 0,
\]
which is possible if and only if $\eta = 0$ on $\partial U$ since $|\nabla \mu(\mbf r)|=0$ for at most finitely many $\mbf r \in \partial U$. Therefore $\eta$ must be a scalar multiple of $\mu$ by the uniqueness of minimal polynomials.

\subsection{Proof of Theorems \ref{thm:unique2} and \ref{thm:unique_3}}
Assume $(d[\mbf k] : \mbf k \in \Lambda)$ is another solution of the system, and let $\eta$ denote the associated trigonometric polynomial. Following the same steps as in the proof of Theorem \ref{thm:unique1}, we may show 
\[
\langle \eta\,\nabla f, \bs \psi \rangle = 0, \text{ for all } \bs\psi \in B_{2\Lambda}^2,
\]
Since $f = \sum_{i=1}^n a_i 1_{U_i}$, from \eqref{eq:duality} this gives
\begin{equation}
\sum_{i=1}^n a_i \oint_{\partial U_i} \eta\,(\bs\psi \cdot \mbf n)\, ds  = 0, \text{ for all } \bs\psi \in B_{2\Lambda}^2.
\label{eq:forallphi2}
\end{equation}
Recall we assume $\partial U_i = \{\mu_i = 0\}$ for all $i=1,...,n$, with $\mu_i$ minimal, and $\mu = \mu_1\cdots\mu_n$. Note that $\cup_i \partial U_i = \{\mu=0\}$, and since we assumed the boundaries to be pairwise disjoint, $\mu$ is also the minimal polynomial. Define $\widetilde{\mu}_i = \mu/\mu_i$. Then by the product rule
\begin{equation}
\nabla \mu = \sum_{i=1}^n \widetilde{\mu}_i \nabla \mu_i. 
\label{eq:gradmu}
\end{equation}
Define the vector field
\begin{equation}
\bs\psi_j = \eta^*\,\widetilde{\mu}_j \nabla \mu_j~~\text{for all}~~j=1,...,n.
\label{eq:choice}
\end{equation}
Note the components of $\bs \psi_j$ are trigonometric polynomials in $B_{2\Lambda}$, hence $\bs \psi_j \in B^2_{2\Lambda}$ for all $j=1,\ldots,n$. Since $\widetilde{\mu}_j = 0$ on $\partial U_i$ for all $i\neq j$, substituting $\bs \psi_j$ into \eqref{eq:forallphi2} gives
\[
0 = \sum_{i=1}^n a_i \oint_{\partial U_i} |\eta|^2\, \widetilde{\mu}_j (\nabla \mu_j \cdot \mbf n) \, ds = 
a_j \oint_{\partial U_j}|\eta|^2\,\widetilde{\mu}_j (\nabla \mu_j \cdot \mbf n) \, ds = a_j \oint_{\partial U_j} |\eta|^2\, (\nabla \mu \cdot \mbf n) \, ds.
\]
where in the last step we used that $\widetilde{\mu}_j \nabla \mu_j = \nabla \mu$ on $\partial U_j$ by \eqref{eq:gradmu}.
Since we also assume each $U_i$ is connected, and the curves $\partial U_i$ do not intersect, we know $\mu$ has fixed sign on $U_i$. Therefore the outward normal on each $\partial U_j$ is given by $\mbf n = \pm\nabla \mu/|\nabla\mu|$, where the choice of sign is fixed for each $i$. Hence we have
\[
a_j \oint_{\partial U_j}|\eta|^2\,|\nabla \mu| = 0,\ \text{for all}\ j = 1,...,n,
\]
which implies $\eta$ vanishes on each $\partial U_j$. Since $\mu$ is the minimal polynomial, this is only possible if $\eta$ is a scalar multiple of $\mu$, which proves the claim in Theorem \ref{thm:unique2}.

The proof of Theorem \ref{thm:unique_3} is identical to the above, except we exchange the index set $2\Lambda$ with $\Lambda + \Lambda'$. This allows us to make the same choice of $\bs \psi_j$ as in \eqref{eq:choice}, which again implies $\eta$ vanishes on each $\partial U_j$. Hence by Prop.\ \ref{prop:minpoly} we have $\eta = \mu\, \gamma$, where $\gamma \in B_{\Lambda':\Lambda}$.

\subsection{Proof of Proposition \ref{thm:piecewisecst}}
Suppose $g\in L^1([0,1]^2)$ satisfies $\mu \nabla g = 0$ and let $U$ be any connected component of $\{\mu = 0\}^C$. Let $\bs\varphi$ be any arbitrary smooth test field with support $K$ strictly contained within $U$. Since $\mu = 0$ only on $\partial U$, and $\mu$ is smooth, the field $\bs \psi = \bs \varphi/\mu$ is smooth, and we have 
\[
\langle \mu \nabla g, \bs \psi \rangle = \langle \nabla g, \bs \varphi \rangle = 0.
\]
which shows the distribution $\nabla g$ restricted to $U$ is identically zero. We now show this implies $g$ restricted to $U$ is a constant function almost everywhere.

Let $\phi_\epsilon$ be any smooth approximation of the identity, i.e.\ for all $\epsilon > 0$ define $\phi_\epsilon(\mbf r) := \epsilon^2 \phi( \epsilon \mbf r)$ where $\phi$ is any smooth function satisfying $\phi\geq 0$, $\int \phi(\mbf r) d\mbf r= 1$, $supp(\phi) \subseteq \{|\mbf r| \leq 1\}$. Let $U'$ be any closed connected set inside $U$. Then there exists an $\epsilon'>0$ such that $dist(U',\overline{U}) < \epsilon'$. Fix $\epsilon < \epsilon'$. Since $\nabla g \ast \phi_\epsilon$ is a smooth function, $\nabla g|_{U} = 0$, and $supp(\phi_\epsilon)\subseteq U$ we have
\[
\nabla (g \ast \phi_\epsilon)(\mbf r) = (\nabla g \ast \phi_\epsilon)(\mbf r) = 0\quad\text{for all}\ \mbf r \in U'
\] 
which implies $g \ast \phi_\epsilon$ is constant on $U'$. By properties of approximations of identity, $g \ast \phi_\epsilon \rightarrow g$ pointwise almost everywhere as $\epsilon\rightarrow 0$, which shows $g$ is the limit of constant functions on $U'$, and therefore itself constant on $U'$. Since $U'\subseteq U$ was arbitrary, we may conclude $g$ restricted to $U$ is a constant function almost everywhere, proving the claim.

\subsection{Proof of Theorem \ref{thm:uniqueamplitudes}}
Suppose $g$ is another solution, and set $h = g-f$. We have
\begin{equation}
\label{eq:bv1_h} \mu \nabla h = 0,\quad\text{and}\quad \widehat{h}[\mbf k] = 0 \text{ for all }\mbf k \in {\Gamma \supseteq \Lambda}.
\end{equation}
The proof amounts to showing that the only feasible solution to \eqref{eq:bv1_h} is $h=0$.
By the first constraint and Prop.\ \ref{thm:piecewisecst}, we see that $h$ is almost everywhere a piecewise constant function: 
\begin{equation}
h = \sum_{i=1}^N b_i 1_{U_i},
\end{equation}
where the constants $b_i$ are unknown, and $U_i$ are the connected components of $\{\mu =0\}^C$. From $\widehat{h}[\mbf k] = 0$, $\mbf k \in \Gamma$, we also have $\widehat{\nabla h}[\mbf k] = 0$, $\mbf k \in \Gamma$, and so
\[
\widehat{\nabla h}[\mbf k] = \sum_i b_i \widehat{\nabla 1_{U_i}}[\mbf k] = \sum_i b_i \oint_{\partial U_i} e^{-j2\pi\mbf k \cdot \mbf r}\, \mbf n\,ds = 0,\text{ for all } \mbf k \in \Gamma.
\]
By linearity, this implies
\[
\sum_{i=1}^N b_i \oint_{\partial U_i} \bs\varphi \cdot \mbf n \, ds  = 0,
\]
for all vector fields $\bs \varphi \in B^2_{\Gamma}$. The proof now follows from the same argument in the proof to Theorem \ref{thm:unique2} (with $\eta = 1$): Using the same notation as in \eqref{eq:gradmu}, we may choose $\bs \varphi = \widetilde{\mu}_j \nabla \mu_j \in B^2_{\Lambda} \subseteq B^2_{\Gamma}$, which implies
\[
b_j \oint_{\partial U_j}|\nabla \mu| = 0,\ \text{for all}\ j = 1,\ldots,n,
\]
and so $b_i = 0$ for all $i=1,\ldots,N$, and hence $h=0$ proving the claim.

\subsection{Proof of Proposition \ref{prop:rank}}
Let $\mu$ be the minimal polynomial with coefficients $\mbf c = (c[\mbf k]: \mbf k\in \Lambda)$, and for every $\bs \ell\in \Lambda':\Lambda$ define the filter $\mbf c_{\bs \ell}$ supported in $\Lambda'$ by
\[
c_{\bs \ell}[\mbf k] = \begin{cases} 
c[\mbf k - \bs \ell] & \text{if } \mbf k-\bs\ell\in\Lambda\\
0 & \text{elsewhere},
\end{cases}
\]
that is, $\mbf c_{\bs \ell}$ is a shift of the minimal polynomial coefficients within the larger index set $\Lambda$. By the Fourier shift theorem, $\mbf c_{\bs \ell}$ are the Fourier coefficients of $e^{j2\pi\bs \ell\cdot\mbf r}\mu(\mbf r)$. Therefore, $\mbf c_{\bs \ell}$ is also an annihilating filter since the zero sets of $\mu(\mbf r)$ and $e^{j2\pi\bs \ell\cdot\mbf r}\mu(\mbf r)$ coincide. This implies $\mathcal{T}(\widehat{f})\mbf c_{\bs \ell} = 0$ for all $\bs \ell \in \Lambda':\Lambda$. Also, note that the vectors $\mbf c_{\bs \ell}$ are linearly independent because the location of their non-zero entries are distinct. Hence the nullspace of $\mathcal{T}(\widehat{f})$ has dimension at least $|\Lambda':\Lambda|$, which gives the desired rank bound.

Finally, supposing $f$ satisfies the conditions of Theorem \ref{thm:unique2}, and $\Gamma \supseteq 2\Lambda' + \Lambda$ so that $\Gamma\,{:}\,\Lambda' \supseteq \Lambda' + \Lambda$. Theorem \ref{thm:unique_3} shows that if $\mbf d$ is any null space vector, then $\mbf d$ are the Fourier coefficients of $\mu\, \gamma$, for some trigonometric polynomial $\gamma$ having coefficients in $\Lambda':\Lambda$, which shows that $\mbf d$ belongs to the span of the $\mbf c_{\bs \ell}$, proving the claim.

\subsection{Proof of Proposition \ref{prop:finitezeros}}
Let $\mbf r \in \{\mu = 0\}^C$. Then $\overline{\mu}(\mbf r) = 0$ if and only if $\gamma_i(\mbf r) = 0$ for all $i = 1,\ldots,R$, where $R =|\Lambda':\Lambda|$. 
We will show the set $P = \cap_{i=1}^R \{\gamma_i = 0\}$ is at most finite. Note that if $P$ is infinite, by B\'ezout's theorem, each $\gamma_i$ must have a common non-constant factor $\phi$, and so $\gamma_i = \rho_i\, \phi$, where each $\rho_i$ is a trigonometric polynomial with degree strictly smaller than $\gamma_i$. This means each $\rho_i \in B_{\Lambda'}$ has coefficients supported within some index set $\Lambda_1\subseteq \Lambda':\Lambda$ with $\Lambda_1\neq \Lambda':\Lambda$. Since the coefficient vectors $\mbf d_i$ are orthonormal, so are the $\mu_i$ as functions in $L^2([0,1]^2)$:
\[
\langle \mu_i , \mu_j \rangle = \int_{[0,1]^2} \rho_i(\mbf r) \overline{\rho_j(\mbf r)}\, |\phi(\mbf r) \mu(\mbf r)|^2\, d\mbf r = \delta_{i,j}.
\]
where $\delta_{i,j} = 0$ if $i\neq j$ and $\delta_{i,j} = 1$ if $i=j$. However, the above weighted integral also defines an inner product in which the $\{\rho_i\}_{i=1}^R$ are mutually orthogonal, and hence linearly independent. But this is impossible, since the $\rho_i$ are contained in a space of dimension less than $R$. Therefore the $\gamma_i$ have no common factors, which by B\`ezout's theorem implies the intersection of their zero sets is at most finite.

\subsection{Proof of Proposition \ref{prop:MUSIC}}
Note that $\mbf e_{\mbf{r}} = P_{\Lambda}\mathcal{F}\{\delta_{\mbf r}\}$, where $\delta_{r}$ is a Dirac delta centered at $\mbf r\in[0,1]^2$, and $\mathcal{F}$ is the (distributional) Fourier transform, and $P_{\Lambda}$ is projection onto the index set $\Lambda$. Since each column $\mbf d_i$ in $\mbf D$, is a filter supported in $\Lambda$, we trivially have $P_{\Lambda}(\mbf d_i) = \mbf d_i$. Therefore, by Parseval's theorem, 
\begin{align*}
\|\mbf D\mbf D^H \mbf e(\mbf r) \|^2 & = \sum_{i=1}^n |\langle \mbf e_{\mbf r}, \mbf d_i \rangle|^2 = \sum_{i=1}^n |\langle P_{\Lambda}\mathcal{F}\{\delta_{\mbf r}\},\mbf d_i \rangle|^2\\
& = \sum_{i=1}^n |\langle \mathcal{F}\{\delta_{\mbf r}\}, \mbf d_i \rangle|^2 = \sum_{i=1}^n |\langle \delta_{\mbf r},\mu_i\rangle|^2 = \sum_{i=1}^n |\mu_i(\mbf r)|^2.
\end{align*}

\section*{Acknowledgments}
The authors would like to thank Theirry Blu, Dustin Mixon, and Cynthia Vinzant for fruitful discussions about this project. Parts of this work were presented at the IEEE International Symposium on Biomedical Imaging (ISBI `15) held April 16–-19, 2015 in Brooklyn, New York, and at the International Conference on Sampling Theory and its Applications (SampTA `15) held May 25–-29, 2015 in Washington D.C.



\bibliographystyle{abbrv}
\bibliographystyle{siam}
\bibliography{root}

\begin{thebibliography}{10}

\bibitem{baraniuk2007compressive}
R.~G. Baraniuk.
\newblock Compressive sensing.
\newblock {\em Signal Processing Magazine, IEEE}, 24(4):118--121, July 2007.

\bibitem{blu2008sparse}
T.~Blu, P.-L. Dragotti, M.~Vetterli, P.~Marziliano, and L.~Coulot.
\newblock Sparse sampling of signal innovations.
\newblock {\em Signal Processing Magazine, IEEE}, 25(2):31--40, 2008.

\bibitem{burger2001level}
M.~Burger.
\newblock A level set method for inverse problems.
\newblock {\em Inverse problems}, 17(5):1327, 2001.

\bibitem{burger2005survey}
M.~Burger and S.~J. Osher.
\newblock A survey on level set methods for inverse problems and optimal
  design.
\newblock {\em European Journal of Applied Mathematics}, 16(02):263--301, 2005.

\bibitem{cadzow1988signal}
J.~A. Cadzow.
\newblock Signal enhancement-a composite property mapping algorithm.
\newblock {\em Acoustics, Speech and Signal Processing, IEEE Transactions on},
  36(1):49--62, 1988.

\bibitem{candes2013super}
E.~J. Cand{\`e}s and C.~Fernandez-Granda.
\newblock Super-resolution from noisy data.
\newblock {\em Journal of Fourier Analysis and Applications}, 19(6):1229--1254,
  2013.

\bibitem{candes2014towards}
E.~J. Cand{\`e}s and C.~Fernandez-Granda.
\newblock Towards a mathematical theory of super-resolution.
\newblock {\em Communications on Pure and Applied Mathematics}, 67(6):906--956,
  2014.

\bibitem{chan2001level}
T.~Chan and L.~Vese.
\newblock A level set algorithm for minimizing the {M}umford-{S}hah functional
  in image processing.
\newblock In {\em Variational and Level Set Methods in Computer Vision, 2001.
  Proceedings. IEEE Workshop on}, pages 161--168, 2001.

\bibitem{chartrand2009fast}
R.~Chartrand.
\newblock Fast algorithms for nonconvex compressive sensing: {MRI}
  reconstruction from very few data.
\newblock In {\em International Symposium on Biomedical Imaging: ISBI 2009},
  pages 262--265. IEEE, 2009.

\bibitem{chartrand2011frequency}
R.~Chartrand, E.~Y. Sidky, and X.~Pan.
\newblock Frequency extrapolation by nonconvex compressive sensing.
\newblock In {\em International Symposium on Biomedical Imaging: ISBI 2011},
  pages 1056--1060. IEEE, 2011.

\bibitem{chen20122d}
C.~Chen, P.~Marziliano, and A.~C. Kot.
\newblock {2D} finite rate of innovation reconstruction method for step edge
  and polygon signals in the presence of noise.
\newblock {\em Signal Processing, IEEE Transactions on}, 60(6):2851--2859,
  2012.

\bibitem{chen2014robust}
Y.~Chen and Y.~Chi.
\newblock Robust spectral compressed sensing via structured matrix completion.
\newblock {\em Information Theory, IEEE Transactions on}, 60(10):6576--6601,
  2014.

\bibitem{cheng2003review}
Q.~Cheng and Y.~Hua.
\newblock A review of parametric high-resolution methods.
\newblock In Y.~Hua, A.~Gershman, and Q.~Cheng, editors, {\em High-resolution
  and robust signal processing}. CRC Press, 2003.

\bibitem{condat2015cadzow}
L.~Condat and A.~Hirabayashi.
\newblock Cadzow denoising upgraded: A new projection method for the recovery
  of dirac pulses from noisy linear measurements.
\newblock {\em Sampling Theory in Signal and Image Processing}, 14(1):p--17,
  2015.

\bibitem{do2005contourlet}
M.~N. Do and M.~Vetterli.
\newblock The contourlet transform: an efficient directional multiresolution
  image representation.
\newblock {\em Image Processing, IEEE Transactions on}, 14(12):2091--2106,
  2005.

\bibitem{donoho2006compressed}
D.~L. Donoho.
\newblock Compressed sensing.
\newblock {\em Information Theory, IEEE Transactions on}, 52(4):1289--1306,
  2006.

\bibitem{dorn2006level}
O.~Dorn and D.~Lesselier.
\newblock Level set methods for inverse scattering.
\newblock {\em Inverse Problems}, 22(4):R67, 2006.

\bibitem{dragotti2007sampling}
P.~L. Dragotti, M.~Vetterli, and T.~Blu.
\newblock Sampling moments and reconstructing signals of finite rate of
  innovation: Shannon meets strang--fix.
\newblock {\em Signal Processing, IEEE Transactions on}, 55(5):1741--1757,
  2007.

\bibitem{edwards2012fourier}
R.~E. Edwards.
\newblock {\em Fourier series: a modern introduction}, volume~2.
\newblock Springer Science \& Business Media, 2012.

\bibitem{eslami2010robust}
R.~Eslami and M.~Jacob.
\newblock Robust reconstruction of {MRSI} data using a sparse spectral model
  and high resolution {MRI} priors.
\newblock {\em Medical Imaging, IEEE Transactions on}, 29(6):1297--1309, 2010.

\bibitem{fessler2010model}
J.~A. Fessler.
\newblock Model-based image reconstruction for {MRI}.
\newblock {\em Signal Processing Magazine, IEEE}, 27(4):81--89, 2010.

\bibitem{fulton1993introduction}
W.~Fulton.
\newblock {\em Introduction to toric varieties}.
\newblock Number 131. Princeton University Press, 1993.

\bibitem{gong2015promise}
E.~Gong, F.~Huang, K.~Ying, W.~Wu, S.~Wang, and C.~Yuan.
\newblock {PROMISE}: Parallel-imaging and compressed-sensing reconstruction of
  multicontrast imaging using sharable information.
\newblock {\em Magnetic Resonance in Medicine}, 73(2):523--535, 2015.

\bibitem{guerquin2012realistic}
M.~Guerquin-Kern, L.~Lejeune, K.~P. Pruessmann, and M.~Unser.
\newblock Realistic analytical phantoms for parallel magnetic resonance
  imaging.
\newblock {\em Medical Imaging, IEEE Transactions on}, 31(3):626--636, 2012.

\bibitem{haacke1989super}
E.~Haacke, Z.-P. Liang, and S.~Izen.
\newblock Superresolution reconstruction through object modeling and parameter
  estimation.
\newblock {\em Acoustics, Speech and Signal Processing, IEEE Transactions on},
  37(4):592--595, April 1989.

\bibitem{haacke1990image}
E.~M. Haacke, Z.-P. Liang, and F.~E. Boada.
\newblock Image reconstruction using {POCS}, model constraints, and linear
  prediction theory for the removal of phase, motion, and gibbs artifacts in
  magnetic resonance and ultrasound imaging.
\newblock {\em Optical Engineering}, 29(5):555--566, 1990.

\bibitem{haacke1989constrained}
E.~M. Haacke, Z.-P. Liang, and S.~H. Izen.
\newblock Constrained reconstruction: A superresolution, optimal
  signal-to-noise alternative to the fourier transform in magnetic resonance
  imaging.
\newblock {\em Medical Physics}, 16(3):388--397, 1989.

\bibitem{haldar2014low}
J.~P. Haldar.
\newblock Low-rank modeling of local {k}-space neighborhoods ({LORAKS}) for
  constrained {MRI}.
\newblock {\em Transactions on Medical Imaging}, 33(3):668--681, 2014.

\bibitem{haldar2008anatomically}
J.~P. Haldar, D.~Hernando, S.-K. Song, and Z.-P. Liang.
\newblock Anatomically constrained reconstruction from noisy data.
\newblock {\em Magnetic Resonance in Medicine}, 59(4):810--818, 2008.

\bibitem{hu1988slim}
X.~Hu, D.~N. Levin, P.~C. Lauterbur, and T.~Spraggins.
\newblock {SLIM}: Spectral localization by imaging.
\newblock {\em Magnetic Resonance in Medicine}, 8(3):314--322, 1988.

\bibitem{hua1992estimating}
Y.~Hua.
\newblock Estimating two-dimensional frequencies by matrix enhancement and
  matrix pencil.
\newblock {\em Signal Processing, IEEE Transactions on}, 40(9):2267--2280,
  1992.

\bibitem{Jacob:2006p2905}
M.~Jacob, Y.~Bresler, V.~Toronov, X.~Zhang, and A.~Webb.
\newblock Level set algorithm for the reconstruction of functional activation
  in near-infrared spectroscopic imaging.
\newblock {\em Journal of Biomedical Optics}, Jan 2006.

\bibitem{jacob2007improved}
M.~Jacob, X.~Zhu, A.~Ebel, N.~Schuff, and Z.-P. Liang.
\newblock Improved model-based magnetic resonance spectroscopic imaging.
\newblock {\em Medical Imaging, IEEE Transactions on}, 26(10):1305--1318, 2007.

\bibitem{jin2015general}
K.~H. Jin, D.~Lee, and J.~C. Ye.
\newblock A general framework for compressed sensing and parallel {MRI} using
  annihilating filter based low-rank {H}ankel matrix.
\newblock {\em arXiv preprint arXiv:1504.00532}, 2015.

\bibitem{khalidov2007bslim}
I.~Khalidov, D.~Van De~Ville, M.~Jacob, F.~Lazeyras, and M.~Unser.
\newblock {BSLIM}: Spectral localization by imaging with explicit field
  inhomogeneity compensation.
\newblock {\em Medical Imaging, IEEE Transactions on}, 26(7):990--1000, 2007.

\bibitem{lee2003smooth}
J.~M. Lee.
\newblock {\em Smooth manifolds}.
\newblock Springer, 2003.

\bibitem{liang1989high}
Z.-P. Liang, E.~Haacke, and C.~Thomas.
\newblock High-resolution inversion of finite {F}ourier transform data through
  a localised polynomial approximation.
\newblock {\em Inverse Problems}, 5(5):831, 1989.

\bibitem{luo2012mri}
J.~Luo, Y.~Zhu, W.~Li, P.~Croisille, and I.~E. Magnin.
\newblock {MRI} reconstruction from {2D} truncated {k}-space.
\newblock {\em Journal of Magnetic Resonance Imaging}, 35(5):1196--1206, 2012.

\bibitem{mallat1989theory}
S.~G. Mallat.
\newblock A theory for multiresolution signal decomposition: the wavelet
  representation.
\newblock {\em Pattern Analysis and Machine Intelligence, IEEE Transactions
  on}, 11(7):674--693, 1989.

\bibitem{maravic2005sampling}
I.~Maravic and M.~Vetterli.
\newblock Sampling and reconstruction of signals with finite rate of innovation
  in the presence of noise.
\newblock {\em Signal Processing, IEEE Transactions on}, 53(8):2788--2805,
  2005.

\bibitem{milanfar1995reconstructing}
P.~Milanfar, G.~C. Verghese, W.~C. Karl, and A.~S. Willsky.
\newblock Reconstructing polygons from moments with connections to array
  processing.
\newblock {\em Signal Processing, IEEE Transactions on}, 43(2):432--443, 1995.

\bibitem{mumford1988boundary}
D.~Mumford and J.~Shah.
\newblock Boundary detection by minimizing functionals.
\newblock {\em Image understanding}, pages 19--43, 1988.

\bibitem{mumford1989optimal}
D.~Mumford and J.~Shah.
\newblock Optimal approximations by piecewise smooth functions and associated
  variational problems.
\newblock {\em Communications on pure and applied mathematics}, 42(5):577--685,
  1989.

\bibitem{sampta2015}
G.~Ongie and M.~Jacob.
\newblock Recovery of piecewise smooth images from few fourier samples.
\newblock {\em Sampling Theory and Applications (SampTA)}, 2015.

\bibitem{isbi2015}
G.~Ongie and M.~Jacob.
\newblock Super-resolution {MRI} using finite rate of innovation.
\newblock {\em IEEE International Symposium on Biomedical Imaging: ISBI 2015.},
  2015.

\bibitem{pal2014grid}
P.~Pal and P.~Vaidyanathan.
\newblock A grid-less approach to underdetermined direction of arrival
  estimation via low rank matrix denoising.
\newblock {\em Signal Processing Letters, IEEE}, 21(6):737--741, 2014.

\bibitem{pan2013sampling}
H.~Pan, T.~Blu, and P.~L. Dragotti.
\newblock Sampling curves with finite rate of innovation.
\newblock {\em Signal Processing, IEEE Transactions on}, 62(2), 2014.

\bibitem{schmidt1986multiple}
R.~O. Schmidt.
\newblock Multiple emitter location and signal parameter estimation.
\newblock {\em Antennas and Propagation, IEEE Transactions on}, 34(3):276--280,
  1986.

\bibitem{schweiger20103d}
M.~Schweiger, O.~Dorn, A.~Zacharopoulos, I.~Nissila, and S.~Arridge.
\newblock {3D} level set reconstruction of model and experimental data in
  diffuse optical tomography.
\newblock {\em Optics express}, 18(1):150--164, 2010.

\bibitem{shafarevich1994basic}
I.~R. Shafarevich.
\newblock {\em Basic algebraic geometry. 1. Varieties in projective space.}
\newblock Springer-Verlag, Berlin, 1994.

\bibitem{sake}
P.~J. Shin, P.~E. Larson, M.~A. Ohliger, M.~Elad, J.~M. Pauly, D.~B. Vigneron,
  and M.~Lustig.
\newblock Calibrationless parallel imaging reconstruction based on structured
  low-rank matrix completion.
\newblock {\em Magnetic Resonance in Medicine}, 2013.

\bibitem{shukla2007sampling}
P.~Shukla and P.~L. Dragotti.
\newblock Sampling schemes for multidimensional signals with finite rate of
  innovation.
\newblock {\em Signal Processing, IEEE Transactions on}, 55(7):3670--3686,
  2007.

\bibitem{solymosi2012incidence}
J.~Solymosi and T.~Tao.
\newblock An incidence theorem in higher dimensions.
\newblock {\em Discrete \& Computational Geometry}, 48(2):255--280, 2012.

\bibitem{starck2002curvelet}
J.-L. Starck, E.~J. Cand{\`e}s, and D.~L. Donoho.
\newblock The curvelet transform for image denoising.
\newblock {\em Image Processing, IEEE Transactions on}, 11(6):670--684, 2002.

\bibitem{stoica1997introduction}
P.~Stoica and R.~L. Moses.
\newblock {\em Introduction to spectral analysis}, volume~1.
\newblock Prentice hall Upper Saddle River, NJ, 1997.

\bibitem{storath2014jump}
M.~Storath, A.~Weinmann, and L.~Demaret.
\newblock Jump-sparse and sparse recovery using potts functionals.
\newblock 2014.

\bibitem{strichartz2003guide}
R.~S. Strichartz.
\newblock {\em A guide to distribution theory and Fourier transforms}.
\newblock World Scientific, 2003.

\bibitem{sturmfels1998polynomial}
B.~Sturmfels.
\newblock Polynomial equations and convex polytopes.
\newblock {\em The American Mathematical Monthly}, 105(10):907--922, 1998.

\bibitem{tang2013near}
G.~Tang, B.~N. Bhaskar, and B.~Recht.
\newblock Near minimax line spectral estimation.
\newblock In {\em Information Sciences and Systems (CISS), 2013 47th Annual
  Conference on}, pages 1--6. IEEE, 2013.

\bibitem{tsai2001curve}
A.~Tsai, A.~Yezzi~Jr, and A.~S. Willsky.
\newblock Curve evolution implementation of the {M}umford-{S}hah functional for
  image segmentation, denoising, interpolation, and magnification.
\newblock {\em Image Processing, IEEE Transactions on}, 10(8):1169--1186, 2001.

\bibitem{vaswani2010modified}
N.~Vaswani and W.~Lu.
\newblock Modified-{CS}: Modifying compressive sensing for problems with
  partially known support.
\newblock {\em Signal Processing, IEEE Transactions on}, 58(9):4595--4607,
  2010.

\bibitem{vese2002multiphase}
L.~A. Vese and T.~F. Chan.
\newblock A multiphase level set framework for image segmentation using the
  {M}umford and {S}hah model.
\newblock {\em International journal of computer vision}, 50(3):271--293, 2002.

\bibitem{vetterli2002sampling}
M.~Vetterli, P.~Marziliano, and T.~Blu.
\newblock Sampling signals with finite rate of innovation.
\newblock {\em Signal Processing, IEEE Transactions on}, 50(6):1417--1428,
  2002.

\bibitem{wang2004image}
Z.~Wang, A.~C. Bovik, H.~R. Sheikh, and E.~P. Simoncelli.
\newblock Image quality assessment: from error visibility to structural
  similarity.
\newblock {\em Image Processing, IEEE Transactions on}, 13(4):600--612, 2004.

\bibitem{xiang2005accelerating}
Q.-S. Xiang.
\newblock Accelerating {MRI} by skipped phase encoding and edge deghosting
  ({SPEED}).
\newblock {\em Magnetic Resonance in Medicine}, 53(5):1112--1117, 2005.

\bibitem{xu2014precise}
W.~Xu, J.-F. Cai, K.~V. Mishra, M.~Cho, and A.~Kruger.
\newblock Precise semidefinite programming formulation of atomic norm
  minimization for recovering d-dimensional ($d \geq 2$) off-the-grid
  frequencies.
\newblock In {\em Information Theory and Applications Workshop (ITA), 2014}.
  IEEE, 2014.

\bibitem{ye2002self}
J.~C. Ye, Y.~Bresler, and P.~Moulin.
\newblock A self-referencing level-set method for image reconstruction from
  sparse {F}ourier samples.
\newblock {\em International Journal of Computer Vision}, 50(3):253--270, 2002.

\bibitem{zhang2013coil}
T.~Zhang, J.~M. Pauly, S.~S. Vasanawala, and M.~Lustig.
\newblock Coil compression for accelerated imaging with {C}artesian sampling.
\newblock {\em Magnetic Resonance in Medicine}, 69(2):571--582, 2013.

\bibitem{zhu2008efficient}
M.~Zhu and T.~Chan.
\newblock An efficient primal-dual hybrid gradient algorithm for total
  variation image restoration.
\newblock {\em UCLA CAM Report}, pages 08--34, 2008.

\end{thebibliography}
%

\end{document}